\documentclass{article}

 \usepackage[main, final]{neurips_2026}
\usepackage[utf8]{inputenc} % allow utf-8 input
\usepackage[T1]{fontenc}    % use 8-bit T1 fonts
\usepackage{hyperref}       % hyperlinks
\usepackage{tabularx}
\usepackage{wrapfig} % in preamble
\usepackage{url}            % simple URL typesetting
\usepackage{booktabs}       % professional-quality tables
\usepackage{amsfonts}       % blackboard math symbols
\usepackage{nicefrac}       % compact symbols for 1/2, etc.
\usepackage{microtype}      % microtypography
\usepackage{xcolor}  
\usepackage{microtype}
\usepackage{graphicx}
\usepackage{subcaption}
\usepackage{booktabs} % for professional tables
\usepackage{algorithm}
\usepackage{algpseudocode}
\usepackage{booktabs}
\usepackage{multirow}
\usepackage{amssymb}
\usepackage{pifont}
\newcommand{\xmark}{\ding{55}}
\usepackage{amsmath}
\usepackage{amssymb}
\usepackage{mathtools}
\usepackage{amsthm}
\usepackage[utf8]{inputenc} % allow utf-8 input
\usepackage[T1]{fontenc}    % use 8-bit T1 fonts
\usepackage{hyperref}       % hyperlinks
\usepackage{url}            % simple URL typesetting
\usepackage{booktabs}       % professional-quality tables
\usepackage{amsfonts}       % blackboard math symbols
\usepackage{nicefrac}       % compact symbols for 1/2, etc.
\usepackage{microtype}      % microtypography
\usepackage{xcolor}         % colors
\usepackage{amsmath}
\usepackage{subcaption}
\usepackage{graphicx}
\usepackage{amsmath, amssymb, amsfonts, amsthm}
\usepackage{algpseudocode}
\usepackage{setspace}
\usepackage{multirow}
\usepackage{algorithm}
\usepackage{amsthm}

\usepackage{booktabs}
\usepackage{multirow}
\usepackage{amssymb}
\usepackage{pifont}

\newtheorem{theorem}{Theorem}[section]
\newtheorem{lemma}[theorem]{Lemma}
\newtheorem{proposition}[theorem]{Proposition}
\newtheorem{corollary}[theorem]{Corollary}

\newtheorem*{theorem_1_restated}{Theorem~\ref{thm: theorem_1} (Formal)}
\newtheorem*{theorem_2_restated}{Theorem~\ref{thm: theorem_2} (Formal)}
\newtheorem*{theorem_3_restated}{Theorem~\ref{thm: theorem_3} (Formal)}
\newtheorem*{theorem_4_restated}{Proposition~\ref{thm: theorem_4} (Formal)}
\newtheorem*{theorem_5_restated}{Theorem~\ref{thm: theorem_5} (Formal)}
\newtheorem*{theorem_6_restated}{Corollary~\ref{thm: corollary_1} (Formal)}

\theoremstyle{definition}

\theoremstyle{remark}

\definecolor{ForestGreen}{cmyk}{0.864, 0.0, 0.429, 0.396}
\definecolor{Green}{cmyk}{1.0, 0.0, 1.0, 0.5}
\definecolor{Brown}{rgb}{0.59,0.29,0.0}
\definecolor{Blue}{rgb}{0.0,0.0,1.0}
\definecolor{Orange}{rgb}{1.0,0.5,0.0}
\definecolor{Red}{rgb}{1.0,0.0,0.0}

\newif\ifcomments
\commentstrue    % set to \commentsfalse to hide all notes

\newcommand\shortsection[1]{\vspace{3pt}{\noindent\textbf{#1.}}}

\usepackage{amsmath,amsfonts,bm}

\def\eqref#1{equation~\ref{#1}}
\def\1{\bm{1}}

\DeclareMathAlphabet{\mathsfit}{\encodingdefault}{\sfdefault}{m}{sl}
\SetMathAlphabet{\mathsfit}{bold}{\encodingdefault}{\sfdefault}{bx}{n}

\usepackage[textsize=tiny]{todonotes}

\usepackage{hyperref}
\title{SplineFlow: Flow Matching for Dynamical Systems with B-Spline Interpolants}

\author{
  Santanu Subhash Rathod \\
  CISPA Helmholtz Center for
  Information Security \\
  \texttt{santanu.rathod@cispa.de}
  \And
  Pietro Li\`{o} \\
  Department of Computer Science \\
  and Technology, University of Cambridge \\
  \texttt{pl219@cam.ac.uk}
  \And
  Xiao Zhang \\
  CISPA Helmholtz Center for
  Information Security \\
  \texttt{xiao.zhang@cispa.de}
}
\begin{document}

\maketitle

\begin{abstract}
Flow matching is an emerging, scalable generative framework for characterizing continuous normalizing flows with wide-range applications. However, state-of-the-art methods are not well-suited for modeling dynamical systems, as they construct conditional paths that are restricted to linear interpolants, a suboptimal supervision signal, which may not capture the system's underlying state evolution. Moreover, constructing unified paths to satisfy multi-marginal constraints across observations is challenging, since naïve higher-order polynomials tend to be unstable and oscillatory. To address these limitations, we introduce SplineFlow, a theoretically grounded flow matching algorithm that jointly models conditional paths across observations via B-spline interpolation. Specifically, SplineFlow exploits the smoothness and stability of B-spline bases to learn the complex underlying dynamics in a structured manner while ensuring the multi-marginal requirements are met. Comprehensive experiments across deterministic and stochastic dynamical systems under various configurations, as well as cellular trajectory inference tasks, demonstrate that SplineFlow outperforms existing baselines, especially when the underlying dynamics are of higher degree or in irregular sampling regimes.
\end{abstract}

% Our code is available at: \url{https://github.com/santanurathod/SplineFlow}.
\section{Introduction}
Flow matching \citep{8_lipman_FM} is a generative framework with growing applications across a wide variety of domains, including materials science \citep{3_luo_crystalflow}, language modeling \citep{4_gat_discreteFM}, single-cell biology~\citep{12_rohbeck_FM}, and life science \citep{1_FM_Biology_Survey}.
The wide adoption results from the simplicity, scalability, and generality of its training process compared with alternative frameworks, such as score-based diffusion \citep{5_song_scorebaseddiffusion} and continuous normalizing flows (CNFs) \citep{6_chen_CNF, 7_rubanova_latentode}. % Flow matching trains the velocity field of the underlying generative dynamics by \textit{regressing} over a tractable conditional velocity field in a simulation-free manner, as opposed to several CNF-based techniques \citep{6_chen_CNF, 7_rubanova_latentode} that require expensive simulations for training.
In particular, unlike CNF-based methods that require expensive simulations, flow matching trains the velocity field of the generative dynamics in a simulation-free manner. At inference time, sampling involves forward ODE integration steps, which are more efficient than the denoising steps of score-based diffusion. Since accurately characterizing the velocity field lies at the heart of modeling dynamical systems, flow matching naturally provides a scalable probabilistic framework for efficiently learning the dynamics of the underlying system. 

The training objective of flow matching is a regression loss over conditional velocities, making the construction of conditional paths critical, as they resemble the \textit{supervision signal}.
These paths are typically constructed by linearly interpolating observed data points \citep{8_lipman_FM,13_zhang_TFM}, which enables efficient generation of high-fidelity samples with far fewer function evaluations \citep{29_straightflows_pooladian}. In these works, the primary emphasis is on the quality of the generated samples rather than on whether the learned dynamics meaningfully capture the underlying ground truth. To faithfully model dynamical systems, the learned neural velocity field must closely match the underlying dynamics throughout the time horizon, rather than merely generating valid samples at observed time points. Therefore, linear interpolants are a restrictive design choice, particularly suboptimal when the underlying dynamics are nonlinear, high-curvature, or contain exponential or sinusoidal terms, necessitating more flexible conditional path designs. 
However, constructing higher-degree conditional paths that are stable, scalable, and easily integrated with existing frameworks remains challenging, since higher-degree polynomial paths that satisfy all marginal observations tend to be oscillatory, a behavior known as Runge’s phenomenon \cite{14_epperson_rungephenomenon}.
While prior work has improved the vanilla design for modeling cellular dynamics~\citep{11_tong_MOTFM,2_rathod_contextflow}, the emphasis has been on better designs of the coupling scheme while still relying on linear paths. 

We present SplineFlow, a flow matching algorithm that constructs conditional probability paths using B-splines (Algorithm \ref{alg:splineflow}), a 
class of spline functions created using the Cox-de Boor recursion formula~\citep{15_boor_splinesbook}.
Specifically, we explain the limitations of linear conditional paths for modeling dynamical systems and highlight the advantages of B-spline interpolants, including lower approximation errors, greater flexibility, and improved stability (Section \ref{sec:linear vs. B-spline}). Theoretically, we show how SplineFlow integrates B-spline conditional paths into the flow matching framework, ensuring the regressed velocity field satisfies the marginal distributions and a more accurate characterization of the underlying dynamics (Section \ref{sec:splineflow algorithm}). We also extend our spline-based method to model stochastic dynamics with additive noise by regressing on both the drift and the score using simulation-free Schrödinger Bridges (SF2M)~\citep{20_tong_SF2M}. 
Comprehensive experiments across simulations of deterministic and stochastic dynamics with varying complexities, as well as real-world unpaired cellular trajectories, demonstrate the advantages of SplineFlow over baselines, especially pronounced for nonlinear and oscillatory systems or when sampling is irregular (Section \ref{sec:experiments}).

In summary, our contributions are as follows:
\begin{itemize}
    \setlength{\itemsep}{0pt}
    \setlength{\topsep}{0.1pt}
    \vspace{-0.05in}
    \item Using B-splines to construct conditional paths, we develop SplineFlow, an efficient and effective flow matching framework for modeling deterministic dynamical systems. We also present a natural SF2M extension of SplineFlow for handling stochastic dynamics. 

    \item By leveraging functional and first-derivative approximation error rates of B-splines, we analyze the theoretical advantages of SplineFlow in terms of asymptotic dynamics approximation error and sample generation quality compared to linear interpolants.
    
    \item 
    % We evaluate SplineFlow across ODE dynamical systems, their SDE counterparts,
    % and cellular trajectory inference tasks. 
    The improvement of SplineFlow over baselines is most pronounced for nonlinear and oscillatory ODE systems under irregular sampling, whereas for SDEs, the gains are strongest for nonlinear systems. Compared with adjoint-based methods, SplineFlow preserves performance while being significantly faster. For cellular trajectory inference, SplineFlow outperforms baselines under interpolation while remaining competitive under extrapolation.
\end{itemize}

    % \item Theoretically, we prove the \add{functional and first derivative} approximation error bounds of B-spline interpolants with comparisons to linear ones, showing that linear interpolants are a special case. \add{We also prove the theoretical advantages of SplineFlow in learning the true underlying system dynamics and trajectory generation.}
    % We also demonstrate the validity of B-spline probability paths relative to existing flow matching frameworks, as well as the form of the associated conditional velocity field used during training.
    
    % Additionally, SplineFlow computes spline parameters per trajectory before training, inducing negligible overhead.
    
\section{Related Work}

% \subsection{Continuous Generative Models}

\shortsection{Flow Matching}
Flow matching \citep{8_lipman_FM, 9_Vanden_FM, 10_Xingchao_FM} is a probabilistic modeling framework for characterizing the generative dynamics by regressing a neural network on tractable conditional velocity fields, which bypasses expensive simulations as required in previous continuous normalizing flow techniques \citep{27_neural_ODE, 25_JMLR_normalizing_flows}. Linear velocity fields are typically used in flow matching, as they have been shown to achieve faster training convergence and greater efficiency in function evaluations than other choices, such as diffusion paths \citep{8_lipman_FM,10_Xingchao_FM}. Notably, linear paths between noise and observations using optimal-transport couplings have resulted in even better performances \citep{29_straightflows_pooladian}. However, it is important to note that these techniques primarily optimize for endpoint generation fidelity and efficiency, and not for faithful characterization of intermediary dynamics.

\shortsection{Score Matching}
% Generative tasks usually involve transforming a tractable distribution into a desired distribution representing observations. Flow-Matching techniques are used when the underlying dynamics of the transformations are modeled deterministically, while score based techniques are used to model the transformations stochastically. 
In score-based diffusion~\citep{5_song_scorebaseddiffusion, 31_diffusion_nicol}, stochastic transformations are modeled by reversing a tractable Gaussian-noising process, which results in a score-dependent denoising velocity field. While capable of generating complex, high-dimensional data, its application to modeling stochastic dynamics is limited by its reliance on the Gaussian-noising steps. Stochastically transforming two arbitrary distributions between each other is well known as the Schrödinger Bridge problem \citep{31_schrodingerbridge_leonard,32_SB_schrodinger} with several applications in generative modeling \citep{30_schrodingerbridge_bortoli,33_SB_vargas}. However, the iterative methods for their approximate solutions \citep{30_schrodingerbridge_bortoli,34_SB_Bunne} suffer from instability and numerical issues, limiting their applicability. To bridge these gaps, \cite{20_tong_SF2M} proposed SF2M (score and flow matching), which models the stochastic dynamics by regressing probability flow velocity fields and probability scores over conditional probability paths modeled via Brownian bridges. Similar to flow matching, it again uses linear interpolants to construct the conditional paths.

\shortsection{Modeling Dynamical Systems}
Although the underlying state transformations of deterministic systems can be modeled via CNFs~\citep{27_neural_ODE}, these methods require computing an adjoint state via forward simulation at each iteration, which limits scalability. Similar requirements arise when modeling stochastic dynamics \citep{35_neuralsde_kidger} or dynamics from irregularly sampled observations \citep{7_rubanova_latentode}. Flow matching addresses the above scalability issue; however, most efforts on the dynamical front have been limited to solutions for modeling processes underlying unpaired datasets, such as transcriptomic trajectories \cite{1_FM_Biology_Survey}. 
For instance, \citet{11_tong_MOTFM} constructed the conditional paths between temporal transcriptomic samples by linearly interpolating between minibatch optimal transport (OT) couplings, while \citet{2_rathod_contextflow} achieved improvements by incorporating biological priors using spatial omics data. 
\citet{13_zhang_TFM} employed flow matching for modeling dynamical systems, but they still rely on piecewise linear conditional paths, with no provision for faithfully characterizing the intermediary system dynamics.
Recently, \citep{12_rohbeck_FM} proposed multi-marginal OT-couplings with cubic splines. However, while the method enables smooth interpolation, it is primarily motivated by minimum-curvature considerations rather than faithful approximation of the underlying system dynamics, and lacks rigorous guarantees regarding spline selection and the flexibility to convert to lower-order or higher-order polynomials when needed. Different from these works, SplineFlow is the first method to demonstrate the efficacy of modeling dynamics across degrees of complexity using well-motivated B-spline interpolants. 

\section{Motivating B-Splines for Dynamic Modeling}
\label{sec:splineflow modeling dynamical systems}

% The evolution of the random variable $x_{t}$ is characterized from the source distribution at $t = 0$ to the target distribution at $t = 1$. 

\subsection{Problem Setup}

Let $x_t$ be the state variables of the underlying dynamics at time $t$ governed either by an ODE $dx_{t} = u_{t}(x_{t})dt$ if the system is deterministic, or by a SDE $dx_{t} = u_{t}(x_{t})dt + g(t)d{w}_{t}$ if stochastic.
Here, $u_t(\cdot)$ is the underlying velocity or drift, $g(t)$ stands for the diffusion schedule, and $w_t$ is the standard Wiener process. 
Let $\{X_i\}_{i=0}^{N-1}$ be a collection of $N$ training trajectories that follow the underlying system dynamics but with different initial values, where $X_i = [ x_i(t_0^i), x_i(t_1^i), \ldots, x_i(t_{n_i}^i) ]$ is the $i$-th trajectory, $x_i(t_j^i)\in\mathbb{R}^d$ denotes the $j$-observation, $x_i(t_0^i)$ stands for the initial value, and $0 = t_0^i < t_1^i < \cdots < t_{n_i}^i = 1$ denote the observed timestamps.
We use $p\in[0, 1]$ to denote the degree of sampling irregularity (or sparsity), where a fraction $p$ of observations are randomly masked.
The task is to learn a velocity network $u_{\theta}$ (and also a score network $s_{\phi}$ for stochastic systems), such that for any testing trajectory $X$, it minimizes the MSE across any sampled time point $t'$ between the ground-truth observation $x_{t'}$ and the predicted observations by integrating from the initial value $x_{t_0}$. 
Formally, if the underlying system is defined by an ODE, the goal is to minimize:
\begin{align}
\label{eq: ode objective}
    \min_{\theta} \: \mathbb{E} \bigg[ \Big\| x_{t_0} + \int_{t_0}^{t'} u_{\theta}(x_{t}, t) dt - x_{t'} \Big\|^2 \bigg]. 
\end{align}
Otherwise, if the underlying system is stochastic and determined by an SDE, we minimize:
\begin{align}
\label{eq: sde objective}
    \min_{\theta, \phi} \: \mathbb{E} \bigg[ \Big\| x_{t_0} + \int_{t_0}^{t'} \Big( (u_{\theta}(x_{t}, t) + \frac{1}{2} g(t)^2 s_{\phi}(x_{t}, t))dt+g(t)dw_{t} \Big) - x_{t'} \Big\|^2 \bigg].
\end{align}
See Appendix \ref{section: Flow And Score Matching} for detailed derivations of why minimizing Equation \ref{eq: sde objective} is desirable for SDEs.

% We defer the preliminaries discussing \textit{matching} based on ODE and SDE formulations to the Appendix~\ref{sec:preliminaries}.

\subsection{Limitation of Linear Interpolants}

While flow matching is primarily designed to optimize endpoint data generation (i.e., only considering $t=1$), it has been extended to model the state evolution of dynamical systems~\citep{13_zhang_TFM, 11_tong_MOTFM}, where observations are made at multiple time points across different trajectories. 
In prior literature on flow matching, $u_{\theta}$ is typically trained by regressing on conditional velocity fields that generate Gaussian conditional paths, using \textit{linear interpolants} for each consecutive observation pair.
Specifically, for any $t\in[t^{i}_{j}, t^{i}_{j+1})$, the target conditional probability path is constructed as:
\begin{align}
\label{eq:multi-time-sampling-path}
    p_{t}(x | z) = \mathcal{N}\!\left(
    x; \mu_t(z), \sigma^2 I
    \right), \text{ where }  \mu_t(z) = \frac{t^{i}_{j+1} - t}{t^{i}_{j+1} - t^{i}_{j}} x_{i}(t^{i}_{j})
    + \frac{t - t^{i}_{j}}{t^{i}_{j+1} - t^{i}_{j}} x_{i}(t^{i}_{j+1}),
\end{align}
where $z= (x_{i}(t^{i}_{j}), x_{i}(t^{i}_{j+1}))$ denotes a latent variable, and $\sigma>0$ is a small constant representing the standard deviation.
According to the construction, the associated conditional velocity field takes a simple form of piecewise functions: for any $t\in[t^{i}_{j}, t^{i}_{j+1})$, $u_t(x | z) = ( x_{i}(t^{i}_{j+1}) - x_{i}(t^{i}_{j} )) / (t^{i}_{j+1} - t^{i}_{j})$.
% $$
% u_t(x | z) = \frac{x_{i}(t^{i}_{j+1}) - x_{i}(t^{i}_{j})}{t^{i}_{j+1} - t^{i}_{j}}, \:\: \forall t\in[t^{i}_{j}, t^{i}_{j+1}).
% $$
Flow matching algorithms are then designed to minimize the following regression loss:
\begin{align}
\label{eq: conditional flow matching loss main}
    \min_{\theta} \: \mathbb{E}_{t \sim \mathcal{U}(0,1), z \sim q(z), x \sim p_t(x | z)} \big\| u_\theta(t, x) - u_t(x | z) \big\|^2,
\end{align}
where $\mathcal{U}(0,1)$ is the uniform distribution over $[0,1]$, and $q(z)$ refers to the joint distribution, usually corresponding to either independent~\citep{8_lipman_FM} or OT-based couplings~\citep{29_straightflows_pooladian}, for the consecutive time points $(t_j, t_{j+1})$ such that the sampled $t$ falls inside. Full discussions on the preliminaries of flow matching and their extensions to SDEs are provided in Appendix \ref{sec:preliminaries}. 

Since linear probability paths interpolate between observational points, the trained velocity field will induce a marginal distribution that is approximately similar to the observational distribution, thus modeling the underlying dynamics. However, while linear paths connect observations at endpoints, they cannot approximate the intermediary functional values well, 
and in particular, often miss out on curvature and induce significant bias, as illustrated in Figure \ref{fig: linear_vs_bspline_interpolation}.

\begin{figure}[t]
  \centering
  \begin{subfigure}[t]{0.32\textwidth}
    \includegraphics[width=\linewidth]{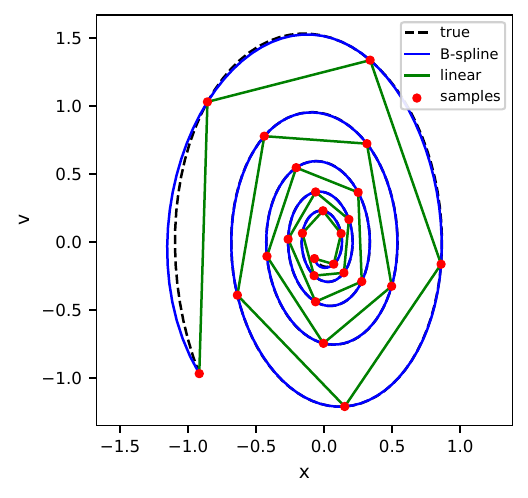}
    \caption{Phase Space}
    \label{fig: dh_phase}
  \end{subfigure}\hfill
  \begin{subfigure}[t]{0.32\textwidth}
    \includegraphics[width=\linewidth]{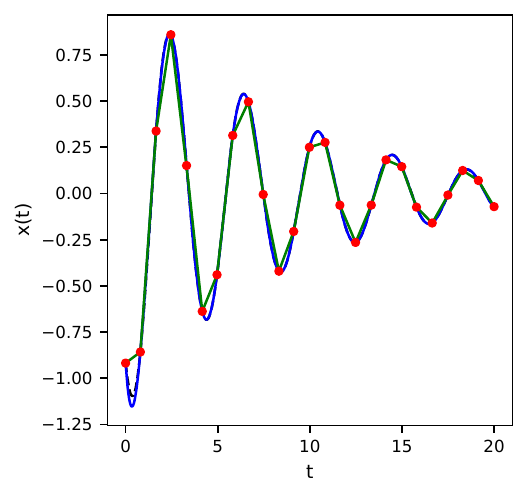}
    \caption{Position}
    \label{fig: dh_position}
  \end{subfigure}\hfill
  \begin{subfigure}[t]{0.32\textwidth}
    \includegraphics[width=\linewidth]{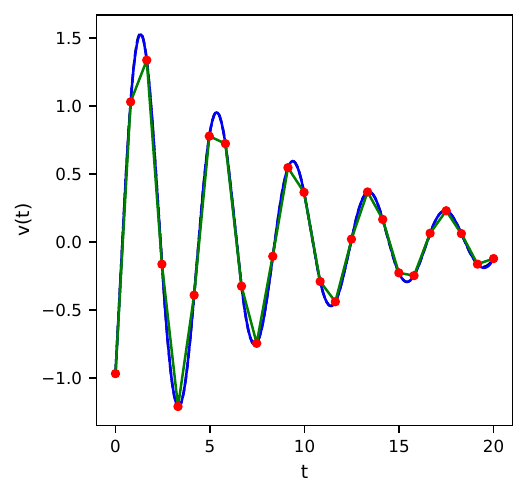}
    \caption{Velocity}
    \label{fig: dh_velocity}
  \end{subfigure}
  \vspace{-0.05in}
  \caption{Visualizations of interpolated trajectories from a damped harmonic oscillator.}
  \vspace{-0.1in}
  \label{fig: linear_vs_bspline_interpolation}
\end{figure}

\subsection{Advantages of B-Spline Interpolants}
\label{sec:linear vs. B-spline}

% \vspace{-1em}
% Flow Matching \citep{8_lipman_FM} is a promising candidate for approximating the underlying dynamics, since it is essentially a faster way to train continuous normalizing flows \citep{6_chen_CNF} that generate the observational data. 

% However, from Figure \ref{fig: linear_vs_bspline_interpolation}, which shows a trajectory (x(t), v(t)) from a damped harmonic system (Appendix \ref{apdx: dynamical_systems_datasets}), we can clearly see that linearly interpolating observations does not capture the true underlying trajectory, especially when interpolating between points lying on a curvature. 

% \subsection{Advantages of B-spline Interpolants}
% \label{sec:B-spline interpolants}

% In the flow matching framework, constructing valid conditional probability paths essentially boils down to ensuring that the paths satisfy the observational distribution at sampled time points, as shown in Theorem \ref{thm: theorem_3}. When using a Gaussian distribution, this task reduces to representing the mean $\mu_{t}(\cdot)$ as an interpolant function of the observations: $\mu_{t^{i}_{0}}(z)=x^{i}_{t^{i}_{0}}, \cdots, \mu_{t^{i}_{n_{i}}} (z)=x^{i}_{t^{i}_{n_{i}}}$, with $z= (X^{i})$. \xnote{Do we need the above paragraph?}

% Given the above limitations, we propose constructing the conditional paths using B-spline interpolants to faithfully capture the complex underlying system dynamics. 
B-Splines are widely used to model complex shapes and functions across several domains~\citep{16_hasan_bsplineapplications,17_biswas_bsplineapplications,18_li_bsplineapplications}. Notably, B-splines are known to be stable, smooth, and flexible with respect to the choice of polynomial degree, whereas other alternative choices of higher-order interpolants are prone to oscillatory behavior such as Runge's phenomenon \cite{14_epperson_rungephenomenon} and to non-smoothness (see Appendix \ref{apdx: Polynomial Interpolants} for detailed discussions).
Let $[x_{{t_{0}}}, x_{t_{1}}, \cdots, x_{t_{n}}]$ be the $n+1$ points we need to interpolate over. The B-spline interpolant function is built over by recursively combining basis splines using the Cox-De Boor Recursion formula \citep{15_boor_splinesbook}. Specifically, the B-spline bases are defined recursively by:
\begin{equation}
\begin{aligned}
\label{eq: Cox-De Boor Recursion}
    \mathcal{B}_{j,m}(t)
    =
    \frac{t - t_j}{t_{j+m} - t_j}\, \mathcal{B}_{j,m-1}(t) +
    \frac{t_{j+m+1} - t}{t_{j+m+1} - t_{j+1}}\, \mathcal{B}_{j+1,m-1}(t), \:\: \forall m\geq 1,
\end{aligned}
\end{equation}
where the $0$-th degree B-spline basis is given by
$\mathcal{B}_{j,0}=1$ if $t\in[t_{j}, t_{j+1})$, and $\mathcal{B}_{j,0}=0$ otherwise.
The interpolant function $\mu(t)$ of degree $m$ is then a linear combination of $m$-th degree B-splines, such that it recovers the $n+1$ observation points. Formally, $\mu(t)$ is defined as:
\begin{equation}
\label{eq: B-spline interpolant function}
\begin{aligned}
\mu(t)& = \sum_{j=0}^{n} c_{j,m} \mathcal{B}_{j,m} (t), \:\: \text{ s.t. } \:
\mu(t_{j})=x_{t_j}, \: \forall \: j \in \{0,1,\ldots,n\}, \\
% \mu(t)&= [\mu^{1}(t), \cdots, \mu^{d}(t)]
\end{aligned}
\end{equation}
% \vspace{-0.5em}
where $c_{j,m}$ is determined by solving the linear constraints. Note that $\mu(t)$ is continuous by construction and gets increasingly smooth with the degree of the chosen bases \citep{15_boor_splinesbook}. In practice, for data $x \in \mathbb{R}^{d}$, we calculate B-spline values at time $t$ separately for each dimension $d$ and concatenate them while using. 
The following theorem compares the approximation errors, both functional and first-derivative, of higher-degree B-splines with those of linear interpolants.
\begin{theorem}[Informal]
\label{thm: theorem_1}
Assume $f\in\mathcal{C}^{m}[a, b]$ is a function with $m$-th bounded derivatives, and we observe the its value at
$n+1$ equidistant points $\{t_j\}_{j=0}^n$. Let $\mu(t)$ be the B-spline interpolant of degree $m$ and $p_{1}(t)$ be the linear interpolant that fits the observations. Then, we have
\begin{align*}
    \|f(t)-\mu(t)\|_\infty=\mathcal{O}(n^{-m}) \quad &\text{and} \quad \|f(t)-p_{1}(t)\|_\infty=\mathcal{O}(n^{-2}), \\
    \|f'(t)-\mu'(t)\|_\infty=\mathcal{O}(n^{-m+1}) \quad &\text{and} \quad \|f'(t)-p_{1}'(t)\|_\infty=\mathcal{O}(n^{-1}).
\end{align*}
% $\|f(t)-\mu(t)\|_\infty=\mathcal{O}(n^{-m})$ and $\|f(t)-p_{1}(t)\|_\infty=\mathcal{O}(n^{-2})$, along with $\|f'(t)-\mu'(t)\|_\infty=\mathcal{O}(n^{-m+1})$ and $\|f'(t)-p_{1}'(t)\|_\infty=\mathcal{O}(n^{-1})$.
\end{theorem}
Theorem \ref{thm: theorem_1} shows that if a function has bounded $m$-th derivatives, then B-spline interpolants can achieve much lower approximation error compared with linear interpolants.
We provide the formal, complete statement and its detailed proof in Appendix \ref{section: proof of theorem 4.1}.
The degree $m$ is chosen based on the nature of the underlying data. For nonlinear dynamical systems, the best-fit is usually greater than $2$.
% Analogous properties transfer to the first derivatives as well, which ultimately enable better velocity fields as presented in further sections.
While Theorem \ref{thm: theorem_1} is stated for equidistant or regularly sampled data points, B-splines are known to be stable even in the case of irregular sampling~\citep{15_boor_splinesbook,24_boor_irregularbound} as seen in Lemmas~\ref{thm: lemma_3}-\ref{thm: lemma_5}, which are central results in our proofs and are not limited to regularity.
% where the operator norm becomes dependent on the global mesh ratio and is shown to be upper bounded still \cite{15_boor_splinesbook,24_boor_irregularbound}. 

We conduct preliminary experiments to compare estimates obtained with B-spline and linear interpolants for trajectories sampled from a damped harmonic oscillator (see Appendix \ref{apdx:damped harmonic oscillator} for its formal definition). 
Figure \ref{fig: linear_vs_bspline_interpolation} visualizes the comparison results, where 
the practical implications of our theoretical results become quite clear: B-splines are much more adept at capturing the underlying curvature compared to linear interpolants. B-splines also retain the flexibility in terms of degree with which to model, unlike cubic interpolants, which are restricted to the cubic degree.

\begin{theorem}[Informal]
\label{thm: theorem_2}
A linear interpolant is a special case of a B-spline interpolant of degree $1$.
\end{theorem}
Theorem \ref{thm: theorem_2} with a formal statement is proven in Appendix \ref{section: proof of theorem 4.2}.
When the underlying function is linear in time, B-spline interpolants are expected to recover the performance of linear interpolants.

\section{SplineFlow: Flow Matching with B-Spline Interpolation}
\label{sec:splineflow algorithm}

% \vspace{-1em}
Motivated by the efficacy, stability, and flexibility of B-splines in Section \ref{sec:linear vs. B-spline}, we propose SplineFlow, which constructs conditional paths in flow matching frameworks via B-spline interpolation. Algorithm \ref{alg:splineflow} in Appendix \ref{apdx:SFAlgorithm} presents the pseudocode of SplineFlow for both deterministic and stochastic cases.

\subsection{Modeling Deterministic Dynamics} 

When the dynamics are governed by an ODE, we model $p(x | z)$ using Gaussian distributions $\mathcal{N}(x; \mu_{t}(z), \sigma_t^{2}I)$, where $\mu_t(z)$ is replaced by B-spline interpolants. 
The following theorem, proven in Appendix~\ref{section: proof of theorem 4.5}, derives an error bound with respect to the neural velocity field learned by SplineFlow.

\begin{theorem}[Informal]
\label{thm: theorem_5}
Let $f(t)$ be the ground-truth ODE dynamics of the system. Then, with equidistant observations, the mean Euclidean error between the learned velocity field $u_{\theta}(t,x)$ from SplineFlow and the true velocity field is upper bounded by:
\begin{align*}
\mathbb{E}\big[ \|u_{\theta}(t,x) - f'(t)\| \big]
\le
\sqrt{\text{Training loss}}+\mathcal{O}(n^{-m+1}),
\end{align*}
where $m$ denotes the degree of B-spline functions selected in SplineFlow. 
% Linear baselines correspond to an error rate of $\mathcal{O}(n^{-1})$.
\end{theorem}

Theorem \ref{thm: theorem_5} shows that SplineFlow approximates the ground-truth dynamics better than existing linear-interpolant-based approaches with an error term corresponding to $\mathcal{O}(n^{-1})$.  
In practice, the training loss usually converges to $0$ (see Appendix~\ref{apdx:loss_curves}), and thus, better approximation properties of B-splines translate to better conditional paths whose derivatives are regressed over by the neural network $u_{\theta}$, eventually leading to better prediction of the ground-truth dynamics. 

Theorem \ref{thm: theorem_5} has implications for the sampling or generation process, in which SplineFlow generates samples closer to the ground truth than existing approaches, as shown in the corollary below.

\begin{corollary}[Informal]
\label{thm: corollary_1}
Under the same setting as Theorem~\ref{thm: theorem_5}, the distance between ground truth $x_{t}$ and generated samples $\hat{x}_{t}$ from a trajectory for SplineFlow at any $t\in[a,b]$ is upper bounded by
$$
\mathbb{E} \big[\| x_t - \hat{x}_t\| \big]\le(b-a)\times\sqrt{\text{Training loss}} + \mathcal{O}(n^{-m+1}).
$$
\end{corollary}

We present the formal version of Corollary \ref{thm: corollary_1} and its detailed proofs in Appendix \ref{section: proof of corollary 4.6}. 
The following theorem proves that if we employ the B-spline-based construction, we can ensure that (i) the marginal distributions can be recovered from conditional probability paths, (ii) the conditional velocity field for the dynamic case, analogous to the static case as in Equation \ref{eq: conditional probability path} satisfies the continuity equation, and (iii) the conditional velocity field that induces the B-spline probability path has an analytical form.
% Theorem \ref{thm: theorem_3} below proves these rigorously. 

% Most existing Flow Matching frameworks use linear conditional paths and utilize the velocity field to model either data generation dynamics going from noise $(t=0)$ to data $(t=1)$ \citep{8_lipman_FM,9_Vanden_FM} and pairwise uncoupled dynamics \citep{11_tong_MOTFM,2_rathod_contextflow}. Trajectory Flow Matching \citep{13_zhang_TFM} comes closest to our work in terms of modeling dynamics of underlying systems from the observed trajectories, however, it too utilizes linear conditional paths that may not be suitable for higher-degree and nonlinear dynamics as we've discussed before and as we'll show in our experiments. 

\begin{theorem}[Informal]
\label{thm: theorem_3}
Let $\mu_t(\cdot)$ be the B-spline interpolant of degree $m$ and the conditional probability path be $p_t(x | z)=\mathcal{N}(x;\mu_{t}(z), \sigma_t^2 I)$ with $\sigma_t(z) = \sigma \to 0$.
Then, the induced marginal $p_t$ recovers the data marginals and satisfies the continuity equation, with a conditional velocity field:
$$
u_t(x| z)
=
\sum_{j=1}^{n} m \cdot c_{j,m}
\left(
\frac{\mathcal{B}_{j,m-1}(t)}{t_{j+m}-t_j}
-
\frac{\mathcal{B}_{j+1,m-1}(t)}{t_{j+m+1}-t_{j+1}}
\right).
$$
\end{theorem}
% \vspace{-1em}
We present the formal version of Theorem \ref{thm: theorem_3} and its proof in Appendix \ref{section: proof of theorem 4.3}.
According to Theorem \ref{thm: theorem_3}, one can show that $\nabla_{\theta}\mathcal{L}_{\mathrm{CFM}}=\nabla_{\theta}\mathcal{L}_{\mathrm{FM}}$ even in the dynamic case, analogous to the static case as shown in \citet{8_lipman_FM}. This enables us to finally train the network $u_{\theta}$ using the regression loss (Equation \ref{eq: flow matching regression loss}), with conditional velocity fields corresponding to a B-spline probability path. 
% Formally, the training objective of SplineFlow is defined as:
% \begin{align}
% \label{eq: conditional flow matching loss SplineFlow}
%     \min_{\theta} \: \mathbb{E}_{t \sim \mathcal{U}(0,1), z \sim q(z), x \sim p_t(x | z)} \big\| u_{\theta}(t, x) - u_t(x | z) \big\|^2.
% \end{align}
% Algorithm \ref{alg:splineflow} summarizes the training pipeline used to train the neural network dynamics. 
% Once training is completed, the time-varying observations for a given initial value can be generated by integrating the velocity field: $x_{t'}= x_{t_0} + \int_{t_0}^{t'} u_{\theta}(x_{t}, t)dt$ for any $t'$.

\subsection{Modeling Stochastic Dynamics}

We further extend SplineFlow to the \textit{flow and score matching} (SF2M) framework~\citep{20_tong_SF2M} for handling stochastic dynamics (see Appendix \ref{section: Flow And Score Matching} for full preliminaries on SF2M).
When the underlying SDE has a constant diffusion schedule $g(t)=\sigma$, we can express the drift $u_{t}(x_{t})$ based on the probability flow velocity field $u^{o}_{t}$ and the probability score $\nabla_{x}\log(p_{t}(x_{t}))$:
\begin{align}
\label{eq: original SDE drift main}
    u_{t}(x_{t})= u^{o}_{t}(x_{t})+\dfrac{\sigma^2}{2}\nabla_{x}\log(p_{t}(x_{t})),
\end{align}
where $u^{o}_{t}(x_{t})$ is formally defined in Equation \ref{eq: probability flow ODE}.
From Theorem 3.2 in  \citet{20_tong_SF2M}, we know that constructing a conditional probability flow $u_{t}^{o}(\cdot)$ and its induced probability path $p_{t}(\cdot)$ satisfying marginal constraints gives us a tractable way to train the probability flow velocity field approximation $u_{\theta}(t,x)$ and the score approximation $s_\phi(t,x)$ by optimizing over the SF2M regression loss:
\begin{equation}
\begin{aligned}
\label{eq: SF2M loss main}
&\mathcal{L}_{\text{SF2M}}(\theta, \phi)
=
\mathbb{E}_{t, z, x}
\Big[
\| u_\theta(t,x) - u^{o}_t(x | z) \|^2 + \lambda(t)^2
\| s_\phi(t, x) - \nabla \log p_t(x | z) \|^2
\Big].
\end{aligned}
\end{equation}

We adapt our SplineFlow framework to construct  $p_{t}(x | z)$ using B-spline interpolation and the induced $u^{o}_{t}(x| z)$ for modeling stochastic dynamics. The following proposition, proven in Appendix \ref{section: proof of proposition 4.4}, states the simplified optimization objective of  SplineFlow for modeling SDE dynamics. 

\begin{proposition}
\label{thm: theorem_4}
Assume the underlying SDE dynamics have a constant diffusion $\sigma$.
Let $p_t(x | z)=\mathcal{N}(x;\mu_t(z),\sigma_t^2(z))$ be a
Spline-Bridge with B-spline mean $\mu_t$ and standard deviation $\sigma_{t_i}\to 0$ at observed timestamps. Then, for  $\lambda(t)=\sigma_t(z)$, the $\text{SF2M}$ regression objective can be simplified as:
$$
\mathcal{L}_{\mathrm{SplineFlow}}(\theta, \phi)=\mathbb{E}\Big[\| u_\theta(t,x)-(\epsilon\,\sigma_t'+\mu_t') \|^2+\|\lambda(t)s_\phi(t,x)+\epsilon\|^2\Big],
$$
where $\epsilon\sim\mathcal{N}(0,I)$, $\sigma_t'$ and $\mu_t'$ stand for the derivatives of $\sigma_t$ and $\mu_t$.
\end{proposition}
% \vspace{-1em}
When modeling stochastic dynamical systems, we observed that the choice of $\sigma_{t}(z)$ can drastically affect performance. Apart from setting $\sigma_t(z) = \sigma$, we proposed a piecewise quadratic scheme:
\begin{align}
\label{eq:quadratic_variance}
\sigma^{2}_{t}(z)=\dfrac{\sigma^{2}}{(t_{j+1}-t_{j})^{2}} {(t-t_j)(t_{j+1}-t)}, \:\: \forall t \in [t_j, t_{j+1}), 
\end{align}
% $\sigma^{2}_{t}(z)=\dfrac{\sigma^{2}_{\text{min}}}{(t_{i+1}-t_{i})^{2}} {(t-t_i)(t_{i+1})}$, $t \in [t_{i}, t_{i+1})$,
which empirically works the best when modeling cellular dynamics. 
Algorithm \ref{alg:splineflow} outlines all the necessary steps used to train the velocity network $u_{\theta}$ and the score network $s_{\phi}$ and to obtain the SDE drift term $u_{t}$ (Equation \ref{eq: original SDE drift main}). Together with the diffusion schedule $\sigma$, which can be estimated (see Appendix~\ref{apdx:hyperparameter_selection} for detailed procedure), this drift can then be integrated using the Euler-Maruyama SDE discretization scheme~\citep{37_SDEnumericalmethods_platen} to generate SDE trajectories from given initial values.

\begin{table*}[t]
\centering
\caption{Comparisons of different methods for modeling ODE dynamics in MSE across $6$ deterministic systems and varying sampling irregularity $p$. The best performance is highlighted in bold.}
\label{tab:ode_combined_all_datasets}
\vspace{-0.05in}
\small
\setlength{\tabcolsep}{6pt}
\resizebox{0.98\linewidth}{!}{
\begin{tabular}{llcccccc}
\toprule
$\bm{p}$ & \textbf{Model} & \textbf{Exp.} & \textbf{Harm.} & \textbf{Damp.} & \textbf{LV} & \textbf{Hopper.} & \textbf{Lorenz} \\
\midrule

\multirow{5}{*}{$0$}
& NeuralODE  & $9.99\mathrm{e}{-1}$ & $6.54$ & $3.47$ & $1.98\mathrm{e}{4}$ & $2.95$ & $1.49$ \\
& LatentODE  & $1.00\mathrm{e}{-4}$ & $8.20\mathrm{e}{-4}$ & $1.20\mathrm{e}{-4}$ & $\mathbf{1.56\mathrm{e}{-3}}$ & $1.89$ & $1.16$ \\
& MMFM       & $1.40\mathrm{e}{-3}$ & $1.40\mathrm{e}{-2}$ & $2.00\mathrm{e}{-3}$ & $4.87\mathrm{e}{-1}$ & $1.82\mathrm{e}{+1}$ & $8.02\mathrm{e}{-1}$ \\
& TFM        & $6.60\mathrm{e}{-4}$ & $1.81\mathrm{e}{-1}$ & $1.50\mathrm{e}{-2}$ & $1.84\mathrm{e}{-1}$ & $1.55$ & $2.44$ \\
& SplineFlow & $\mathbf{8.90\mathrm{e}{-5}}$ & $\mathbf{3.70\mathrm{e}{-4}}$ & $\mathbf{9.50\mathrm{e}{-5}}$ & $8.60\mathrm{e}{-2}$ & $\mathbf{1.41}$ & $\mathbf{6.39\mathrm{e}{-1}}$ \\

\midrule
\multirow{4}{*}{$0.25$}
& LatentODE  & $1.00\mathrm{e}{-4}$ & $1.00\mathrm{e}{-3}$ & $3.40\mathrm{e}{-4}$ & $\mathbf{9.40\mathrm{e}{-4}}$ & $1.61$ & -- \\
& MMFM       & $2.00\mathrm{e}{-3}$ & $1.40\mathrm{e}{-2}$ & $1.80\mathrm{e}{-3}$ & $3.72\mathrm{e}{-1}$ & $3.13\mathrm{e}{+1}$ & -- \\
& TFM        & $6.90\mathrm{e}{-4}$ & $1.74\mathrm{e}{-1}$ & $1.50\mathrm{e}{-2}$ & $2.59\mathrm{e}{-1}$ & $1.51$ & -- \\
& SplineFlow & $\mathbf{8.20\mathrm{e}{-5}}$ & $\mathbf{6.50\mathrm{e}{-4}}$ & $\mathbf{3.00\mathrm{e}{-5}}$ & $1.00\mathrm{e}{-1}$ & $\mathbf{1.43}$ & -- \\

\midrule
\multirow{4}{*}{$0.5$}
& LatentODE  & $1.60\mathrm{e}{-4}$ & $7.60\mathrm{e}{-4}$ & $7.30\mathrm{e}{-4}$ & $\mathbf{1.55\mathrm{e}{-3}}$ & $\mathbf{1.65}$ & -- \\
& MMFM       & $1.70\mathrm{e}{-3}$ & $1.20\mathrm{e}{-2}$ & $3.00\mathrm{e}{-3}$ & $4.41$ & $4.78\mathrm{e}{+1}$ & -- \\
& TFM        & $9.40\mathrm{e}{-4}$ & $1.77\mathrm{e}{-1}$ & $1.60\mathrm{e}{-2}$ & $4.26$ & $1.72$ & -- \\
& SplineFlow & $\mathbf{1.10\mathrm{e}{-4}}$ & $\mathbf{5.50\mathrm{e}{-4}}$ & $\mathbf{4.70\mathrm{e}{-5}}$ & $7.05\mathrm{e}{-1}$ & $1.83$ & -- \\

\midrule
\multirow{4}{*}{$0.75$}
& LatentODE  & $9.60\mathrm{e}{-5}$ & $\mathbf{7.90\mathrm{e}{-4}}$ & $9.70\mathrm{e}{-4}$ & $\mathbf{2.90\mathrm{e}{-3}}$ & $3.56$ & -- \\
& MMFM       & $2.40\mathrm{e}{-3}$ & $1.00\mathrm{e}{-2}$ & $4.00\mathrm{e}{-3}$ & $8.62$ & $6.11\mathrm{e}{+1}$ & -- \\
& TFM        & $1.40\mathrm{e}{-3}$ & $2.05\mathrm{e}{-1}$ & $3.10\mathrm{e}{-2}$ & $1.98$ & $\mathbf{2.81}$ & -- \\
& SplineFlow & $\mathbf{8.60\mathrm{e}{-5}}$ & $1.12\mathrm{e}{-3}$ & $\mathbf{4.80\mathrm{e}{-5}}$ & $1.78$ & $3.34$ & -- \\

\midrule
\multicolumn{8}{c}{\textbf{Avg. Training Runtime (s)}} \\
\midrule
& LatentODE  & $2.80\mathrm{e}{+4}$ & $4.47\mathrm{e}{+4}$ & $4.38\mathrm{e}{+4}$ & $4.44\mathrm{e}{+4}$ & $1.21\mathrm{e}{+5}$ & $5.30\mathrm{e}{+4}$ \\
& SplineFlow & $\mathbf{1.90\mathrm{e}{+2}}$ & $\mathbf{2.73\mathrm{e}{+2}}$ & $\mathbf{2.67\mathrm{e}{+2}}$ & $\mathbf{2.65\mathrm{e}{+2}}$ & $\mathbf{2.02\mathrm{e}{+2}}$ & $\mathbf{2.79\mathrm{e}{+2}}$ \\
& Speedup ($\times$) & $\mathbf{147}$ & $\mathbf{163}$ & $\mathbf{164}$ & $\mathbf{167}$ & $\mathbf{597}$ & $\mathbf{190}$ \\
\bottomrule
\end{tabular}
}
\vspace{-0.05in}
\end{table*}

\section{Experiments}
\label{sec:experiments}
% 
% \shortsection{Datasets}
We test the performance of SplineFlow on several ODE dynamical systems, including Exponential Decay, Harmonic Oscillator, Damped Harmonic Oscillator, Lotka--Volterra, and Lorenz System, along with their SDE counterparts (see Appendix \ref{apdx: dynamical_systems_datasets} for their definitions). In addition, we evaluate SplineFlow on the HopperPhysics dataset from the MuJoCo Physics simulation engine, and on preprocessed, log-normalized longitudinal transcriptomic datasets, including the Embryoid Body temporal data from \citet{38_ebdata_burkhardt} and a post-traumatic brain regeneration dataset from \citet{39_stereoseq_wei}, across two-dimensional PHATE and PCA embeddings with 10 principal components (PCs).

% \shortsection{Baselines \& Metrics}  
For ODE, we compare SplineFlow with adjoint-based NeuralODE~\citep{27_neural_ODE} and LatentODE~\citep{7_rubanova_latentode}, as well as Trajectory Flow Matching (TFM)~\citep{13_zhang_TFM} and Multi-Marginal Flow Matching (MMFM)~\citep{12_rohbeck_FM}. We benchmark SplineFlow against SF2M~\citep{20_tong_SF2M} for SDE and, additionally, against Minibatch-OT Flow Matching (MOTFM)~\citep{11_tong_MOTFM} for transcriptomics datasets. 
For ODE evaluation, we compute the mean squared error (MSE) between the predicted trajectories and the ground truth. For SDE, we use several metrics, including the 2-Wasserstein distance ($\mathcal{W}_{2}$), maximum mean discrepancy (MMD), and energy distance (Energy), which measure statistical fidelity, as well as MSE on SDE and PODE (probability flow ODE) trajectories. For all experiments, we report average statistics over $5$ runs.

\subsection{ODE Dynamics}
Table \ref{tab:ode_combined_all_datasets} summarizes the comparisons for ODE dynamics, showing that SplineFlow outperforms the baselines in most settings. Compared to adjoint-based Neural ODE and Latent ODE, simulation-free methods (TFM, MMFM, and SplineFlow) are at least on par in MSE with regularly sampled observations ($p=0$), if not better, while being significantly (over $100$x) faster (see Appendix~\ref{apdx:runtime_complexity} for detailed runtime analyses of these algorithms). 
Higher-degree interpolants in SplineFlow capture the underlying dynamics much better than linear interpolants in TFM, especially for nonlinear and oscillatory systems, such as Lotka--Volterra and the harmonic family. We discuss the experiments on the chaotic Lorenz systems in Section \ref{sec:lorenz}, and the ablations of B-spline degree in Section \ref{sec:further analysis}.

% Note that the Lorenz system is known to be chaotic, highly different from other dynamical systems under both deterministic ODE and stochastic SDE settings; therefore, we discuss our empirical findings separately for the Lorenz system in Section \ref{sec:lorenz}.
For irregularly sampled observations with $p>0$, TFM with its linear interpolants performs poorly compared to SplineFlow, as observed not only from the metrics in Table \ref{tab:ode_combined_all_datasets} but also from the visualizations of the evolved dynamics in Appendix \ref{apdx: viz_dynamics}. Higher-degree spline bases tend to better capture the underlying behavior under irregular sampling regimes. 
In scenarios where the underlying dynamics behave approximately linearly, such as in HopperPhysics, MMFM (owing to its nonlinear interpolants) performs poorly compared to SplineFlow, which is on par with TFM. This is expected, since linear interpolants are a special case of B-spline interpolants with $m=1$ as shown in Theorem \ref{thm: theorem_2}. Figure \ref{fig: heatmap_ode} in Appendix \ref{apdx:expanded_results} further visualizes the relative improvement of SplineFlow in a heatmap.
\subsection{SDE Dynamics}
For SDE dynamical systems with \textit{additive constant diffusion}, we observe that constructing conditional probability paths with a constant variance scheme $\sigma_{t}(\bm{z})=\sigma$ performs better than the piecewise quadratic scheme as introduced in Equation \ref{eq:quadratic_variance}. Table~\ref{tab:sde_results_condensed} summarizes the comparison results across different evaluation metrics between SplineFlow implemented with the constant scheme and the baseline $\text{SF2M}$.
For regularly sampled observations, higher-degree B-spline interpolants used to model the conditional probability paths are particularly beneficial for nonlinear Lotka--Volterra and oscillatory damped harmonic systems, which aligns with Theorem \ref{thm: theorem_1}. For irregularly sampled observations, SplineFlow outperforms the baseline for Lotka--Volterra, whereas $\text{SF2M}$ with its linear interpolants is the better-performing model in other cases. 
Due to space limit, we present full SDE comparison results with $p\in\{0.25, 0.75\}$ and across other metrics in Table \ref{tab:apdx_sde_irregular} of Appendix \ref{apdx:expanded_results}.
% Note that this aligns with our expectations, since B-splines of order $m=1$ reduce to linear interpolants, suggesting that SplineFlow can still effectively be applied to simpler SDE dynamics. 

% new table
\begin{table}[t]
\centering
\caption{Comparisons of SplineFlow and SF2M for modeling SDE dynamics across multiple datasets and metrics, where $p$ denotes the degree of irregularity. The best results are highlighted in bold.}
\vspace{0.05in}
\centering
% \small
\setlength{\tabcolsep}{10pt}
\label{tab:sde_results_condensed}
\resizebox{0.98\linewidth}{!}{
\begin{tabular}{lllccc}
\toprule
\textbf{Dataset} & $\bm{p}$ & \textbf{Model} & \textbf{MSE (SDE)} & $\bm{\mathcal{W}_2}$ & \textbf{MMD} \\
\midrule
\multirow{4}{*}{Exp-Decay}
& \multirow{2}{*}{$0$}
& SF2M
& $0.477 \pm 0.010$
& $\mathbf{0.305 \pm 0.005}$
& $\mathbf{0.024 \pm 0.001}$ \\
& & SplineFlow ($m=1$)
& $\mathbf{0.456 \pm 0.008}$
& $0.331 \pm 0.003$
& $0.027 \pm 0.002$ \\
\cmidrule(lr){2-6}

& \multirow{2}{*}{$0.5$}
& SF2M
& $\mathbf{0.524 \pm 0.011}$
& $0.204 \pm 0.012$
& $0.009 \pm 0.002$ \\
& & SplineFlow ($m=1$)
& $0.530 \pm 0.007$
& $\mathbf{0.195 \pm 0.001}$
& $\mathbf{0.007 \pm 3.3{e}{-4}}$ \\

\midrule
\multirow{4}{*}{Damped Harmonic}
& \multirow{2}{*}{$0$}
& SF2M
& $\mathbf{0.967 \pm 0.033}$
& $0.430 \pm 0.007$
& $0.048 \pm 0.003$ \\
& & SplineFlow ($m=1$)
& $0.970 \pm 0.009$
& $\mathbf{0.410 \pm 0.004}$
& $\mathbf{0.042 \pm 0.002}$ \\
\cmidrule(lr){2-6}

& \multirow{2}{*}{$0.5$}
& SF2M
& $\mathbf{0.973 \pm 0.025}$
& $\mathbf{0.319 \pm 0.007}$
& $\mathbf{0.022 \pm 0.003}$ \\
& & SplineFlow ($m=1$)
& $0.981 \pm 0.040$
& $0.321 \pm 0.028$
& $0.022 \pm 0.006$ \\

\midrule
\multirow{4}{*}{Lotka--Volterra}
& \multirow{2}{*}{$0$}
& SF2M
& $0.461 \pm 0.097$
& $0.443 \pm 0.050$
& $0.125 \pm 0.024$ \\
& & SplineFlow ($m=2$)
& $\mathbf{0.273 \pm 0.014}$
& $\mathbf{0.294 \pm 0.003}$
& $\mathbf{0.042 \pm 0.003}$ \\
\cmidrule(lr){2-6}

& \multirow{2}{*}{$0.5$}
& SF2M
& $0.355 \pm 0.031$
& $0.375 \pm 0.037$
& $0.080 \pm 0.029$ \\
& & SplineFlow ($m=3$)
& $\mathbf{0.285 \pm 0.025}$
& $\mathbf{0.310 \pm 0.012}$
& $\mathbf{0.044 \pm 0.004}$ \\

\midrule
\multirow{2}{*}{Lorenz}
& \multirow{2}{*}{$0$}
& SF2M
& $2.159 \pm 0.008$
& $0.632 \pm 0.008$
& $0.130 \pm 0.003$ \\
& & SplineFlow ($m=4$)
& $\mathbf{1.227 \pm 0.056}$
& $\mathbf{0.180 \pm 0.015}$
& $\mathbf{0.005 \pm 0.002}$ \\

\bottomrule
\end{tabular}
}
\vspace{-0.05in}
\end{table}

% \begin{figure*}[t]
% \centering
% \begin{subfigure}[t]{0.235\linewidth}
%     \centering
%     \includegraphics[width=\linewidth]{figures/lorenz/lorenz_ode_baseline.png}
%     \caption{TFM (ODE)}
%     \label{fig:lorenz_ode_baseline}
% \end{subfigure}\hfill
% \begin{subfigure}[t]{0.235\linewidth}
%     \centering
%     \includegraphics[width=\linewidth]{figures/lorenz/lorenz_ode_splineflow.png}
%     \caption{SplineFlow (ODE)}
%     \label{fig:lorenz_ode_splineflow}
% \end{subfigure}
% \begin{subfigure}[t]{0.235\linewidth}
%     \centering
%     \includegraphics[width=\linewidth]{figures/lorenz/lorenz_sde_baseline.png}
%     \caption{SF2M (SDE)}
%     \label{fig:lorenz_sde_baseline}
% \end{subfigure}\hfill
% \begin{subfigure}[t]{0.235\linewidth}
%     \centering
%     \includegraphics[width=\linewidth]{figures/lorenz/lorenz_sde_splineflow.png}
%     \caption{SplineFlow (SDE)}
%     \label{fig:lorenz_sde_splineflow}
% \end{subfigure}
% \vspace{-0.05in}
% \caption{Visualizations of the ODE and SDE Lorenz trajectories predicted by different methods. Compared with baselines, SplineFlow learns the underlying structural properties of the Lorenz system more effectively.}
% \label{fig:lorenz_comparison}
% \end{figure*}

% 
\subsection{Lorenz Systems}
\label{sec:lorenz}

% 

% Note that the Lorenz system is known to be chaotic, highly different from other dynamical systems under both deterministic ODE and stochastic SDE settings; therefore, we discuss our empirical findings separately for the Lorenz system in Section \ref{sec:lorenz}.

Chaotic dynamical systems exhibit distinctive evolution dynamics, with drastic shifts in their states at certain timepoints and highly nonlinear trajectories. For Lorenz, in particular, depending on the initial value, the state may converge onto one attractor basin for a period before shifting to another attractor basin, known as the \textit{Butterfly Effect}. This behavior makes the system sensitive to masking and irregular sampling, since the state values at certain critical time points contain information about which attractor the trajectory will approach. Therefore, we restrict the configurations on Lorenz to regular sampling ($p = 0$).
Our experiments show that higher-degree SplineFlow not only outperform TFM for the ODE case and $\text{SF2M}$ for SDE in terms of metric performance (Tables~\ref{tab:ode_combined_all_datasets} and ~\ref{tab:sde_results_condensed}), but also enable SplineFlow to capture the underlying structure of the dynamics much better than the baselines, including the adjoint-based methods (Figures \ref{fig:lorenz_ode_baseline}, \ref{fig:lorenz_ode_splineflow} and \ref{fig:lorenz_ode_splineflow_img2}–\ref{fig:lorenz_sde_splineflow_img2}). This aligns with our theoretical motivation for using SplineFlow to model nonlinear and higher-degree dynamical systems.

% % old table 4
% \begin{table}[H]
% \centering
% \caption{Performance for Lorenz System under regular sampling $p=0$. ODE trajectories are evaluated using MSE, while SDE trajectories are evaluated using MSE (SDE) , Wasserstein distance, and MMD. The MSE for SDE setting corresponds to the Probability flow ODE generated trajectory.}
% \label{tab:lorenz_ode_sde}
% \resizebox{\linewidth}{!}{
% \begin{tabular}{llcccc}
% \toprule
% Setting & Model & MSE & MSE (SDE) & Wasserstein & MMD \\
% \midrule
% \multirow{4}{*}{ODE}
% & NODE
% & $1.491 \pm 0.359$ & -- & -- & -- \\
% & LatentODE
% & $1.157 \pm 0.204$ & -- & -- & -- \\
% & TFM
% & $2.442 \pm 0.076$ & -- & -- & -- \\
% & SplineFlow
% & $\mathbf{0.639 \pm 0.004}$ & -- & -- & -- \\

% \midrule
% \multirow{2}{*}{SDE}
% & SF2M
% & $2.170 \pm 0.008$
% & $2.159 \pm 0.008$
% & $0.632 \pm 0.008$
% & $0.130 \pm 0.003$ \\
% & SplineFlow
% & $\mathbf{1.232 \pm 0.040}$
% & $\mathbf{1.227 \pm 0.056}$
% & $\mathbf{0.180 \pm 0.015}$
% & $\mathbf{0.005 \pm 0.002}$ \\
% \bottomrule
% \end{tabular}
% }
% \end{table}
% 

\begin{figure*}[t]
\centering

\begin{subfigure}[t]{0.32\linewidth}
    \centering
    \includegraphics[width=\linewidth, height=0.8\linewidth]{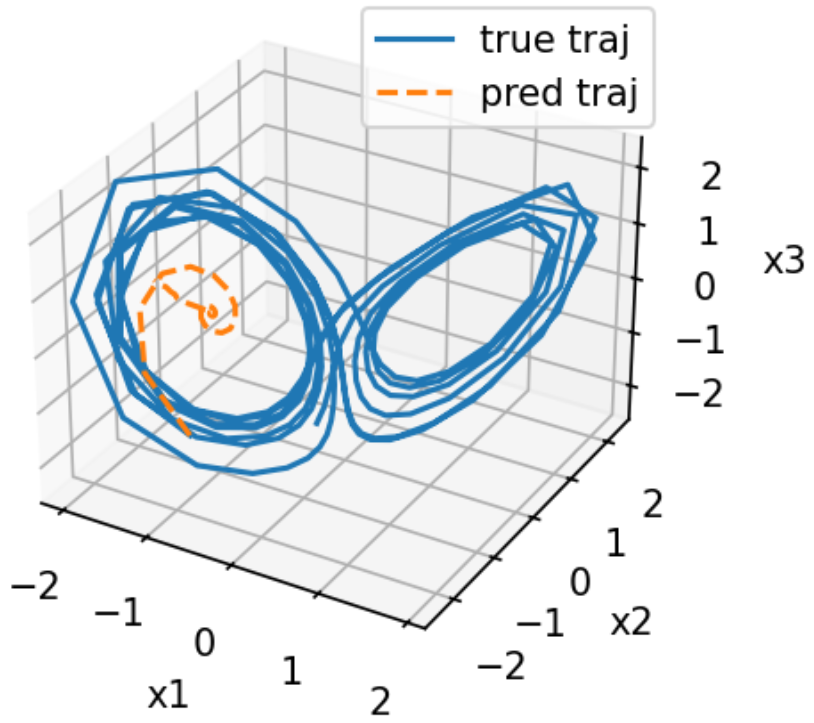}
    \caption{TFM}
    \label{fig:lorenz_ode_baseline}
\end{subfigure}\hfill
\begin{subfigure}[t]{0.32\linewidth}
    \centering
    \includegraphics[width=\linewidth, height=0.8\linewidth]{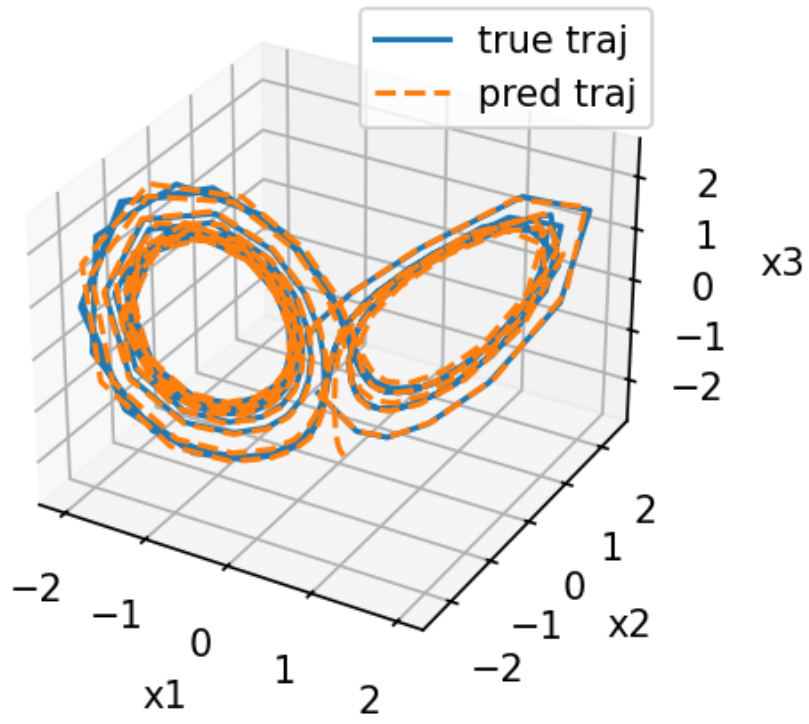}
    \caption{SplineFlow}
    \label{fig:lorenz_ode_splineflow}
\end{subfigure}\hfill
\begin{subfigure}[t]{0.32\linewidth}
    \centering
    \includegraphics[width=\linewidth, height=0.8\linewidth]{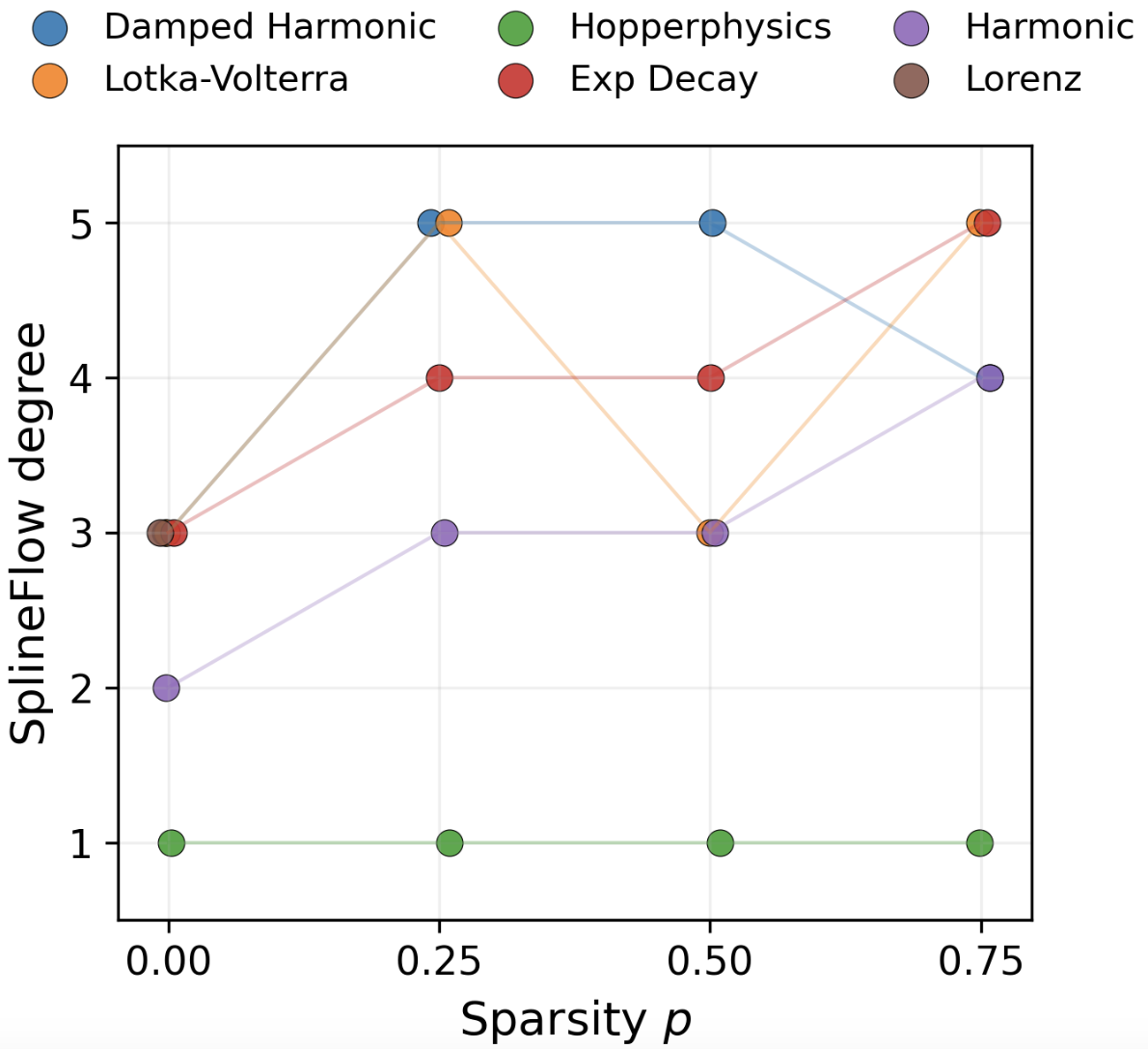}
    \caption{Degree vs. Irregularity}
    \label{fig:degree_m_vs_p}
\end{subfigure}
% \hfill
% \begin{subfigure}[t]{0.235\linewidth}
%     \centering
%     \includegraphics[width=\linewidth]{figures/degree_correlation.png}
%     \caption{CV-selected vs.\ optimal $m$}
%     \label{fig:degree_cv_vs_sf}
% \end{subfigure}
\vspace{-0.05in}
\caption{(a) and (b) compare the Lorenz ODE trajectories learned by SplineFlow with TFM, whereas (c) presents the ablations of spline degree $m$ on ODE dynamics with varying sampling irregularity $p$.
% , and simple cross-validation reliably recovers near-optimal choices with Spearman correlation $\rho=0.85$.
}
\vspace{-0.05in}
\label{fig:lorenz_and_degree_combined}
\end{figure*}

\section{Further Analysis}
\label{sec:further analysis}

\shortsection{Selection of Spline Degree}
We analyze the role of spline degree $m$ under varying levels of sampling sparsity.
Figure~\ref{fig:degree_m_vs_p} shows that higher spline degrees become increasingly beneficial as $p$ increases for oscillatory and nonlinear ODE systems.
This observation is expected, as higher-degree B-splines better capture curvature and
higher-order dynamics when observations are limited. In contrast, for approximately linear dynamics (e.g., HopperPhysics), lower degrees ($m=1$) suffice and perform best, consistent with the fact that linear
interpolants are a special case of B-splines (Theorem~\ref{thm: theorem_2}). Full ablation results are deferred to Appendix \ref{apdx:ablations_all}. These results
highlight the importance of adapting the supervision signal in flow matching
to the underlying data complexity. Figure~\ref{fig:degree_cv_vs_sf} in Appendix~\ref{apdx:hyperparameter_selection} shows that a simple
cross-validation procedure produces estimates strongly correlated with the optimal degree
(Spearman correlation $\rho = 0.85$), suggesting a practical mechanism for selecting the spline degree parameter $m$ in SplineFlow without extensive hyperparameter tuning. We further report the SplineFlow performance when using the cross-validation selection rule ($m_{\text{CV}}$) in Appendix~\ref{apdx:hyperparameter_selection}.
% \begin{figure*}[t]
% \centering
% \begin{subfigure}[t]{0.4\linewidth}
%     \centering
%     \includegraphics[width=\linewidth]{figures/degree_vs_p_ode.png}
%     \caption{Optimal SplineFlow degree $m$ vs. Sparsity $p$}
%     \label{fig:degree_m_vs_p}
% \end{subfigure}\hfill
% \begin{subfigure}[t]{0.4\linewidth}
%     \centering
%     \includegraphics[width=\linewidth]{figures/degree_correlation.png}
%     \caption{Optimal SplineFlow degree $m$ vs. Optimal cross validated degree $m$}
%     \label{fig:degree_cv_vs_sf}
% \end{subfigure}
% \vspace{-0.05in}
% \caption{Influence of $m$. (a) Higher degrees help under sparsity. (b) Cross-validation recovers optimal $m$.}
% \label{fig:degree_m_influence}
% \end{figure*}

\begin{table}[t]
\caption{Comparison results of SplineFlow on transcriptomic datasets with PHATE embedding.}
\vspace{0.05in}
\label{tab:transcriptomic_data_results_temporal}
\centering
\setlength{\tabcolsep}{10pt}
\resizebox{0.98\textwidth}{!}{
\begin{tabular}{lll ccc}
\toprule
\textbf{Dataset} & \textbf{Task} & \textbf{Model}
& \textbf{MMD} & $\bm{\mathcal{W}_2}$ & \textbf{Energy} \\
\midrule

\multirow{6}{*}{Brain Regeneration} 
&
\multirow{3}{*}{Inter.}
& MOTFM
& $0.503 \pm 0.016$
& $0.794 \pm 0.026$
& $1.440 \pm 0.062$ \\
& & SF2M
& $0.457 \pm 0.008$
& $0.753 \pm 0.026$
& $1.310 \pm 0.036$ \\
& & SplineFlow $(m=5)$
& $\mathbf{0.285 \pm 0.007}$
& $\mathbf{0.607 \pm 0.027}$
& $\mathbf{0.815 \pm 0.021}$ \\

\cmidrule{2-6}
& \multirow{3}{*}{Extra.}
& MOTFM
& $1.010 \pm 0.012$
& $1.570 \pm 0.005$
& $3.890 \pm 0.051$ \\
& & SF2M
& $0.972 \pm 0.009$
& $\mathbf{1.530 \pm 0.047}$
& $3.770 \pm 0.072$ \\
& & SplineFlow $(m=4)$
& $\mathbf{0.955 \pm 0.009}$
& $1.536 \pm 0.019$
& $\mathbf{3.740 \pm 0.042}$ \\

\midrule
\multirow{6}{*}{Embryoid Evolution} &
\multirow{3}{*}{Inter.}
& MOTFM
& $0.060 \pm 0.003$
& $0.378 \pm 0.011$
& $0.175 \pm 0.005$ \\
& & SF2M
& $0.025 \pm 0.004$
& $0.272 \pm 0.024$
& $0.086 \pm 0.013$ \\
& & SplineFlow $(m=3)$
& $\mathbf{0.013 \pm 0.001}$
& $\mathbf{0.207 \pm 0.017}$
& $\mathbf{0.040 \pm 0.001}$ \\

\cmidrule{2-6}
& \multirow{3}{*}{Extra.}
& MOTFM
& $0.083 \pm 0.009$
& $0.448 \pm 0.051$
& $0.239 \pm 0.007$ \\
& & SF2M
& $\mathbf{0.058 \pm 0.011}$
& $\mathbf{0.439 \pm 0.044}$
& $\mathbf{0.194 \pm 0.016}$ \\
& & SplineFlow $(m=2)$
& $0.069 \pm 0.002$
& $0.448 \pm 0.044$
& $0.207 \pm 0.030$ \\

\bottomrule
\end{tabular}
}
\vspace{-0.05in}
\end{table}

\shortsection{Cellular Dynamics}
We evaluate SplineFlow on a Stereo-seq spatial transcriptomics dataset \citep{39_stereoseq_wei} from a temporal study of post-traumatic brain regeneration in an Axolotl Telencephalon (Salamander) tissue, as well as on an unpaired transcriptomic Embryoid developmental dataset~\citep{38_ebdata_burkhardt}. 
% Following trauma or injury, cells undergo a complex molecular process whose underlying dynamics lead to the emergence of different cell types over time as the brain heals, making this a dynamic system of interest. 
Both datasets contain unpaired transcriptomic expression profiles across $5$ distinct stages, with replicates collected from different individual organisms at each stage, and we consider both interpolation ($t = 3$) and extrapolation ($t = 5$) settings. For SplineFlow, we use static pairwise entropic OT couplings to generate trajectories 
% $X_{i} = [x_i(1), x_i(2), x_i(3), x_i(4), x_i(5)]$ 
to define the conditional probability paths. Since SDE modeling of cellular processes has been shown to outperform ODE modeling, as used in MOTFM \citep{11_tong_MOTFM}, we adopt the stochastic version of SplineFlow.

Table~\ref{tab:transcriptomic_data_results_temporal} shows that for both datasets and across interpolation and extrapolation settings, SplineFlow and $\text{SF2M}$ outperform MOTFM, highlighting the importance of stochastic modeling of cellular processes. Moreover, SplineFlow outperforms $\text{SF2M}$ quite comfortably in the interpolation setting, while being best overall in the extrapolation setting for the Stereo-seq dataset. PHATE embedding captures geometric nonlinear structures in transcriptomic data, thereby making SplineFlow better equipped than $\text{SF2M}$ to model temporal processes more accurately. 
We observe similar trends for experiments with PCA embedding, which contains linear projections unlike PHATE embedding (see Tables~\ref{tab:apdx_stereoseq_condensed} and \ref{tab:apdx_citeseq_condensed} in Appendix \ref{apdx:expanded_results} for detailed comparison results).
For the Embryoid Evolution dataset, Table~\ref{tab:transcriptomic_data_results_temporal} shows that SplineFlow outperforms MOTFM and $\text{SF2M}$ under the interpolation setting, while remaining competitive under the extrapolation setup.

\section{Conclusion}

We introduced SplineFlow, a simulation-free flow matching algorithm that constructs conditional paths using B-spline interpolants. SplineFlow provides a scalable and flexible framework for modeling complex dynamical systems, with potential applications ranging from faster data-driven simulations to modeling temporal biological processes, such as longitudinal drug treatment response. 
% SplineFlow improves performance compared to existing baselines across a range of ODE and SDE systems, while better capturing qualitative structural properties. 
% Beyond classical dynamical systems, SplineFlow also improves the modeling of latent temporal processes driving longitudinal transcriptomic changes. 
Future work may further improve upon this by learning optimal knot placement during training and by extending the framework to model stochastic dynamics with state-dependent diffusion terms.

\section*{Availability}
\label{sec:availability}

Our code is available at: \url{https://github.com/santanurathod/SplineFlow}.

%%%% submit version 1
% We introduced SplineFlow, a flow-matching algorithm for modeling dynamical systems by building conditional probability paths using B-splines. SplineFlow provides a theoretically principled way to model nonlinear, higher-degree, and irregularly sampled dynamics using a highly scalable simulation-free framework. SplineFlow improves metric performance compared to existing baselines across a range of ODE and SDE systems, while better capturing qualitative structural properties. Beyond classical systems, SplineFlow also improves modeling of latent temporal processes driving longitudinal transcriptomic changes. Future works may further improve performance by learning optimal knot placement during training and by extending the framework scalably to model stochastic dynamics with state dependent diffusion terms. 

\nocite{langley00}

% \section*{Impact Statement}
% % This paper presents work whose goal is to advance the field of Machine Learning. There are many potential societal consequences of our work, none which we feel must be specifically highlighted here.
% Our work contributes a general framework for data-driven modeling of dynamical systems from observational data, without requiring explicit mechanistic formulations. The proposed algorithm is particularly relevant in settings where the underlying dynamics are complex, the data are partially observed, and analytical descriptions are unavailable. The proposed framework can serve as a foundation for a range of applications, from financial market modeling to biomedical settings involving longitudinal processes, for example, analyzing biological dynamics from unpaired datasets. Future work should focus on real-world validation of model predictions in collaboration with domain experts to better assess their reliability and practical usefulness.

\bibliography{citations}
\bibliographystyle{plainnat}

%%%%%%%%%%%%%%%%%%%%%%%%%%%%%%%%%%%%%%%%%%%%%%%%%%%%%%%%%%%%%%%%%%%%%%%%%%%%%%%
%%%%%%%%%%%%%%%%%%%%%%%%%%%%%%%%%%%%%%%%%%%%%%%%%%%%%%%%%%%%%%%%%%%%%%%%%%%%%%%
% APPENDIX
%%%%%%%%%%%%%%%%%%%%%%%%%%%%%%%%%%%%%%%%%%%%%%%%%%%%%%%%%%%%%%%%%%%%%%%%%%%%%%%
%%%%%%%%%%%%%%%%%%%%%%%%%%%%%%%%%%%%%%%%%%%%%%%%%%%%%%%%%%%%%%%%%%%%%%%%%%%%%%%

\appendix
\onecolumn

\section{Preliminaries}
\label{sec:preliminaries}

\subsection{Flow Matching}
\label{section: Flow Matching}

Flow matching (FM)~\citep{8_lipman_FM} is a technique that characterizes continuous normalizing flows for a system of variables. Let $\psi : [0,1] \times \mathbb{R}^{d} \rightarrow \mathbb{R}^{d}$ be the underlying continuous flow that acts on random variables at each time $t \in [0,1]$, which transforms samples $x_{0}$ from the source distribution $q_{0} = q(x_{0})$ into samples $x_{1}$ from the target distribution $q_{1} = q(x_{1})$, so that $\psi(1, x_{0}) = x_{1}$. In continuous normalizing flows, $\psi$ can be defined by an ordinary differential equation (ODE): 
\begin{align}
    \dfrac{d \psi(t, x_{0})}{dt} = \dfrac{dx_{t}}{dt} = u_{t}(x_{t}),
\end{align}
where $u_{t}(x) : [0,1] \times \mathbb{R}^{d} \rightarrow \mathbb{R}^{d}$ is the velocity field describing the rate of change of the random variables. Transforming the random variable $x_{t}$ also transforms the associated probability distribution $p_{t}(x) : [0,1] \times \mathbb{R}^{d} \rightarrow \mathbb{R}_{+}$, which satisfies the boundary conditions $p_{t=0} = q(x_{0})$ and $p_{t=1} = q(x_{1})$. The relationship between the velocity field and its induced probability distribution at time $t$ is given by the well-known \emph{continuity equation}~\citep{19_villani_OTtextbook}:
 \begin{align}
 \label{eq: continuity equation}
\frac{\partial p_t(x)}{\partial t} = - \nabla_{x} \cdot \big(u_t(x) p_t(x)\big),
 \end{align}
where $\nabla_{x}$ denotes the divergence operator. Flow matching approximates the velocity field by a neural network $u_{\theta}: [0,1]\times\mathbb{R}^{d} \rightarrow \mathbb{R}^{d}$ by minimizing a regression loss $\mathcal{L}_{\text{FM}}$:
\begin{align}
\label{eq: flow matching regression loss}
    \min_{\theta} \: \mathbb{E}_{t \sim \mathcal{U}(0,1), x \sim p_t(x)} \big\| u_\theta(t, x) - u_t(x) \big\|^2,
\end{align}
where $\mathcal{U}(0,1)$ denotes the uniform distribution over $[0,1]$, and $p_{t}$ is the time-varying probability distribution induced by the velocity field $u_{t}$. Once trained, the velocity field $u_{\theta}$ can be integrated from $0$ to $1$ to transform a source sample $x_{0} \sim q_{0}$ into a target sample $x_{1} \sim q_{1}$. In practice, however, the velocity field $u_{t}(x)$ is unknown, and several velocity fields can exist such that the boundary conditions for the induced probability paths are satisfied.

\subsection{Conditional Flow Matching}
\label{section: Conditional Flow Matching}

To tackle the intractability of the flow matching objective $\mathcal{L}_{\text{FM}}$ in Equation \ref{eq: flow matching regression loss}, conditional flow matching (CFM) was introduced~\citep{8_lipman_FM} , where the velocity field $u_{t}(x)$ and the induced probability path $p_{t}(x)$ are defined as:
\begin{align}
\label{eq: conditional probability path}
p_t(x) = \int p_t(x  |  z) q(z) dz, \:\: u_t(x) = \int u_t(x  |  z) \frac{p_t(x  |  z) q(z)}{p_t(x)} dz.
\end{align}
Here, $z$ is the conditional variable, and $u_{t}(x | z)$ denotes the conditional velocity field such that the induced conditional probability paths $p_{t}(x | z)$ satisfy the boundary conditions: $\int p_{t=0}(x | z) q(z) dz = q_{0}$ and $\int p_{t=1}(x | z) q(z) dz = q_{1}$. 
The CFM regression loss $\mathcal{L}_{\text{CFM}}$ is defined as:
\begin{align}
\label{eq: conditional flow matching loss}
    \min_{\theta} \: \mathbb{E}_{t \sim \mathcal{U}(0,1), z \sim q(z), x \sim p_t(x  |  z)} \big\| u_\theta(t, x) - u_t(x  |  z) \big\|^2.
\end{align}
\citet{8_lipman_FM} proved that $\nabla_{\theta} \mathcal{L}_{\text{FM}}=\nabla_{\theta} \mathcal{L}_{\text{CFM}}$, suggesting that we can regress the velocity field using the more tractable conditional velocities. Therefore, the remaining task is to define an appropriate conditional velocity field $u_{t}(x  |  z)$ and the associated $p_{t}(x  |  z)$ to sample from.

In the existing literature, the conditional path $p_t(x | z)$ is typically modeled using a Gaussian distribution with $\mu_{t}(z)$ and $\sigma_{t}(z)$ being the constructed, latent-dependent, time-varying mean and standard deviation, such that the distribution at a sampled time point matches the observations. For such a particular choice, the associated conditional velocity $u_t(x | z)$ can be expressed in an analytical form:
\begin{align}
\label{eq: gaussian probability path}
p_t(x | z) = \mathcal{N}(x \: | \: \mu_t(z), \sigma_t(z)^2), \:\:
u_t(x | z) = \frac{\sigma_t'(z)}{\sigma_t(z)} (x - \mu_t(z)) + \mu_t'(z).
\end{align}

\subsection{Flow and Score Matching}
\label{section: Flow And Score Matching}

% So far, we've discussed how to characterize the velocity field when the underlying dynamics are governed by an ODE. 
% $dx_{t} = u_{t}(x_{t})dt$, such that the samples generated from this flow resemble the observed data. 
In certain scenarios, the underlying dynamics of the state random variables $x_{t}$ can evolve stochastically, corresponding to a stochastic differential equation (SDE):
\begin{align}
\label{eq: SDE}
dx_{t}= u_{t}(x_{t})dt + g(t)dw_{t},
\end{align}
where $u_{t}(x): [0,1] \times \mathbb{R}^{d} \rightarrow \mathbb{R}^{d}$ is the velocity field, $g(t): [0,1] \rightarrow \mathbb{R}_{>0}$ is a positive diffusion function, and $w_{t}$ stands for the standard Wiener process. The relationship between $u_{t}(x)$ and the SDE induced probability path $p_{t}(x): [0,1]\times\mathbb{R^{d}}\rightarrow\mathbb{R}_{+}$ is given by a \emph{Fokker-Planck equation} (FPE): 
\begin{align}
\label{eq: FPE}
\frac{\partial p_{t}(x)}{\partial t}= -\nabla_{x} \cdot (u_{t}(x)p_{t}(x))+\dfrac{g^{2}(t)}{2}\nabla^{2}_{x}(p_{t}(x)).
\end{align}
\citet{5_song_scorebaseddiffusion} showed that starting from some initial distribution $p_{0}(x)$ and evolving the random variable according to the probability flow ODE with drift $u^{o}_{t}(x_{t})$:
\begin{align}
\label{eq: probability flow ODE}
    d x_{t}&= \Big[u_{t}(x_{t})- \dfrac{g^{2}(t)}{2}\nabla_{x}\log(p_{t}(x_{t}))\Big]dt= u^{o}_{t}(x_{t}) dt,
\end{align}
induces the same time varying marginal distribution $p_{t}(x)$ as the SDE defined by Equation \ref{eq: SDE}. Therefore, given a diffusion schedule $g(t)$, the score function $\nabla_{x}\log(p_{t}(x_{t}))$, and the drift $u^{o}_{t}(x_{t})$ defined by Equation \ref{eq: probability flow ODE}, we can recover the drift for the original SDE in Equation \ref{eq: SDE} as:
\begin{align}
\label{eq: original SDE drift}
    u_{t}(x_{t})= u^{o}_{t}(x_{t})+\dfrac{g^{2}(t)}{2}\nabla_{x}\log(p_{t}(x_{t})).
\end{align}
Analogous to conditional flow matching, the drift $u^{o}_{t}(x)$ and the score $\nabla_{x}\log(p_{t}(x))$ can be approximated by neural networks optimizing the following regression loss over their conditional equivalents (Theorem 3.2 in \citet{20_tong_SF2M}):
\begin{equation}
\begin{aligned}
\label{eq: SF2M loss}
&\mathcal{L}_{\text{SF2M}}(\theta, \phi)
=
\mathbb{E}_{t, z, x}
\Big[
\| u_\theta(t,x) - u^{o}_t(x  |  z) \|^2 + \lambda(t)^2
\| s_\phi(t, x) - \nabla \log p_t(x  |  z) \|^2
\Big].
\end{aligned}
\end{equation}
Here, $u^{o}_{t}(x |  z)$ induces the conditional probability $p_{t}(x  |  z)$ and the expectation is taken over uniformly distributed time $t \sim \mathcal{U}(0,1)$, the latent variables $z \sim q(z)$ and the conditional probability samples $x \sim p_t(x  |  z)$, and $\lambda(t)$ denotes the time conditional weights.  
When the conditional probability path is constructed as a linear Brownian bridge: 
\begin{align}
p_{t}(x | z) = \mathcal{N} (x; (1-t)x_{0}+tx_{1}, \:\: \sigma^{2}t(1-t)),
\end{align}
where $z= (x_{0}, x_{1})$ is the conditional variable. From Equation \ref{eq: gaussian probability path}, we can compute:
\begin{align}
    \nonumber u^{o}_{t}(x| z)= \frac{1 - 2t}{t(1 - t)}\big(x - t x_1 - (1 - t) x_0 \big)+ x_1 - x_0, \:\: \nabla_{x} \log p_t(x)= \frac{t x_1 + (1 - t) x_0 - x}{\sigma^2 t(1 - t)}.
\end{align}

\section{Polynomial Interpolants}
\label{apdx: Polynomial Interpolants}

% \begin{table}[H] 
% \centering 
% \small 
% \setlength{\tabcolsep}{5pt} 
% \begin{tabular}{lp{2.2cm}ccp{2.2cm}p{2.2cm}} 
% \toprule 
% Interpolant & Mathematical Form & Locality & Smoothness & Derivatives & Key Limitation \\ \midrule 
% Linear & Piecewise linear & Local & $C^0$ & Discontinuous & Cannot model smooth or higher-order dynamics \\ 
% Cubic & Piecewise cubic & Global & $C^2$ & Continuous & Sensitive to boundary conditions, reduced locality \\ 
% Lagrangian & Global polynomial & Global & Smooth & Unstable & Runge's phenomenon for large $N$ \\ 
% Chebyshev & Polynomial in $T_k$ basis & Global & Smooth & stable than lagrangian & Requires structured sampling, unsuitable for irregular data \\ 
% Hermite & Value + Derivative interpolation & Local & $C^{1+}$ & Noise-sensitive & Requires accurate derivatives \\ 
% B-spline & $\sum_k c_k B_{k,m}$ & Local & Degree-controlled & Stable & -- \\ 
% \bottomrule 
% \end{tabular} 
% \caption{Comparison of interpolation operators and their structural properties.} \label{tab:interpolation_comparison} 
% \end{table}

Let $f : [a,b] \to \mathbb{R}^d$ be an unknown function observed at samples
$\{(t_i, x_i)\}_{i=0}^N$, where $x_i = f(t_i)$ and
$a \le t_0 < \cdots < t_N \le b$.
An interpolation operator $\mathcal{I}$ constructs a continuous-time function
$\mathcal{I}f$ such that $(\mathcal{I}f)(t_i) = x_i$ for all $i$. Below, give formal definitions for different kinds of polynomial interpolants, and discuss their structural properties, as well as limitations (see Table \ref{tab:interpolation_comparison} for a summary).

\shortsection{Linear Interpolation}
A linear interpolant is a piecewise linear function defined as:
\begin{equation}
(\mathcal{I}_{\mathrm{lin}} f)(t)
=
x_i \frac{t_{i+1}-t}{t_{i+1}-t_i}
+
x_{i+1} \frac{t-t_i}{t_{i+1}-t_i},
\: \text{ for } t \in [t_i,t_{i+1}].
\end{equation}
It is numerically stable since it depends on local estimations. However, while the resulting interpolant is continuous, its derivatives are discontinuous at sampling points. This makes linear interpolation a suboptimal choice for modeling smooth or higher-order dynamics~\citep{40_ODE_hairer}.

\shortsection{Cubic Interpolation}
A cubic interpolant $s(t)$ is a polynomial of degree three satisfying
\begin{equation}
s(t_i) = x_i,
\qquad
s \in C^2([a,b]).
\end{equation}
In addition to being a continuous function, the first two derivatives ($s'(t) $ and $ s''(t)$) are also continuous, making it a popular choice for approximating smooth functions. However, the construction of $s(t)$ depends on all the sampled points ($x_{0}, \cdots, x_{N}$) to ensure the continuity of its first two derivatives or $\mathcal{C}^{2}[a,b]$, making it extremely sensitive to boundary conditions and reducing locality. 

\subsection{Lagrangian Interpolation}
A Lagrangian interpolant is a global polynomial defined as:
\begin{equation}
(\mathcal{I}_{\mathrm{Lag}} f)(t)
=
\sum_{i=0}^N x_i \ell_i(t),
\qquad
\ell_i(t) =
\prod_{j \neq i} \frac{t - t_j}{t_i - t_j}.
\end{equation}
While it is the most natural polynomial satisfying all the boundary conditions, it suffers from Runge's Phenomenon~\citep{14_epperson_rungephenomenon}, characterized by large oscillations and instability, making it a poor choice to interpolate between a large number of points $N$. While piecewise polynomial construction can be done, it comes at the cost of discontinuous derivatives at observational time points.

\begin{table}[t] 
\centering 
\caption{Comparison of different interpolation operators and their structural properties.}
\vspace{0.05in}
\small
\setlength{\tabcolsep}{6pt} 
\renewcommand{\arraystretch}{1.1}
\begin{tabular}{lcccc} 
\toprule 
\textbf{Interpolant} & \textbf{Mathematical Form} & \textbf{Locality} & \textbf{Smoothness} & \textbf{Derivatives} \\ 
\midrule 
Linear & Piecewise linear & Local & $C^0$ & Discontinuous \\ 
Cubic & Piecewise cubic & Global & $C^2$ & Continuous \\ 
Lagrangian & Global polynomial & Global & Smooth & Unstable  \\ 
Chebyshev & Polynomial in $T_k$ basis & Global & Smooth & Stable \\ 
Hermite & Value+Derivative interpolation & Local & $C^{1+}$ & Noise-Sensitive \\ 
B-spline & $\sum_k c_k B_{k,m}$ & Local & Degree-Controlled & Stable \\ 
\bottomrule 
\end{tabular} 
\label{tab:interpolation_comparison} 
\end{table}

% \begin{table}[H] 
% \centering 
% \small 
% \setlength{\tabcolsep}{5pt} 
% \begin{tabular}{lp{2.2cm}ccp{2.2cm}p{2.2cm}} 
% \toprule 
% Interpolant & Mathematical Form & Locality & Smoothness & Derivatives & Key Limitation \\ \midrule 
% Linear & Piecewise linear & Local & $C^0$ & Discontinuous & Cannot model smooth or higher-order dynamics \\ 
% Cubic & Piecewise cubic & Global & $C^2$ & Continuous & Sensitive to boundary conditions, reduced locality \\ 
% Lagrangian & Global polynomial & Global & Smooth & Unstable & Runge's phenomenon for large $N$ \\ 
% Chebyshev & Polynomial in $T_k$ basis & Global & Smooth & stable than lagrangian & Requires structured sampling, unsuitable for irregular data \\ 
% Hermite & Value + Derivative interpolation & Local & $C^{1+}$ & Noise-sensitive & Requires accurate derivatives \\ 
% B-spline & $\sum_k c_k B_{k,m}$ & Local & Degree-controlled & Stable & -- \\ 
% \bottomrule 
% \end{tabular} 
% \caption{Comparison of interpolation operators and their structural properties.} \label{tab:interpolation_comparison} 
% \end{table}

\subsection{Chebyshev Interpolation}
Chebyshev interpolation is equivalent to polynomial interpolation at Chebyshev nodes, defined as:
\begin{equation}
t_k = \cos\!\left( \frac{(2k+1)\pi}{2(N+1)} \right),
\end{equation}
which gives us Chebyshev interpolants: 
\begin{equation}
(\mathcal{I}_{\mathrm{Cheb}} f)(t)
=
\sum_{k=0}^N c_k T_k(t),
\end{equation}
where $T_k$ are Chebyshev polynomials of the first kind, defined by
$T_0(x)=1$, $T_1(x)=x$, and $T_{k+1}(x)=2xT_k(x)-T_{k-1}(x)$. Chebyshev Interpolation is known to improve numerical stability compared to other naïve polynomial interpolants, such as the Lagrangian interpolant above~\citep{41_approximationtheory_treftehan}. However, since the interpolant is constructed globally, it's highly unsuitable for irregularly sampled time points as is common~\citep{41_approximationtheory_treftehan}.

\subsection{Hermite Interpolation}
A Hermitian interpolant is a generalization of a Lagrangian interpolant, such that apart from satisfying observational values, the derivatives at the observational points are also satisfied:
\begin{equation}
(\mathcal{I}_{\mathrm{Herm}} f)(t_i) = x_i,
\qquad
(\mathcal{I}_{\mathrm{Herm}} f)'(t_i) = v_i .
\end{equation}
When accurate derivative information is available, Hermitian interpolants are known to produce
smooth and structurally consistent interpolations. However, the quality of the interpolant depends heavily on the availability of derivatives, which are often unavailable. Errors in derivatives lead to poorly fitted interpolants, making them a suboptimal choice~\citep{42_splinefunctions_schumaker}.

\subsection{B-Spline Interpolation}
A B-spline interpolant of degree $m$ is defined as:
\begin{equation}
(\mathcal{I}_{\mathrm{B}} f)(t)
=
\sum_{k} c_k B_{k,m}(t),
\end{equation}
where $\{B_{k,m}\}$ are B-spline basis functions defined over observations via the
Cox--de Boor recursion formula (Equation~\ref{eq: Cox-De Boor Recursion}). Each basis function is locally defined, ensuring its stability, while the smoothness is controlled by the degree $m$. This combination of locality and smoothness control makes B-splines a particularly optimal choice for modeling dynamical functions.

\section{Synthetic Dynamical Systems}
\label{apdx: dynamical_systems_datasets}

\setlength{\tabcolsep}{10pt}
\begin{table}[t]
\caption{Summary of the properties of different dynamical systems considered in our experiments.}
\vspace{0.05in}
\centering
\begin{tabular}{lcccc}
\toprule
System & Linear & Non-linear & Oscillatory & Chaotic \\
\midrule
Exponential Decay        & \checkmark & \xmark & \xmark & \xmark \\
Harmonic Oscillator     & \checkmark & \xmark & \checkmark & \xmark \\
Damped Harmonic Oscillator & \checkmark & \xmark & \checkmark & \xmark \\
Lotka--Volterra          & \xmark & \checkmark & \checkmark & \xmark \\
Lorenz System            & \xmark & \checkmark & \xmark & \checkmark \\
\bottomrule
\end{tabular}
\label{tab:dynamics_summary}
\end{table}

We provide the definitions of all the synthetic dynamical systems evaluated in our experiments. Tables \ref{tab:dynamics_summary} and \ref{tab:synthetic_dynamics_summary} summarize the properties, the corresponding ODE and SDE dynamics, and their use cases.

\subsection{Exponential Decay}
Exponential decay is a simple linear dynamical system, which can be written as the following ODE:
\begin{equation}
\dot x(t) = -\lambda x(t), \qquad \lambda > 0.
\end{equation}
It's often used to model first-order relaxation processes such as radioactive decay or simple stable linear systems~\citep{45_expdecay_citation}.
The SDE equivalent of the exponential decay is known as an Ornstein--Uhlenbeck (OU) process given by:
\begin{equation}
dX_t = -\lambda X_t\,dt + \sigma\,dW_t,
\end{equation}
where $W_{t}$ denotes the standard Wiener process, $\lambda$ is the hyperparameter controlling the rate of decay, and $\sigma$ is the constant exogenous noise rate. 
OU processes are used to model systems such as the velocity of a Brownian particle under friction, gene-expression fluctuations around the steady state~\citep{46_expdecay_citation}, among others.

\subsection{Harmonic Oscillator}
A harmonic oscillator is bistate dynamical system $z(t)=(x(t),v(t))$ with the following ODE:
\begin{equation}
\dot x(t) = v(t), 
\qquad 
\dot v(t) = -\omega^2 x(t),
\end{equation}
where $W_{t}$ is the standard Wiener process, and $\omega$ is the natural undamped frequency of the system, and $\sigma$ is the constant exogenous noise rate. 
It's often used to model oscillations near stable equilibria in systems, such as LC circuits~\citep{47_harmonic_citation} and linearized mechanical vibrations.

\subsection{Damped Harmonic Oscillator}
\label{apdx:damped harmonic oscillator}

A damped harmonic oscillator is a bistate dynamical system with a damping coefficient ($\gamma$) compared to a simple harmonic oscillator above, given by the following ODE:
\begin{equation}
\dot x(t) = v(t), 
\qquad 
\dot v(t) = -\omega^2 x(t) - 2\gamma v(t),
\end{equation}
often used to model damped vibrations in mechanical systems with friction or RLC electronic circuits~\citep{47_harmonic_citation} with resistance, among others. The stochastic equivalent is also known as a Langevin System, written as: 
\begin{equation}
dX_t = V_t\,dt, 
\qquad 
dV_t = \bigl(-\omega^2 X_t - 2\gamma V_t\bigr)\,dt + \sigma\,dW_t,
\end{equation}
$W_{t}$ is the standard Wiener process, and $\omega$ is the natural undamped frequency of the system, $\gamma$ is the damping coefficient, and $\sigma$ is the constant exogenous noise rate. 
It is used to model the thermal fluctuations in mechanical systems, or the motion of a Brownian particle in a quadratic potential.

\begin{table}[t]
\caption{Full mathematical definition and descriptive summary of the considered dynamical systems.}
% \vspace{-0.05in}
\centering
\small
\setlength{\tabcolsep}{6pt}
\renewcommand{\arraystretch}{1.2}
\resizebox{\linewidth}{!}{
\begin{tabular}{llll}
\toprule
\textbf{Dynamical System} 
& \textbf{ODE} 
& \textbf{SDE (additive noise)} 
& \textbf{Model Use Case} \\
\midrule

\multirow{3}{*}{Exponential Decay}
& \multirow{3}{*}{$\dot x(t) = -\lambda x(t)$} 
& \multirow{3}{*}{$dX_t = -\lambda X_t\,dt + \sigma\,dW_t$} 
& First-order relaxation processes \\
& & & radioactive decay, simple stable linear systems \\
& & & mean-reverting stochastic dynamics (OU) \\

\midrule

\multirow{2}{*}{Harmonic} & $\dot x(t) = v(t)$ & $dX_t = V_t\,dt$ & Oscillations near stable equilibria \\
& $\dot v(t) = -\omega^2 x(t)$ & $dV_t = -\omega^2 X_t\,dt + \sigma\,dW_t$ & LC circuits, linearized mechanical vibrations \\

\midrule

\multirow{2}{*}{Damped Harmonic}
& $\dot x(t) = v(t)$ & $dX_t = V_t\,dt$ & Damped mechanical vibrations \\
& $\dot v(t) = -\omega^2 x(t) - 2\gamma v(t)$ & $ dV_t = (-\omega^2 X_t - 2\gamma V_t)\,dt + \sigma\,dW_t$ & RLC circuits, underdamped Langevin dynamics \\

\midrule

\multirow{2}{*}{Lotka--Volterra}
& $\dot x(t) = \alpha x(t) - \beta x(t)y(t)$ & $dX_t = (\alpha X_t - \beta X_tY_t)\,dt + \sigma\,dW_t^{(1)}$ & Predator--prey population dynamics, ecological systems
\\
& $\dot y(t) = \delta x(t)y(t) - \gamma y(t)$
& 
$dY_t = (\delta X_tY_t - \gamma Y_t)\,dt + \sigma\,dW_t^{(2)}$
& chemical reaction kinetics, resource--consumer models \\

\midrule

\multirow{3}{*}{Lorenz}
& $\dot x(t) = \sigma\bigl(y(t)-x(t)\bigr)$ & $dX_t = \sigma\bigl(Y_t - X_t\bigr)\,dt + \eta\,dW_t^{(1)}$ &  Canonical nonlinear and chaotic system
\\
& $\dot y(t) = x(t)\bigl(\rho-z(t)\bigr)-y(t)$ & $dY_t = \bigl(X_t(\rho - Z_t) - Y_t\bigr)\,dt + \eta\,dW_t^{(2)}$ & sensitivity to initial conditions
\\
& $\dot z(t) = x(t)y(t) - \beta z(t)$ &
$dZ_t = \bigl(X_tY_t - \beta Z_t\bigr)\,dt + \eta\,dW_t^{(3)}$
& benchmark for learning chaotic dynamics \\

\bottomrule
\end{tabular}
}
\label{tab:synthetic_dynamics_summary}
\end{table}

\subsection{Lotka--Volterra System}
A Lotka--Volterra System is also known as a predator-prey dynamical system with prey populations $x(t)$ and predator populations $y(t)$, which can be written as the following ODE:
\begin{equation}
\dot x(t) = \alpha x(t) - \beta x(t)y(t),
\qquad
\dot y(t) = \delta x(t)y(t) - \gamma y(t).
\end{equation}
It is often used to model simple predator-prey population dynamics~\citep{48_lotkavolterra_citation} in ecological environments, chemical reaction kinetics, and resource-consumer models. The stochastic version is used to model the influence of the outside environment in the form of perturbations, given by: 
\begin{equation}
dX_t = \bigl(\alpha X_t - \beta X_tY_t\bigr)\,dt + \sigma\,dW_t^{(1)},
\qquad
dY_t = \bigl(\delta X_tY_t - \gamma Y_t\bigr)\,dt + \sigma\,dW_t^{(2)},
\end{equation}
where $W_{t}$ is the standard Wiener process, $\alpha$ is the prey growth rate in the absence of predators, $\beta$ is the rate at which predators comsume prey, $\gamma$ is the death rate of predators in the absence of prey to feed over, and $\delta$ is the rate at which the consumed prey converted into predator population, and $\sigma$ is the constant exogenous noise rate. 

\subsection{Lorenz System}
The Lorenz system was originally used to model simplified thermal convection dynamics in the atmosphere~\citep{49_lorenz_citation} and was shown to exhibit interesting properties, including sensitivity to initial conditions, deterministic chaotic behavior, and strange attractor patterns. It's now considered a canonical example~\citep{50_lorenz_citation} in non-linear and chaotic dynamical system studies. The state evolution can be cast into an ODE as follows: 
\begin{equation}
\dot x(t) = \sigma\bigl(y(t)-x(t)\bigr),\qquad
\dot y(t) = x(t)\bigl(\rho-z(t)\bigr)-y(t),\qquad
\dot z(t) = x(t)y(t) - \beta z(t).
\end{equation}
A standard additive-noise variant is used to incorporate modeling error and exogenous influences. The stochastic Lorenz system can be written as:
\begin{equation}
\begin{aligned}
dX_t &= \sigma\bigl(Y_t - X_t\bigr)\,dt + \eta\,dW_t^{(1)}, \\
dY_t &= \bigl(X_t(\rho - Z_t) - Y_t\bigr)\,dt + \eta\,dW_t^{(2)}, \\
dZ_t &= \bigl(X_tY_t - \beta Z_t\bigr)\,dt + \eta\,dW_t^{(3)}.
\end{aligned}
\end{equation}
where $W_{t}^{(i)}$ are the standard Wiener processes, and ($\sigma, \rho, \beta$) are system-dependent model parameters controlling coupling strengths and dissipation rates, and $\eta$ is the constant exogenous noise rate. 

\subsection{Hopper Physics (MuJoCo)}
Hopper Physics is a dataset generated using the dm-control MuJoCo physics simulation library, which models the locomotion of a rigid-legged system governed by gravity and rigid-body mechanics. Let $q(t)$ be the generalized coordinates of the joints and $\dot q(t)$ be their generalized velocities. A standard form can be written as:
\begin{equation}
M(q)\ddot q + C(q,\dot q) + g(q) = \tau + J(q)^\top \lambda,
\end{equation}
where $M$ is the mass matrix, $C$ collects Coriolis/centrifugal terms, $g$ gravity, $\tau$ actuator torques, and $J^\top\lambda$ enforces constraints/contacts.
Equivalently in first-order state $s(t)=(q(t),\dot q(t))$:
\begin{equation}
\dot s(t) = 
\begin{bmatrix}
\dot q(t) \\
M(q)^{-1}\bigl(\tau + J(q)^\top\lambda - C(q,\dot q) - g(q)\bigr)
\end{bmatrix}.
\end{equation}

\section{Proofs of Main Theoretical Results}
\label{section: Mathematical Proofs}

\subsection{Proof of Theorem \ref{thm: theorem_1}}
\label{section: proof of theorem 4.1}

% \begin{theorem_1}\label{thm: theorem_1}
\begin{theorem_1_restated}
Let $f$ be a function in $\mathcal{C}^{m}[a, b]$ with bounded $m$-th derivative $\|f^{(m)}\|_{\infty} \ll \infty$, and let $t_{i}=a+(i-1)h$ be $n+1$ equidistant points to be interpolated, with the distant between consecutive points being $h=\frac{b-a}{n}$. Then, assuming a well defined B-Spline interpolant with knots $\{t_{i}\}_{i=0}^{n}$, $ \mu(t)= \sum_{i=0}^{n} c_{i,m} \mathcal{B}_{i,m} (t)$, where $\mathcal{B}_{i,m}$ is the B-Spline polynomial of degree $m$ as defined in Equation \ref{eq: Cox-De Boor Recursion}, has the asymptotic function approximation error $\mathcal{O}(n^{-m})$ compared to the linear interpolant $\mathcal{O}(n^{-2})$, defined as $p_{1}(t)= f(t_{i})+\frac{f(t_{i+1})-f(t_{i})}{t_{i+1}-t_{i}}(t-t_{i})$ for $t \in [t_{i}, t_{i+1}]$. 

Similarly, the approximation error between the first derivatives of the interpolants and the underlying function is $\mathcal{O}(n^{-m+1})$ for the B-Spline interpolant and $\mathcal{O}(n^{-1})$ for the linear interpolant. Formally,
\begin{itemize}
    \item$\|f(t)-\mu(t)\|_\infty \le (1+\|I\|)\,C^{1}_{m,a,b}\,n^{-m}\,\|f^{(m)}\|_\infty$,
    \item $\|f(t) - p_{1}(t)\|_{\infty} \le  \frac{(b-a)^{2}}{8}\|f''\|_{\infty} \cdot n^{-2}$,
    \item $\|f'(t)-\mu'(t)\|_\infty\le C^{2}_{m,a,b}\,n^{-m+1}\,\|f^{(m)}\|_\infty$,
    \item $\|f'(t)-p_{1}'(t)\|_\infty\le (b-a)\,n^{-1}\,\|f''\|_{\infty}\,$,
\end{itemize}
where $\|I\|= \sup\Big\{\frac{\|If\|}{\|f\|}:  v\in \mathcal{C}[a,b]\setminus\{0\} \Big\}$ and $C_{m,a,b}$ is a constant depending only on degree $m$ and extremums $a,b$ and $C_{m}$ is a constant depending only on degree $m$.
\end{theorem_1_restated}
% \end{theorem_1}

\begin{proof}
For the first two function approximation bounds, the proof essentially boils down to combining the results from the first part of Corollary \ref{cor: corollary_1} and Lemma \ref{thm: lemma_2} containing the function approximation error rates for the linear interpolant $p_{1}(t)$ and B-Spline iterpolant $\mu(t)$. 

From Corollary \ref{cor: corollary_1} we can write,
\begin{align}
\label{eq: linear_interpolant_approx_error}
\|f(t)-p_{1}(t)\|_{\infty} \leq \frac{1}{8}\Delta_{\max}^{2}|f''(t)|_{\infty}.
\end{align}
Here, $\Delta_{i}= t_{i+1}-t_{i}$ and $\Delta_{\max}= \max_{i}\Delta_{i}$, and for equidistant points, $\Delta_{\max}= \frac{b-a}{n}$ which after substituting in the above Equation \ref{eq: linear_interpolant_approx_error} gives us the required result for linear interpolation: 
\begin{align}
\label{eq: linear_interpolant_approx_error_with_n}
    \|f(t) - p_{1}(t)\|_{\infty} \le  \frac{(b-a)^{2}}{8}\|f''\|_{\infty} \cdot (n)^{-2},
\end{align}
Similarly, knowing that  $\Delta_{\max}= \frac{b-a}{n}$ for equidistant points and utilizing result from Lemma \ref{thm: lemma_1} gives us, 
\begin{align}
\label{eq: bspline_interpolant_approx_error_with_n}
\|f(t)-\mu(t)\|_\infty
\le
(1+\|I\|)\,C^{1}_{m,a,b}\,n^{-m}\,\|f^{(m)}\|_\infty.
\end{align}
Using similar arguments for results from Lemma~\ref{thm: lemma_5} and second part of Corollary~\ref{cor: corollary_1}, and substituting $\Delta_{\max}= \frac{b-a}{n}$ in the upper bounds we have that, 
\begin{align}
\label{eq: bspline_derivative_approx_error_with_n}
\|f'(t)-\mu'(t)\|_\infty\le C^{2}_{m,a,b}\,n^{-m+1}\,\|f^{(m)}\|_\infty,
\end{align}
and 
\begin{align}
\label{eq: linear_derivative_approx_error_with_n}
\|f'(t)-p_{1}'(t)\|_\infty\le (b-a)\,n^{-1}\,\|f''\|_{\infty}.
\end{align}
From Equations \ref{eq: linear_interpolant_approx_error_with_n}-\ref{eq: linear_derivative_approx_error_with_n}, we get the required result.
\end{proof}

\subsection{Proof of Theorem \ref{thm: theorem_2}}
\label{section: proof of theorem 4.2}

% \begin{theorem_2}
\begin{theorem_2_restated}
Let $\mu(t)$ be a well defined B-Spline interpolant, interpolanting function $f(t) \in \mathcal{C}^{2}[a,b]$ with observation sites at $\{t_{i}\}_{i=0}^{n}$,  $ \mu(t)= \sum_{i=0}^{n} c_{i,m} \mathcal{B}_{i,m} (t)$, where $\mathcal{B}_{i,m}$ is the B-Spline polynomial of degree $m$ as defined in Equation \ref{eq: Cox-De Boor Recursion}. Then, degree $1$ interpolant or $\mu_{t}$ with $m=1$ is equivalent to the linear interpolant defined as: 
$$
p_{1}(t)= f(t_{i})+\frac{f(t_{i+1})-f(t_{i})}{t_{i+1}-t_{i}}(t-t_{i}) \: \text{ for any } \: t \in [t_{i}, t_{i+1}].
$$
\end{theorem_2_restated}
% \end{theorem_2}

\begin{proof}
We clamp the observation sites to get $n+3$ knots: $\tau_{0}=\tau_{1}<\tau_{2} \cdots \tau_{n}<\tau_{n+1}=\tau_{n+2}$ with $\tau_{0}=\tau_{1}=t_{0}$,  
$\tau_{n+1}=\tau_{n+2}=t_{n}$, and $\tau_{i}=t_{i-1}$ for $i \in [1, n+1]$.

From the definition of B-Splines in Section \ref{sec:linear vs. B-spline}, the 0-th degree B-Spline basis is defined as:
\begin{align}
\label{eq: 0th-degree b-spline apdx}
\mathcal{B}_{i,0}=
\begin{cases}
    1 & \text{if} \quad t\in[\tau_{i}, \tau_{i+1}), \\
    0 & \text{otherwise}.
\end{cases}
\end{align}
With the higher degree B-Spline basis defined recursively as:
\begin{align}
\label{eq: Cox-De Boor Recursion apdx}
\mathcal{B}_{i,m}(t)
=
\frac{t - \tau_i}{\tau_{i+m} - \tau_i}\, \mathcal{B}_{i,m-1}(t),
+
\frac{\tau_{i+m+1} - t}{\tau_{i+m+1} - \tau_{i+1}}\, \mathcal{B}_{i+1,m-1}(t).
\end{align}
The interpolant function $\mu(t)$ of degree $m$ is then a linear combination of B-Splines around $n+3$ knots, written as:
\begin{align}
\label{eq: b-spline interpolant function apdx}
\mu(t) &= \sum_{i=0}^{n+2} c_{i,m} \mathcal{B}_{i,m} (t) \quad \text{s.t.} \quad \mu(t_{0})=f(t_{0}), \: \mu(t_{1}) = f(t_{1}), \: \cdots, \: \mu(t_{n})=f(t_{n}).
\end{align}
Since our degree $m=1$, we thus focus on getting expressions for $\mathcal{B}_{i,1}$, starting with the left and right endpoints. 
From Cox-De Boor recursion and Equation \ref{eq: Cox-De Boor Recursion apdx}, we know that
\begin{align*}
\mathcal{B}_{0,1}(t)= \frac{t-\tau_{0}}{\tau_{1}-\tau_{0}} \mathcal{B}_{0,0}(t) + \frac{\tau_{2}-t}{\tau_{2}-\tau_{1}} \mathcal{B}_{1,0}(t).
\end{align*}
Since $\tau_{0}=\tau_{1}=t_{0}$ we have that $\mathcal{B}_{0,0}\equiv 0$, so
we get that $\mathcal{B}_{0,1}(t)= \frac{t_{1}-t}{t_{1}-t_{0}}\mathcal{B}_{1,0}$, which gives us,
\begin{align}
\label{eq: linear_eq_proof_s_1}
\mathcal{B}_{0,1}(t)=
\begin{cases}
    \frac{t_{1}-t}{t_{1}-t_{0}}, \quad \: t\in [t_{0}, t_{1}),\\
    0, \qquad \qquad \text{otherwise}.
\end{cases}
\end{align}
Similarly for the right endpoint, we have that $\tau_{n+1}=\tau_{n+2}=t_{n+1}$ which implies $\mathcal{B}_{n+2,0} \equiv 0$. And using this in the Cox-De Boor recursion formula (Equation \ref{eq: Cox-De Boor Recursion apdx}) gives 
\begin{align}
\label{eq: linear_eq_proof_s_n+2}
\mathcal{B}_{n+2,1}(t)=
\begin{cases}
    \frac{t-t_{n}}{t_{n+1}-t_{n}}, \quad t\in [t_{n-1}, t_{n}], \\
    0, \quad\quad \quad \quad \quad \text{otherwise}.
\end{cases}
\end{align}
For $i\in[1,n+1]$ we have that 
\begin{align}
\mathcal{B}_{i,1}(t)= \frac{t - \tau_i}{\tau_{i+1} - \tau_i}\, \mathcal{B}_{i,0}(t)
+
\frac{\tau_{i+2} - t}{\tau_{i+2} - \tau_{i+1}}\, \mathcal{B}_{i+1,0}(t),
\end{align}
and from Equation \ref{eq: 0th-degree b-spline apdx} we know that $\mathcal{B}_{i,0}(t)=1$ in $t \in [\tau_{i}, \tau_{i+1})$, $\mathcal{B}_{i+1,0}(t)=1$ in $t \in [\tau_{i+1}, \tau_{i+2})$, and utilizing the fact that $\tau_{i}=t_{i-1}$ for $i\in[1,n+1]$, gives us
\begin{align}
\label{eq: linear_eq_proof_s_i}
\mathcal{B}_{i,1}(t)=
\begin{cases}
\frac{t-t_{i-1}}{t_{i}-t_{i-1}}, \quad t\in [t_{i-1}, t_{i}) \\
\frac{t_{i+1}-t}{t_{i+1}-t_{i}}, \quad t\in [t_{i}, t_{i+1}) \\
0 \quad \quad \quad \quad
\end{cases}
\end{align}
We can now use the constraints from Equation \ref{eq: b-spline interpolant function apdx} and Equations \ref{eq: linear_eq_proof_s_1}, \ref{eq: linear_eq_proof_s_i}, \ref{eq: linear_eq_proof_s_n+2} to solve and get, 
\begin{align}
\label{eq: const. values}
c_{0,1}=c_{1,1}= f(t_{0}), c_{2,1}= f(t_{1}), \cdots, c_{n,1}=f(t_{n-1}), c_{n+1,1}=c_{n+2,1}=f(t_{n}).
\end{align}
Finally, to prove the equivalence with the linear interpolant, consider the B-Spline interpolant $\mu(t)$ for $t\in [t_{i}, t_{i+1}]$. We know that
\begin{align*}
\mu(t) = \sum_{i=0}^{n+2}c_{i,m}\mathcal{B}_{i.m}(t).
\end{align*}
Utilizing Equations \ref{eq: linear_eq_proof_s_i} and \ref{eq: const. values}, we get that
\begin{align*}
\mu(t)&= f(t_{i})\frac{t_{i+1}-t}{t_{i+1}-t_{i}}+f(t_{i+1})\frac{t-t_{i}}{t_{i+1}-t_{i}} \\
 &= \frac{f(t_{i})(t_{i+1}-t_{i})+f(t_{i})t_{i}-f(t_{i+1})t_{i}+t(f(t_{i+1})-f(t_{i}))}{t_{i+1}-t_{i}}\\
&=f(t_{i})+\frac{f(t_{i+1})-f(t_{i})}{t_{i+1}-t_{i}}\times(t-t_{i}),
\end{align*}
which gives
\begin{align*}
\mu(t)=p_{1}(t)=f(t_{i})+\frac{f(t_{i+1})-f(t_{i})}{t_{i+1}-t_{i}}\times(t-t_{i}).
\end{align*}
Thus, we have proven the equivalence between the B-Spline interpolant with $m=1$ and $p_{1}(t)$.
\end{proof}

\subsection{Proof of Theorem \ref{thm: theorem_5}}
\label{section: proof of theorem 4.5}

\begin{theorem_5_restated}
Let $f \in \mathcal{C}^{m}[a,b]$ denote the true underlying dynamics with bounded $m$-th derivative $\|f^{(m)}\|_{\infty} < \infty$, and let $\mu(t) = \sum_{i=0}^{n} c_{i,m} \mathcal{B}_{i,m}(t)$ be the degree-$m$ B-Spline interpolant with equidistant knots $\{t_i\}_{i=0}^{n}$ and mesh spacing $h = \frac{b-a}{n}$. Let $u_{\theta}(t,x)$ be the velocity field learned by minimizing the flow matching regression objective (Equation \ref{eq: SFB loss}), and let $p_1(t)$ be the linear interpolant as defined in Theorem \ref{thm: theorem_2}. Then, choosing the conditional probability path as $p_t(x|z) = \mathcal{N}(x; \mu_t(z), \sigma^2)$ with constant $\sigma$, the mean Euclidean error between the learned velocity field and the true dynamics decomposes as:
\begin{align*}
\mathbb{E}_{t,x}[\|u_{\theta}(t,x) - f'(t)\|]
\le
\underbrace{\sqrt{\mathbb{E}_{t,x}[\|u_{\theta}(t,x) - \mu'(t)\|^{2}]}}_{\text{regression error}}
+
\underbrace{C_{m,a,b}\, n^{-m+1}\, \|f^{(m)}\|_{\infty}}_{\text{approximation error}},
\end{align*}
where $C_{m,a,b}$ is a constant depending only on degree $m$ and the interval endpoints $a,b$. In particular, the approximation error is $\mathcal{O}(n^{-m+1})$ for the degree-$m$ B-Spline interpolant compared to $\mathcal{O}(n^{-1})$ for the linear interpolant.
\end{theorem_5_restated}

\begin{proof}
We proceed in two steps: we first establish the derivative-approximation error, then connect it to the flow-matching objective via the conditional velocity field, thereby yielding the decomposition.

\textbf{Step 1: Derivative approximation error.}
From Theorem~\ref{thm: theorem_1} and by setting substituting $\Delta_{\max} = h = \frac{b-a}{n}$ for the equidistant knot sequence, we know that,
\begin{align}
\label{eq: bspline_deriv_error_final}
\|f'(t) - \mu'(t)\|_{\infty} \le C^{2}_{m,a,b}\, n^{-m+1}\, \|f^{(m)}\|_{\infty}.
\end{align}
Similarly, from Theorem~\ref{thm: theorem_1}, we also have the linear interpolant derivative error
\begin{align}
\label{eq: linear_deriv_error}
\|f'(t) - p_1'(t)\|_{\infty} \le (b-a)\, n^{-1}\, \|f''\|_{\infty}.
\end{align}

\textbf{Step 2: Error Decomposition.}
From Equation \ref{eq: gaussian probability path} or Theorem 3 from \citet{8_lipman_FM}, we know that for the conditional probability path $p_t(x|z) = \mathcal{N}(x; \mu_t(z), \sigma^2)$ with constant $\sigma$, the conditional velocity field reduces to
\begin{align}
\label{eq: cond_velocity_dynamics}
u_t(x \mid z) = \frac{\sigma'_t(z)}{\sigma_t(z)}(x - \mu_t(z)) + \mu'_t(z) = \mu'_t(z),
\end{align}
where the last equality follows from $\sigma'_t(z) = 0$. Thus, the flow matching regression objective trains $u_{\theta}$ to approximate $\mu'(t)$, which is itself an approximation to the true dynamics $f'(t)$.

Applying the triangle inequality, we can write
\begin{align*}
\|u_{\theta}(t,x) - f'(t)\|
&\le
\|u_{\theta}(t,x) - u_t(x \mid z)\| + \|u_t(x \mid z) - f'(t)\|_{\infty}, \\
\|u_{\theta}(t,x) - f'(t)\|
&\le
\|u_{\theta}(t,x) - \mu'(t)\| + \|\mu'(t) - f'(t)\|_{\infty},
\end{align*}
where the second equation follows from Equation~\ref{eq: cond_velocity_dynamics}.
Substituting the derivative approximation bound from Equation~\ref{eq: bspline_deriv_error_final} into the second term yields
\begin{align*}
\|u_{\theta}(t,x) - f'(t)\|
&\le
\|u_{\theta}(t,x) - \mu'(t)\| + C^{2}_{m,a,b}\, n^{-m+1}\, \|f^{(m)}\|_{\infty}.
\end{align*}

Finally, taking expectation on both sides with respect to $x, t$ and utilizing Jensen's inequality ($\mathbb{E}[\|X\|]\leq \mathbb{E}([\|X\|^{2}])^{1/2}$) gives us the required result:
\begin{align*}
\mathbb{E}_{t,x}[\|u_{\theta}(t,x) - f'(t)\|]
&\le
\sqrt{\mathbb{E}_{t,x}[\|u_{\theta}(t,x) - \mu'(t)\|^{2}]} + C^{2}_{m,a,b}\, n^{-m+1}\, \|f^{(m)}\|_{\infty}.
\end{align*}
The first term is the regression error, controlled by the flow matching training objective, and the second term is the approximation error, determined by the spline degree and the number of observation sites. Since the analogous decomposition for the linear interpolant gives an approximation error of $\mathcal{O}(n^{-1})$ from Equation~\ref{eq: linear_deriv_error}, increasing the spline degree $m$ strictly reduces the approximation component of the bound, thereby pushing the learned velocity field $u_{\theta}$ closer to the true dynamics $f'(t)$.
\end{proof}

\subsection{Proof of Corollary \ref{thm: corollary_1}}
\label{section: proof of corollary 4.6}

\begin{theorem_6_restated}
Assume similar assumptions as in Theorem~\ref{thm: theorem_5}. Let the ground truth and generated samples be defined as: 
\begin{align*}
x_t = x_0 + \int_0^t f'(s)\,ds, \qquad
\hat{x}_t = x_0 + \int_0^t u_{\theta}(s, \hat{x}_s)\,ds.
\end{align*}
Then, for any $t \in [a,b]$ (including times
$t \notin \{t_i\}_{i=0}^{n}$ not in the training data), the sampling error satisfies
\begin{align*}
\mathbb{E}[\|x_t - \hat{x}_t\|]
\le
(b-a)\bigg(
\mathbb{E}[\|u_{\theta} - \mu'\|]
+ C_{m,a,b}\,n^{-m+1}\,\|f^{(m)}\|_{\infty}
\bigg).
\end{align*}
In particular, the approximation component is $\mathcal{O}(n^{-m+1})$
for the degree-$m$ B-Spline interpolant compared to $\mathcal{O}(n^{-1})$
for the linear interpolant.
\end{theorem_6_restated}

\begin{proof}
Since both trajectories share the same initial condition $x_0$, we can
write
\begin{align}
\label{eq: traj_diff_integral}
x_t - \hat{x}_t
= \int_0^t \big[f'(s) - u_{\theta}(s, \hat{x}_s)\big]\,ds.
\end{align}
Taking absolute values and bounding $|t| \le b - a$ gives
\begin{align}
\label{eq: traj_diff_bound}
\|x_t - \hat{x}_t\|
\le (b-a)\,\|f' - u_{\theta}\|.
\end{align}
The result now follows directly from Theorem~\ref{thm: theorem_5},
which gives
\begin{align*}
\mathbb{E}[\|f' - u_{\theta}\|]
\le
\mathbb{E}[\|u_{\theta} - \mu'\|]
+ C_{m,a,b}\,n^{-m+1}\,\|f^{(m)}\|_{\infty}.
\end{align*}
Substituting into Equation~\ref{eq: traj_diff_bound} yields the
stated bound. The comparison with the linear interpolant follows from
Corollary~\ref{cor: corollary_1}, which gives an approximation error
of $\mathcal{O}(n^{-1})$.
\end{proof}

\subsection{Proof of Theorem \ref{thm: theorem_3}}
\label{section: proof of theorem 4.3}

% \begin{theorem_3}
\begin{theorem_3_restated}
Let $\mu(t)$ be the degree $m$ B-Spline interpolant for trajectories $X=[x_{t_{0}}, x_{t_{1}}, \cdots, x_{t_{n}}]$ as defined in Equation \ref{eq: Cox-De Boor Recursion}, and let the latent be $z=(x_{t_{0}}, x_{t_{1}}, \cdots, x_{t_{n}})$ with the distribution $q(z)$ over observed trajectories. Then, choosing the conditional probability path as $p_{t}(x|z) =\mathcal{N}(x; \mu_t(z), \sigma^{2}_{t}(z))$ with $\sigma_{t}(z) = \sigma \rightarrow 0$, enables the following:
\begin{itemize}
    \item The marginal probability distribution written as, $p_{t}(x)= \int p_{t}(x|z)q(z)dz$ satifies the observation distribution with $p_{t=t_{i}}(x)= q(x_{t_{i}}=x)$.
    \item  The marginal velocity field written as $u_{t}(x)= \int u_t(x \mid z) \frac{p_t(x \mid z) \, q(z)}{p_t(x)} \, dz$ and the marginal probability $p_{t}(.)$ obey the continuity equation $\frac{\partial p_t(x)}{\partial t} = -\nabla \cdot(u_{t}(x)p_{t}(x))$.
    \item The conditional probability path defined above is generated by the conditional velocity field given by $u_{t}(x\mid z)= \sum_{i=1}^{n} m \times c_{i,m} \Big\{ \frac{\mathcal{B}_{i,m-1}(t)}{t_{i+m}-t_{i}}-\frac{\mathcal{B}_{i+1,m-1}(t)}{t_{i+m+1}-t_{i+1}}\Big\}$.
\end{itemize}
\end{theorem_3_restated}
% \end{theorem_3}

\begin{proof}
(A). We know that $\lim_{\sigma \rightarrow 0}\mathcal{N}(x; \mu(t\mid z), \sigma^{2}I)=\delta(x-\mu(t\mid z))$, where $\delta$ is the delta distribution. For $t= t_{i}$ we can thus write
\begin{align*}
p_{t=t_{i}}(x)= \int \delta(x-\mu(t_{i}\mid z) q(z) dz.
\end{align*}
Since $\mu(.)$ is a B-Spline interpolant, we have that $\mu(t_{i}\mid z)=x_{t_{i}}$. Thus, we have
\begin{align*}
p_{t=t_{i}}(x)&= \int \delta(x-x_{t_{i}}) q(z) dz\\
&=\int \cdots\int \delta(x-x_{t_{i}})q(x_{t_{0}}, x_{t_{1}}, \cdots, x_{t_{n}}) dx_{t_{0}} \cdots dx_{t_{n}}\\
&= \int \delta(x-x_{t_{i}}) \Big(\int \cdots\int q(x_{t_{0}}, x_{t_{1}}, \cdots, x_{t_{n}}) dx_{t_{0}} \cdots dx_{t_{n}} \Big) dx_{t_{i}}\\
&= \int \delta(x-x_{t_{i}})q(x_{t_{i}})dx_{t_{i}} = q(x=x_{t_{i}}),
\end{align*}
which gives us the required result. 

(B). This result follows from a simple extension of Theorem 1 from \citet{8_lipman_FM}. We can write the marginal $p_{t}$ as $p_t(x)= \int p_t(x | z) \, q(z) \, dz$, which gives us
\begin{align*}
 \frac{\partial p_t(x)}{\partial t}
= \frac{\partial}{\partial t} \int p_t(x | z) \, q(z) \, dz.
\end{align*}
Invoking the Continuity Equation (from Equation \ref{eq: continuity equation}) for the conditional probability, we get
\begin{align*}
\frac{\partial p_t(x)}{\partial t} &= \int -\nabla \cdot \left( u_t(x \mid z) \, p_t(x \mid z) \right)
\, q(z) \, dz\\
&= \int -\nabla \cdot \left( u_t(x \mid z) \, p_t(x \mid z)
\, q(z) \right) \, dz\\
&= -\nabla \cdot \int u_t(x \mid z) \,
\frac{p_t(x \mid z) \, q(z)}{p_t(x)} \, p_t(x) \, dz\\
&= - \nabla \cdot \left( \int u_t(x \mid z) \,
\frac{p_t(x \mid z) \, p(z)}{p_t(x)} \, dz\right)p_t(x)\\
&= - \nabla \cdot \left( \int u_t(x \mid z) \,
\frac{p_t(x \mid z) \, q(z)}{p_t(x)} \, dz \right)p_{t}(x).
\end{align*}
Utilizing the fact that the marginal velocity can be written as $u_{t}(x)= \int u_t(x \mid z) \frac{p_t(x \mid z) \, q(z)}{p_t(x)} \, dz$, we get the required result: 
\begin{align*}
\frac{\partial p_t(x)}{\partial t} = -\nabla \cdot(u_{t}(x)p_{t}(x)).
\end{align*}

(C). The random variable $x\sim\mathcal{N}(x; \mu_t(z), \sigma^{2}I)$ can be written as
\begin{align*}
    x_{t}= \mu_t(z)+\sigma \epsilon, \quad \quad \epsilon\sim\mathcal{N}(0,1).
\end{align*}
Then, from Equation \ref{eq: gaussian probability path} or Theorem 3 from \citet{8_lipman_FM} we get that 
\begin{align*}
u_t(x | z) &= \frac{\sigma_t'(z)}{\sigma_t(z)} (x - \mu_t(z)) + \mu_t'(z) = 0+\frac{d\mu_t(z)}{dt},
\end{align*}
where the last equality follows from the fact that $\sigma$ is constant. According to Lemma \ref{thm: lemma_4}, we obtain
\begin{align*}
u_t(x | z)= \sum_{i=1}^{n} m \times c_{i,m} \Big\{ \frac{\mathcal{B}_{i,m-1}(t)}{t_{i+m}-t_{i}}-\frac{\mathcal{B}_{i+1,m-1}(t)}{t_{i+m+1}-t_{i+1}}\Big\},
\end{align*}
which completes the proof.
\end{proof}

\subsection{Proof of Proposition \ref{thm: theorem_4}}
\label{section: proof of proposition 4.4}

% \begin{theorem_4}\label{thm: theorem_4}
\begin{theorem_4_restated}
Let the trajectories $X=[x_{t_{0}}, x_{t_{1}}, \cdots, x_{t_{n}}]$ be observations from the SDE dynamics as defined in Equation \ref{eq: SDE}, with constant diffusion term $g(t)=\sigma$,  and let the latent be $z=(x_{t_{0}}, x_{t_{1}}, \cdots, x_{t_{n}})$ with the distribution $q(z)$ over observed trajectories. Then, choosing the conditional probability path as a Spline-Bridge $p_{t}(x|z) =\mathcal{N}(x; \mu_t(z), \sigma^{2}_{t}(z))$ with $\lim_{\sigma(t_{i};z) \rightarrow0}$, where $\mu_{t}(z)$ is the B-spline interpolant and setting the time conditional weights $\lambda(t)=\sigma_{t}(z)$ we can write the regression loss $\mathcal{L}_{\mathrm{SplineFlow}}$ as: 
\begin{equation*}
\begin{aligned}
\label{eq: SFB loss}
\mathcal{L}_{\mathrm{SplineFlow}}(\theta, \phi)
= \mathbb{E}_{t, z, x}
\Big[
\| u_{\theta}(t,x) - (\epsilon\sigma_{t}'(z)+\mu_{t}'(z)) \|^2 +
\|\lambda(t) s_{\phi}(t,x) + \epsilon \|^2
\Big],
\end{aligned}
\end{equation*}
where $\epsilon \sim \mathcal{N}(0,1)$ and the expectation is defined over time $t \sim \mathcal{U}(0,1)$, the latent variables $z \sim q(z)$ and the conditional probability samples $x \sim p_t(x \mid z)$.
\end{theorem_4_restated}

\begin{proof}
We know from the preliminaries that for a given diffusion $g(t)$ the SDE dynamics can be recovered by approximating the drift ($u_{t}(.)$) by utilizing the probability flow velocity field $u^{o}_{t}(.)$ and the probability score ($\nabla_{t}p(.)$). The regression loss from SF2M objective~\citep{20_tong_SF2M} can thus be written as:
\begin{equation}\label{eq:sf2m_loss_apdx}
\begin{aligned}
\mathcal{L}_{[\mathrm{SF}]^{2}\mathrm{M}}(\theta, \phi)
=
\mathbb{E}_{t, z, x}
\Big[
\| u_{\theta}(t,x) - u^{o}_t(x | z) \|^2 
+ \lambda(t)^2
\| s_{\phi}(t,x) - \nabla \log p_t(x \mid z) \|^2
\Big].
\end{aligned}
\end{equation}
And as shown in \citet{8_lipman_FM} for conditional gaussian paths $p_t(x | z) = \mathcal{N}(x \: | \: \mu_t(z), \sigma_t(z)^2)$, the velocity field inducing the distribution can be written as
\begin{align}
\label{eq: gaussian conditional velocity apdx}
u^{o}_t(x | z) &= \frac{\sigma_t'(z)}{\sigma_t(z)} (x - \mu_t(z)) + \mu_t'(z).
\end{align}
The samples from the conditional probability path $p_{t}(x|z) =\mathcal{N}(x; \mu_t(z), \sigma^{2}_{t}(z))$ can be written as:
\[
x_{t}= \mu_t(z)+\epsilon \sigma_{t}(z).
\]
After substitution in Equation \ref{eq: gaussian conditional velocity apdx}, we have
\begin{align}
\label{eq: loss_term_1}
u^{o}_t(x | z) = \sigma_t'(z)\epsilon + \mu_t'(z).
\end{align}
And since $x_{t}$ is a gaussian random variable, we get
\begin{align*}
    \nabla_{x}p(x \mid z)= -\frac{(x-\mu_{t}(z))}{\sigma_{t}^{2}(z)},
\end{align*}
which after using the fact that $x_{t}= \mu_t(z)+\epsilon \sigma_{t}(z)$ gives us
\begin{align}
\label{eq: loss_term_2}
    \nabla_{x}p(x \mid z)= -\frac{\epsilon}{\sigma_{t}(z)}.
\end{align}
Substituting Equation \ref{eq: loss_term_2}, Equation \ref{eq: loss_term_1}, and the given $\lambda(t)=\sigma_{t}(z)$ into the regression loss in Equation \ref{eq:sf2m_loss_apdx} gives us the required result,
\begin{align*}
\mathcal{L}_{\mathrm{SplineFlow}}(\theta, \phi)
=
&\mathbb{E}_{t, z, x}
\Big[
\| u_{\theta}(t,x) - (\epsilon\sigma_{t}'(z)+\mu_{t}'(z)) \|^2 +
\|\lambda(t) s_{\phi}(t,x) + \epsilon \|^2
\Big].
\end{align*}
\end{proof}

\subsection{Technical Lemmas used for Proving Main Theorems}
\label{section: Technical Lemmas}

% \begin{lemma_1}
\begin{lemma}
\label{thm: lemma_1}
Let $f(t)$ be a continuous and twice differentiable function in $\mathcal{C}^{2}[t_{i}, t_{i+1}]$, with $t_{i}, t_{i+1} \in \mathbb{R}$, then $\exists \xi \in [t_{i}, t_{i+1}]$ s.t.
\begin{align*}
f(t)= f(t_{i})+\Big(\frac{f(t_{i+1})-f(t_{i})}{t_{i+1}-t_{i}}\Big)(t-t_{i})+\frac{f''(\xi)}{2}(t-t_{i})(t-t_{i+1}).
\end{align*}
\end{lemma}
% \end{lemma_1}

\begin{proof}
Define the linear interpolant as
\begin{align}
\label{eq: thm1_eq1}
p_1(t) := f(t_{i}) + \frac{f(t_{i+1})-f(t_{i})}{t_{i+1}-t_{i}}(t-t_{i}).
\end{align}
Consider some intermediary functions as
\begin{align}
\label{eq: thm1_eq2}
L(t) := f(t) - p_1(t),
\qquad
\phi(t) := (t-t_{i})(t-t_{i+1}).
\end{align}
For a fixed $x \in (t_{i},t_{i+1})$, define the auxiliary function as
\begin{align*}
g(y) := L(y) - \frac{L(t)}{\phi(t)} \, \phi(y),
\qquad y \in [t_{i},t_{i+1}].
\end{align*}
We observe that
\begin{align*}
g(t_{i}) = g(t_{i+1}) = g(t) = 0.
\end{align*}
Thus by Mean Value Theorem we know that, $\exists \xi_1 \in (t_{i},x)$ and $\exists \xi_2 \in (t,t_{i+1})$ such that
\begin{align*}
g'(\xi_1) = g'(\xi_2) = 0.
\end{align*}
Applying Mean Value Theorem theorem again, we get that $\exists \xi \in (t_{i},t_{i+1})$ such that 
\begin{align*}
g''(\xi) = 0.
\end{align*}
Since $p_{1}(t)$ is linear, $p_{1}''(t)=0$, and $\phi''(y) = 2$. We thus get 
\begin{align*}
g''(y) &= f''(y) - 2\frac{L(t)}{\phi(t)},\\ 
\implies &f''(\xi) - \frac{2 L(t)}{\phi(t)} = 0\\
\implies & L(t)= \frac{f''(\xi)\phi(t)}{2}\\
\implies & L(t) =\frac{f''(\xi)}{2} (t-t_{i})(t-t_{i+1}).
\end{align*}
Using this and Equations \ref{eq: thm1_eq1} and \ref{eq: thm1_eq2}, we get
\begin{align*}
f(t)
= p_1(t)
+ \frac{f''(\xi)}{2}(t-t_{i})(t-t_{i+1}),
\qquad \xi \in (t_{i},t_{i+1}),
\end{align*}
which completes the proof.
\end{proof}

% \begin{corollary_1}
\begin{corollary}
\label{cor: corollary_1}
Let $f(t)$ be a continuous and twice differentiable function in $\mathcal{C}^{2}[a,b]$, sampled at timepoints $\{t_{0}, t_{1}, \cdots, t_{n}\}$ such that $a<t_{0}<t_{1} \cdots <t_{n}<b$ with $p_1(t) = f(t_{i}) + \frac{f(t_{i+1})-f(t_{i})}{t_{i+1}-t_{i}}(t-t_{i})$ be the linear interpolant in the interval $[t_{i}, t_{i+1}]$. Then, the approximation error is given by the following inequality:
\begin{align*}
    \|f(t)-p_{1}(t)\|_{\infty} \leq \frac{1}{8}\Delta_{\max}^{2}|f''(t)|_{\infty}, 
\end{align*}
Similarly, the approximation error of the first derivatives is given by the following inequality:
\begin{align*}
    \|f'(t)-p_{1}'(t)\|_{\infty} \leq \Delta_{\max}|f''(t)|_{\infty}, 
\end{align*}
where $\Delta_{i}= t_{i+1}-t_{i}$ and $\Delta_{\max}= \max_{i}\Delta_{i}$.
\end{corollary}
% \end{corollary_1}

\begin{proof}
From Lemma \ref{thm: lemma_1}, we know that for $t \in[t_{i}, t_{i+1}]$, $\exists\xi \in[t_{i}, t_{i+1}]$ such that $f(t)$ could be written as
\begin{align*}
f(t)= f(t_{i})+\Big(\frac{f(t_{i+1})-f(t_{i})}{t_{i+1}-t_{i}}\Big)(t-t_{i})+\frac{f''(\xi)}{2}(t-t_{i})(t-t_{i+1}),
\end{align*}
which could be rewritten as
\begin{align}
\label{eq:1st_order_approx}
f(t)= p_{1}(t)+\frac{f''(\xi)}{2}(t-t_{i})(t-t_{i+1}).
\end{align}
The approximation error then becomes 
\begin{align*}
f(t)-p_{1}(t) = \frac{f''(\xi)}{2}(t-t_{i})(t-t_{i+1}), \quad
\|f(t)-p_{1}(t)\|_{\infty} \leq \frac{f''(\xi)}{8}(t_{i+1}-t_{i})^{2}.
\end{align*}
Taking a supremum over all the intervals gives us the approximation error bound
\begin{align*}
\|f(t)-p_{1}(t)\|_{\infty}& \leq \Delta_{\max}^{2}\frac{\|f''(t)\|_{\infty}}{8}.
\end{align*}

Similarly, differentiating Equation ~\ref{eq:1st_order_approx} gives 
\begin{align*}
f'(t)-p_{1}'(t)=&\frac{f''(\xi)}{2}((t-t_{i})+(t-t_{i+1})) \leq \frac{f''(\xi)}{2}(\|t-t_{i}\|+\|t-t_{i+1}\|) \leq f''(\xi)\Delta_{\max},
\end{align*}
which further gives us the required bound on the first derivatives
\begin{align*}
    \|f'(t)-p_{1}'(t)\|_{\infty} \leq \Delta_{\max}|f''(t)|_{\infty}. 
\end{align*}
\end{proof}

% \begin{lemma_3}
\begin{lemma}
\label{thm: lemma_3}
Let $t=(t_i)_{i=1}^{n+k}$ be a knot sequence with
$t_1=\cdots=t_k=a<\cdots<b=t_{n+1}=\cdots=t_{n+k}$ and mesh size
$\Delta_{\max}:=\max_i(t_{i+1}-t_i)$.
Let $S_{k,t}$ denote the spline space of degree $k$ with knots $t$.
Assume that the interpolation operator
$I:C([a,b])\to S_{k,t}$ is well-defined and bounded, with interpolation sites at $\tau_{0}<\tau_{1}<\cdots<\tau_{n}$ for $g\in C^k([a,b])$, such that $Ig(\tau_{i})= g(\tau_{i})$.
Then there exists a constant $C_k>0$, depending only on $k$, such that
\begin{align*}
\|g-Ig\|_\infty
\le
(1+\|I\|)\,C_k\,\Delta_{\max}^k\,\|g^{(k)}\|_\infty .
\end{align*}
\end{lemma}
% \end{lemma_3}

\begin{proof}

We know that the interpolant operator $I$ is linear, and for any $s\in S_{k,t}$ since the interpolant is well defined by assumption we get that $Is=s$. We can thus write
\begin{align*}
g-Ig=(g-s)-I(g-s).
\end{align*}
Hence, we get
\begin{align*}
\|g-Ig\|_\infty
\le
\|g-s\|_\infty+\|I(g-s)\|_\infty
\le
(1+\|I\|)\|g-s\|_\infty.
\end{align*}
Taking the infimum over $s\in S_{k,t}$ yields
\begin{align*}
\|g-Ig\|_\infty
\le
(1+\|I\|)\operatorname{dist}(g,S_{k,t}).
\end{align*}
From Lemma \ref{thm: lemma_2}, we know that, if $g\in C^k([a,b])$ then
$\operatorname{dist}(g,S_{k,t})\le C_k \Delta_{\max}^k \|g^{(k)}\|_\infty$.
Combining the two inequalities proves the result.
\end{proof}

\begin{lemma}
\label{thm: lemma_2}
Let $t=(t_i)_{i=1}^{n+k}$ satisfy $t_1=\cdots=t_k=a<\cdots<b=t_{n+1}=\cdots=t_{n+k}$
and let $\Delta_{\max}:=\max_i(t_{i+1}-t_i)$. Let $S_{k,t}$ be the spline space of degree $k$ with knots $t$, and define
\begin{align*}
\operatorname{dist}(g,S_{k,t}) := \inf_{s\in S_{k,t}}\|g-s\|_\infty .
\end{align*}
If $g\in C^k([a,b])$, then there exists a constant $C_k$ (depending only on $k$)
such that
\begin{align*}
\operatorname{dist}(g,S_{k,t}) \le C_k\,\Delta_{\max}^k\,\|g^{(k)}\|_\infty .
\end{align*}
\end{lemma}

\begin{proof}
We refer readers to Chapter 12 from \cite{15_boor_splinesbook} or \cite{22_boor_splineapproximation}.
\end{proof}

\begin{lemma}
\label{thm: lemma_5}
Let $t=(t_i)_{i=1}^{n+k}$ satisfy $t_1=\cdots=t_k=a<\cdots<b=t_{n+1}=\cdots=t_{n+k}$
and let $\Delta_{\max}:=\max_i(t_{i+1}-t_i)$. Let $S_{k,t}$ be the spline space of order $k$
(degree $\le k-1$) with knots $t$. If $g\in C^k([a,b])$, then there exist B-Spline
interpolation operators $Q\colon C([a,b])\to S_{k,t}$ and constants
$C_{k,j}$ (depending only on $k$ and $j$) such that
\begin{align*}
\|g^{(j)} - (Qg)^{(j)}\|_\infty \le C_{k,j}\,\Delta_{\max}^{k-j}\,\|g^{(k)}\|_\infty,
\qquad j=0,1,\ldots,k-1,
\end{align*}
provided $j\le \lfloor k/2\rfloor$. For $j>\lfloor k/2\rfloor$, the same bound holds
with $C_{k,j}$ replaced by a constant that depends additionally on the
local mesh ratio $\max_i(\Delta t_i/\Delta t_{i-1})$.
\end{lemma}
\begin{proof}
The derivative bounds follow from Theorem~(26) in Chapter~XII of
\cite{15_boor_splinesbook}; see also Theorems~5.7 and~5.8 of
\cite{51_lyche_derivative_error_bounds}.

\end{proof}

% \begin{lemma_4}
\begin{lemma}
\label{thm: lemma_4}
Let $\mu(t)$ be a well defined B-Spline interpolant that interpolates $f(t) \in \mathcal{C}^{m}[a,b]$ with observation sites at $\{t_{i}\}_{i=0}^{n}$,  $ \mu(t)= \sum_{i=0}^{n} c_{i,m} \mathcal{B}_{i,m} (t)$, where $\mathcal{B}_{i,m}$ is the B-Spline polynomial of degree $m$ as defined in Equation \ref{eq: Cox-De Boor Recursion}. Then, the derivative of the interpolant function can be written as:
\begin{align*}
\frac{d\mu(t)}{dt}= \sum_{i=1}^{n} m \times c_{i,m} \Big\{ \frac{\mathcal{B}_{i,m-1}(t)}{t_{i+m}-t_{i}}-\frac{\mathcal{B}_{i+1,m-1}(t)}{t_{i+m+1}-t_{i+1}}\Big\}.
\end{align*}
\end{lemma}
% \end{lemma_4}

\begin{proof}
The result follows from the following classical result in spline theory, which says that let
\begin{align*}
\mathcal{B}_{i,0}=
\begin{cases}
    1 & \text{if} \quad t\in[t_{i}, t_{i+1}), \\
    0 & \text{otherwise}.
\end{cases}
\end{align*}
With a higher degree B-Spline basis defined as:
\begin{align*}
\mathcal{B}_{i,m}(t)
=
\frac{t - t_i}{t_{i+m} - t_i}\, \mathcal{B}_{i,m-1}(t)
+
\frac{t_{i+m+1} - t}{t_{i+m+1} - t_{i+1}}\, \mathcal{B}_{i+1,m-1}(t).
\end{align*}
Then, the derivative of $\mathcal{B}_{i,m}$ is given by
\begin{align}
\label{eq: derivative_spline_basis}
\frac{\mathcal{B}_{i,m}(t)}{dt}
= m \times \Big\{\frac{\mathcal{B}_{i,m-1}(t)}{t_{i+m}-t_{i}}-\frac{\mathcal{B}_{i+1,m-1}(t)}{t_{i+m+1}-t_{i+1}} \Big\}.
\end{align}
This is a classical result from \citet{23_butterfield_bsplinederivative}, also presented in \citet{15_boor_splinesbook}, and can be derived using the method of divided differences.
We can thus write
\begin{align*}
\frac{d\mu(t)}{dt}= \sum_{i} c_{i,m}\frac{d\mathcal{B}_{i,m}(t)}{dt}.
\end{align*}
Utilizing Equation \ref{eq: derivative_spline_basis}, the result follows.
\end{proof}

% \clearpage
% \newpage

\section{Details of SplineFlow}

\subsection{Algorithm Pseudocode}
\label{apdx:SFAlgorithm}

\begin{algorithm}[H]
\caption{SplineFlow: Flow Matching for Dynamical Systems
with B-Spline Interpolants}
\label{alg:splineflow}
\begin{algorithmic}[1]
% \small
\State \textbf{Input:} $N$ training trajectories $\{X_i\}_{i=0}^{N-1}$, spline degree $m$, learning rate $\eta$, base variance $\sigma^2$, \texttt{variance}, \texttt{dynamics}
\State \textbf{Output:} velocity network $u_{\theta}$ and score network $s_\phi$
\State Initialize parameters $\theta$ and $\phi$
\While{not converged}
    \State Sample mini-batch $\mathcal{B}\subset\{X_i\}_{i=0}^{N-1}$
    \For{each $X\in\mathcal{B}$}
        \State Fit a B-spline interpolant of degree $m$ to $X$: $\mu(t)=\sum_j c_{j,m}\,\mathcal{B}_{j,m}(t)$ 
        \State Sample $t_j \sim \mathcal{U}(t_0, t_{n})$, $\ \epsilon\sim\mathcal{N}(0,I)$
        \State Sample $x_{t_j}\sim \mathcal{N}(\mu(t_j),\sigma^2(t_j)I)$
        \State Compute $u(t_j)\gets \mu'(t_j)$  \Comment{via Theorem \ref{thm: theorem_3}}
        \If{\texttt{variance} = \texttt{quadratic}}
            \State Compute $\sigma'(t_j)$ \Comment{via Equation \ref{eq:quadratic_variance}}
        \Else
            \State Set $\sigma'(t_j)\gets 0$
        \EndIf
    \EndFor

    \If{\texttt{dynamics} = \texttt{ODE}}
        \State $\mathcal{L}(\theta)\gets \frac{1}{|\mathcal{B}|}\sum_{X\in\mathcal{B}}
        \|u_{\theta}(t_j,x_{t_j})-u(t_j)\|_2^2$
        \State $\theta \gets \theta - \eta\,\nabla_\theta \mathcal{L}(\theta)$
    \Else \Comment{\texttt{dynamics} = \texttt{SDE}}
        \State $\mathcal{L}(\theta, \phi)\gets \frac{1}{|\mathcal{B}|}\sum_{X\in\mathcal{B}}\big(
        \|u_{\theta}(t_j,x_{t_j})-(u(t_j)+\epsilon\,\sigma'(t_j))\|_2^2
        +\|\lambda(t_j)\,s_\phi(t_j,x_{t_j})+\epsilon\|_2^2\big)$
        \State $\theta \gets \theta - \eta\,\nabla_\theta \mathcal{L}(\theta, \phi)$, \: $\phi \gets \phi - \eta\,\nabla_\phi \mathcal{L}(\theta, \phi)$
    \EndIf
\EndWhile
\end{algorithmic}
\end{algorithm}

\subsection{Runtime Complexity}
\label{apdx:runtime_complexity}

While using SplineFlow, the total runtime depends on (a) creating B-splines, (b) sampling values and derivatives from B-splines, and (c) forward-backward passes for each epoch. Consider a total of $N$ trajectories of dimensionality $D$, each with $n$ observations in total, interpolated with B-splines of order $m$, and let $Q$ be the number of sampling queries made to calculate conditional velocities and values. Let total epochs be $T$ and let $F$ be the average time taken for each forward-backward pass while training on the regression objective for a batch size $B$. 

\begin{table}[b]
\centering
\caption{Time taken for processing of conditional paths and velocities before training epochs begin}
\vspace{0.05in}
\label{tab:preprocessing_time}
% \small
\setlength{\tabcolsep}{8pt}
\begin{tabular}{l l c}
\toprule
\textbf{Dataset} & \textbf{Method} & \textbf{Time (s)} \\
\midrule
\multirow{2}{*}{Exponential Decay} 
    & SplineFlow & $1.116 $\\
    & TFM        & $0.934$ \\
\midrule
\multirow{2}{*}{HopperPhysics} 
    & SplineFlow & $5.558 $\\
    & TFM        & $2.407$ \\
\bottomrule
\end{tabular}
\end{table}

\begin{figure}[t]
    \centering
    \includegraphics[width=0.9\linewidth]{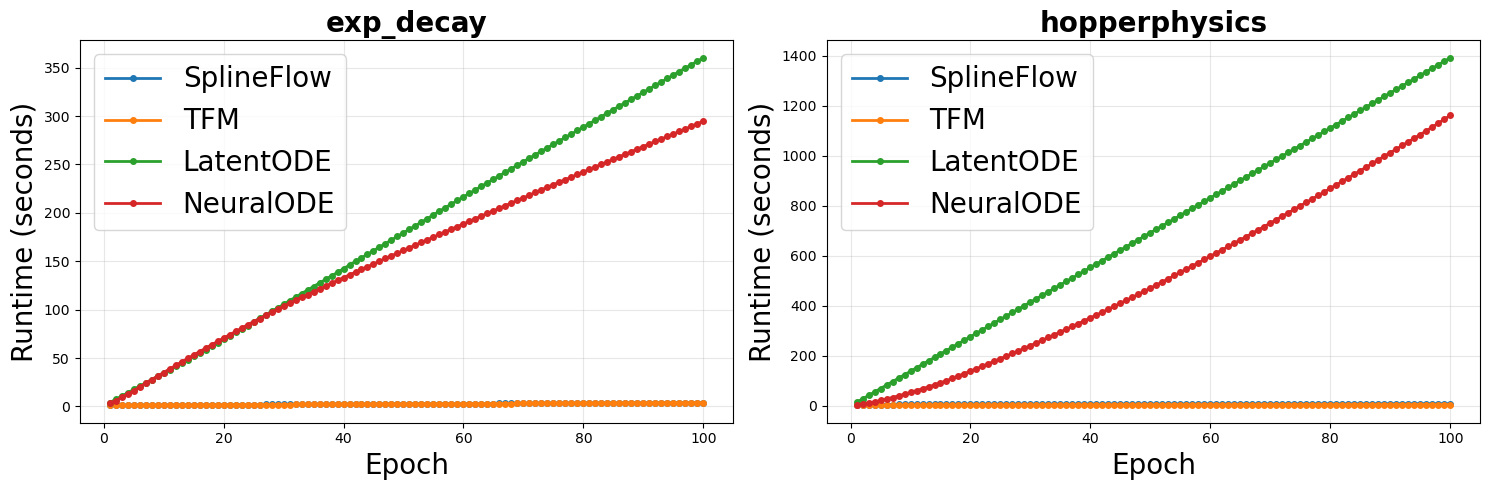}
    \vspace{-0.05in}
    \caption{Runtime of adjoint ODE-based and simulation-free flow matching methods.}
    \vspace{-0.1in}
    \label{fig:runtime_comparison_all}
\end{figure}

\begin{figure}[t]
    \centering
    \includegraphics[width=0.9\linewidth]{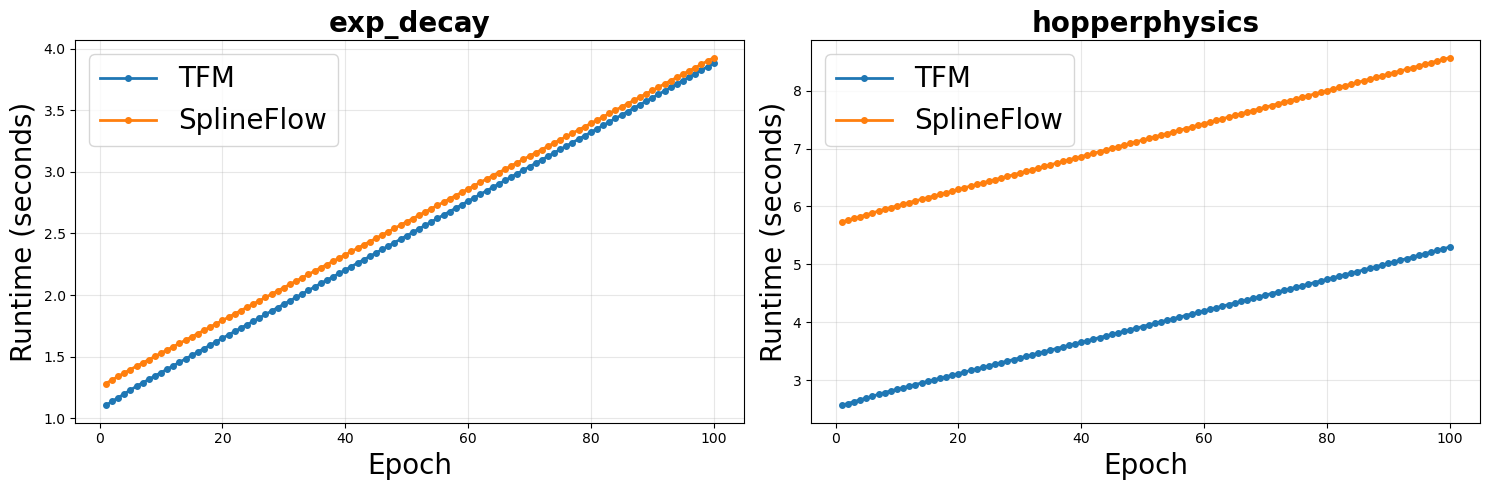}
    \vspace{-0.05in}
    \caption{Runtime comparison between flow matching methods.}
    \vspace{-0.1in}
    \label{fig:runtime_comparison_fm_methods}
\end{figure}

For every trajectory, for each observation $i$, we need to recursively create $m$ spline bases, and for each interval $t \in[t_{i}, t_{i+1})$, we have $m$ bases from neighboring observations to add. And since we do this separately for each dimension $D$, the cost comes out as $\mathcal{O}(NDnm^{2})$. 
For every $Q$ queries we evaluate from the B-spline construction (for conditional entities), for every trajectory, we again sum over $m$ bases for each query, thus bringing the cost to $\mathcal{O}(2 \times NQDm)$, together for $x_{t}$ and $v_{t}$.
The complexity for calculating precomputed conditional paths and velocities is thus $\mathcal{O}(NDm(nm+Q))$. These precomputed entities are then passed along the training algorithm, with a total time of $\mathcal{O}(TF)$ per batch $B$ and $N/B$ dataloader passes per epoch. Thus, the total complexity of SplineFlow comes out to be $\mathcal{O}(NDm(nm+Q)+TF \times (N/B))$.

We also run experiments to compare SplineFlow ($m=3$) with baselines on the smallest Exponential Decay dataset ($d=1$) and the largest HopperPhysics dataset ($d=14$).  
Figure~\ref{fig:runtime_comparison_all} demonstrates significant advantages of simulation-free methods in terms of total runtime. Owing to its encoder-decoder architecture, LatentODE has a higher runtime than NeuralODE. Whereas from Figure~\ref{fig:runtime_comparison_fm_methods}, we can see that for Exponential Decay $d=1$, both TFM and SplineFlow have almost identical runtime. For HopperPhysics, SplineFlow incurs about $3$ seconds of B-spline calculation overhead compared to the linear-interpolant calculation, as shown in Table~\ref{tab:preprocessing_time} below. However, it's important to note that this overhead is a one-time cost incurred before training and thus does not affect the overall training time.

\begin{figure}[t]
    \centering
    \includegraphics[width=0.6\linewidth]{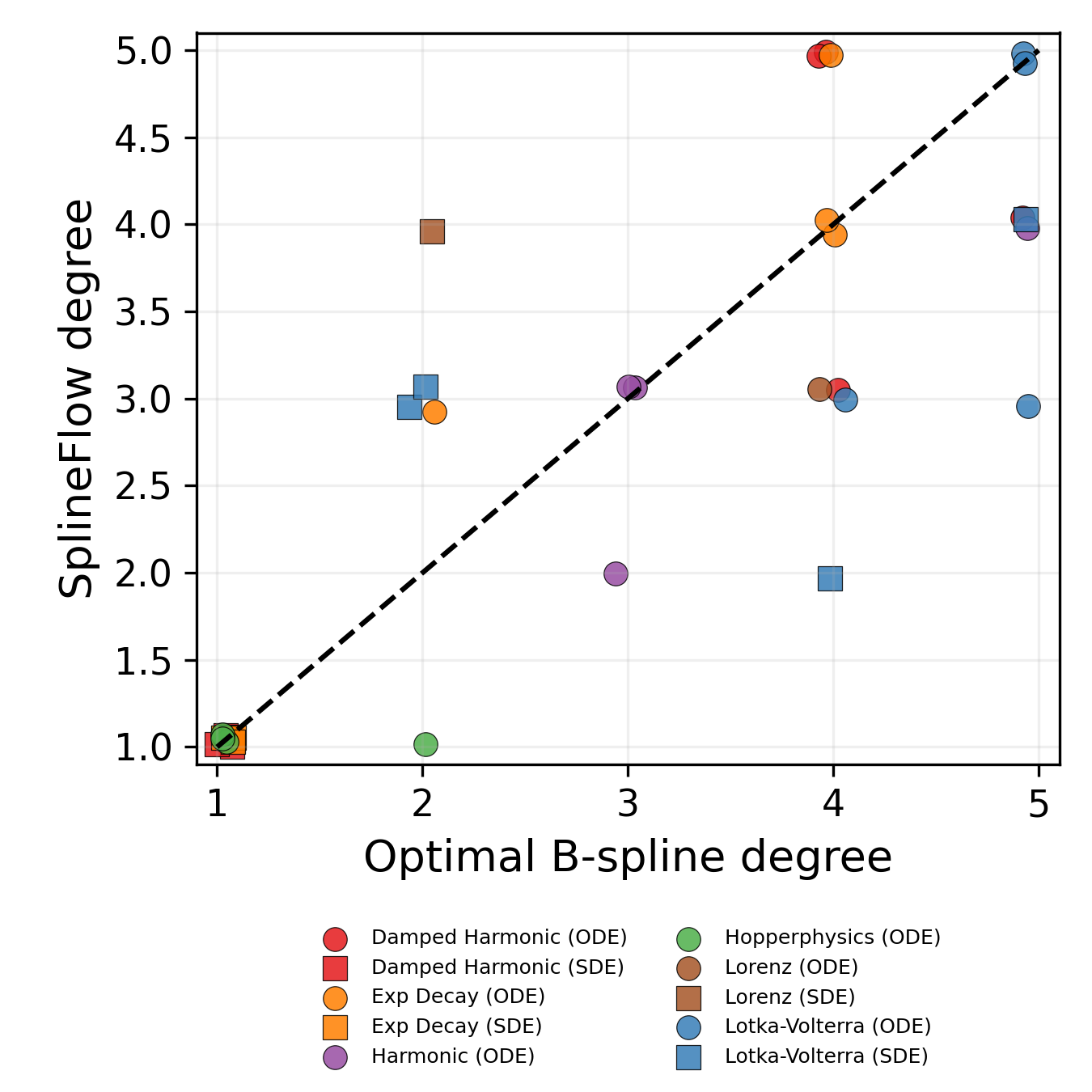}
    \vspace{-0.05in}
    \caption{Visualizations of the B-spline degree selected via cross-validation vs. the optimal degree, which corresponds to a Spearman correlation coefficient of $\rho = 0.85$.}
    \vspace{-0.05in}
    \label{fig:degree_cv_vs_sf}
\end{figure}

\begin{table*}[t]
\centering
\caption{Comparisons between methods, with SplineFlow degree $m_{\text{CV}}$ selected using cross validation.}
\label{tab:ode_combined_all_datasets_CV}
\vspace{-0.05in}
\small
\setlength{\tabcolsep}{6pt}
\resizebox{0.98\linewidth}{!}{
\begin{tabular}{llcccccc}
\toprule
$\bm{p}$ & \textbf{Model} & \textbf{Exp.} & \textbf{Harm.} & \textbf{Damp.} & \textbf{LV} & \textbf{Hopper} & \textbf{Lorenz} \\
\midrule

\multirow{5}{*}{$0$}
& NeuralODE          & $9.99\mathrm{e}{-1}$ & $6.54$ & $3.47$ & $1.98\mathrm{e}{4}$ & $2.95$ & $1.49$ \\
& LatentODE          & $9.80\mathrm{e}{-5}$ & $8.20\mathrm{e}{-4}$ & $1.21\mathrm{e}{-4}$ & $\mathbf{1.56\mathrm{e}{-3}}$ & $1.89$ & $1.16$ \\
& MMFM               & $1.40\mathrm{e}{-3}$ & $1.40\mathrm{e}{-2}$ & $2.00\mathrm{e}{-3}$ & $4.87\mathrm{e}{-1}$ & $1.82\mathrm{e}{1}$ & $8.02\mathrm{e}{-1}$ \\
& TFM                & $6.56\mathrm{e}{-4}$ & $1.81\mathrm{e}{-1}$ & $1.52\mathrm{e}{-2}$ & $1.84\mathrm{e}{-1}$ & $1.55$ & $2.44$ \\
& \textbf{SplineFlow (CV)}
                     & $\mathbf{8.90\mathrm{e}{-5}}$ & $\mathbf{5.28\mathrm{e}{-4}}$ & $\mathbf{7.20\mathrm{e}{-5}}$ & $5.53\mathrm{e}{-2}$ & $\mathbf{1.41}$ & $\mathbf{7.77\mathrm{e}{-1}}$ \\

\midrule
\multirow{4}{*}{$0.25$}
& LatentODE          & $9.80\mathrm{e}{-5}$ & $1.22\mathrm{e}{-3}$ & $3.42\mathrm{e}{-4}$ & $\mathbf{9.42\mathrm{e}{-4}}$ & $1.60$ & -- \\
& MMFM               & $2.00\mathrm{e}{-3}$ & $1.40\mathrm{e}{-2}$ & $1.80\mathrm{e}{-3}$ & $3.72\mathrm{e}{-1}$ & $3.13\mathrm{e}{1}$ & -- \\
& TFM                & $6.86\mathrm{e}{-4}$ & $1.74\mathrm{e}{-1}$ & $1.55\mathrm{e}{-2}$ & $2.59\mathrm{e}{-1}$ & $1.51$ & -- \\
& \textbf{SplineFlow (CV)}
                     & $\mathbf{8.20\mathrm{e}{-5}}$ & $\mathbf{6.53\mathrm{e}{-4}}$ & $\mathbf{5.10\mathrm{e}{-5}}$ & $1.00\mathrm{e}{-1}$ & $\mathbf{1.43}$ & -- \\

\midrule
\multirow{4}{*}{$0.5$}
& LatentODE          & $1.56\mathrm{e}{-4}$ & $7.64\mathrm{e}{-4}$ & $7.25\mathrm{e}{-4}$ & $\mathbf{1.55\mathrm{e}{-3}}$ & $\mathbf{1.65}$ & -- \\
& MMFM               & $1.70\mathrm{e}{-3}$ & $1.20\mathrm{e}{-2}$ & $3.00\mathrm{e}{-3}$ & $4.41$ & $4.78\mathrm{e}{1}$ & -- \\
& TFM                & $9.44\mathrm{e}{-4}$ & $1.77\mathrm{e}{-1}$ & $1.61\mathrm{e}{-2}$ & $4.26$ & $1.72$ & -- \\
& \textbf{SplineFlow (CV)}
                     & $\mathbf{1.06\mathrm{e}{-4}}$ & $\mathbf{5.46\mathrm{e}{-4}}$ & $\mathbf{4.70\mathrm{e}{-5}}$ & $7.05\mathrm{e}{-1}$ & $1.83$ & -- \\

\midrule
\multirow{4}{*}{$0.75$}
& LatentODE          & $9.60\mathrm{e}{-5}$ & $\mathbf{7.94\mathrm{e}{-4}}$ & $9.71\mathrm{e}{-4}$ & $\mathbf{2.90\mathrm{e}{-3}}$ & $3.56$ & -- \\
& MMFM               & $2.40\mathrm{e}{-3}$ & $1.00\mathrm{e}{-2}$ & $4.00\mathrm{e}{-3}$ & $8.62$ & $6.11\mathrm{e}{1}$ & -- \\
& TFM                & $1.35\mathrm{e}{-3}$ & $2.05\mathrm{e}{-1}$ & $3.12\mathrm{e}{-2}$ & $1.98$ & $\mathbf{2.81}$ & -- \\
& \textbf{SplineFlow (CV)}
                     & $\mathbf{8.60\mathrm{e}{-5}}$ & $1.31\mathrm{e}{-3}$ & $\mathbf{4.80\mathrm{e}{-5}}$ & $1.78$ & $3.34$ & -- \\

\midrule
\multicolumn{8}{c}{\textbf{Avg. Training Runtime (s)}} \\
\midrule
& LatentODE
& $2.80\mathrm{e}{4}$
& $4.47\mathrm{e}{4}$
& $4.38\mathrm{e}{4}$
& $4.44\mathrm{e}{4}$
& $1.21\mathrm{e}{5}$
& $5.30\mathrm{e}{4}$ \\

& \textbf{SplineFlow (CV)}
& $\mathbf{1.99\mathrm{e}{2}}$
& $\mathbf{2.88\mathrm{e}{2}}$
& $\mathbf{2.80\mathrm{e}{2}}$
& $\mathbf{2.83\mathrm{e}{2}}$
& $\mathbf{2.32\mathrm{e}{2}}$
& $\mathbf{2.95\mathrm{e}{2}}$ \\

& Speedup ($\times$)
& $\mathbf{141}$
& $\mathbf{155}$
& $\mathbf{156}$
& $\mathbf{157}$
& $\mathbf{520}$
& $\mathbf{180}$ \\

\bottomrule
\end{tabular}
}
\vspace{-0.05in}
\end{table*}

\section{Hyperparameter Selection}
\label{apdx:hyperparameter_selection}

\shortsection{Degree of B-Spline Interpolant $m$}
The degree $m$ could be selected by performing a simple cross-validation (CV) on training data. Let $\{X^{i}\}_{i\in[1\cdots N]}$ be observed trajectories as defined in Section~\ref{sec:splineflow modeling dynamical systems}. For each trajectory $X^{i}$ create a fit $X^{i}_{\text{fit}}$ and test $X^{i}_{\text{test}}$ subsets such that $X^{i}_{\text{fit}} \cup X^{i}_{\text{test}}=X^{i}$. Fit a B-spline interpolant as defined in Equation~\ref{eq: B-spline interpolant function} ($\mathcal{I}_{m}$) of degree $m$ such that the interpolant satisfies observations at fit subset $\mathcal{I}_{m}(X^{i}_{\text{fit}})=X^{i}_{\text{fit}}$, and calculate the error on $X^{i}_{\text{test}}$ as $\sum_{i \in \text{training}}\|\mathcal{I}_{m}(X^{i}_{\text{test}})-X^{i}_{\text{test}}\|^{2}$. Then, we can select the degree that minimizes the error.

Importantly, the above degree selection rule uses only training data, requires no training of a flow model, and evaluates B-spline interpolation errors for the candidate degrees. Thus, its computational overhead is negligible relative to the subsequent model training cost. Below we present the CV selected degree ($m_{\text{CV}}$) performance as opposed to the optimal degree ($m_{\text{oracle}}$) performance presented in Tables~\ref{tab:ode_combined_all_datasets} and~\ref{tab:sde_results_condensed}. Comparing Tables~\ref{tab:ode_combined_all_datasets} and~\ref{tab:ode_combined_all_datasets_CV} for the ODE case, and Tables~\ref{tab:sde_w2_cv} and~\ref{tab:sde_results_condensed} for SDE case we can see that SplineFlow with CV selected degree still retains its performance gains compared to existing baselines, with the CV selected degree ($m_{\text{CV}}$) often lying close to the optimal degree ($m_{\text{oracle}}$) as seen in Table~\ref{tab:cv_oracle_ode_sde} below.

% \begin{table}[h]
% \centering
% \caption{$m_{\mathrm{CV}}$ vs.\ $m_{\mathrm{oracle}}$ for ODE. Each entry reports $m_{\mathrm{CV}}/m_{\mathrm{oracle}}$.}
% \label{tab:ode_cv_oracle_degree}
% \small
% \setlength{\tabcolsep}{6pt}

% \begin{tabular}{lcccc}
% \toprule
% \textbf{Setting} & $\bm{p=0}$ & $\bm{p=0.25}$ & $\bm{p=0.5}$ & $\bm{p=0.75}$ \\
% \midrule
% Exp.    & $4/4$ & $4/4$ & $4/4$ & $4/5$ \\
% Harm.   & $3/2$ & $3/3$ & $3/3$ & $5/4$ \\
% Damp.   & $4/3$ & $4/5$ & $4/5$ & $5/4$ \\
% LV      & $4/3$ & $5/5$ & $5/3$ & $5/5$ \\
% Hopper  & $1/1$ & $1/1$ & $1/1$ & $1/1$ \\
% Lorenz  & $4/3$ & --    & --    & --    \\
% \bottomrule
% \end{tabular}
% \end{table}

\begin{table}[t]
\centering
\caption{$\mathcal{W}_2$ comparison of different methods, with SplineFlow degree $m_{\text{CV}}$ selected using the CV rule above for SDE systems.}
\vspace{0.05in}
\label{tab:sde_w2_cv}
\small
\setlength{\tabcolsep}{6pt}
\begin{tabular}{lcccc}
\toprule
\textbf{Setting} & $\bm{p}$ & \textbf{SplineFlow (CV)} & \textbf{SplineFlow} & \textbf{SF2M} \\
\midrule
\multirow{2}{*}{Exp.}
& $0$   & $3.31\mathrm{e}{-1}$          & $3.31\mathrm{e}{-1}$          & $\mathbf{3.05\mathrm{e}{-1}}$ \\
& $0.5$ & $\mathbf{1.95\mathrm{e}{-1}}$ & $\mathbf{1.95\mathrm{e}{-1}}$ & $2.04\mathrm{e}{-1}$ \\
\midrule
\multirow{2}{*}{Damp.}
& $0$   & $\mathbf{4.10\mathrm{e}{-1}}$ & $\mathbf{4.10\mathrm{e}{-1}}$ & $4.30\mathrm{e}{-1}$ \\
& $0.5$ & $3.21\mathrm{e}{-1}$          & $3.21\mathrm{e}{-1}$          & $\mathbf{3.19\mathrm{e}{-1}}$ \\
\midrule
\multirow{2}{*}{LV}
& $0$   & $3.21\mathrm{e}{-1}$          & $\mathbf{2.94\mathrm{e}{-1}}$ & $4.43\mathrm{e}{-1}$ \\
& $0.5$ & $3.74\mathrm{e}{-1}$          & $\mathbf{3.10\mathrm{e}{-1}}$ & $3.75\mathrm{e}{-1}$ \\
\midrule
Lorenz
& $0$   & $4.30\mathrm{e}{-1}$          & $\mathbf{1.80\mathrm{e}{-1}}$ & $6.32\mathrm{e}{-1}$ \\
\bottomrule
\end{tabular}
\end{table}

% \begin{table}[h]
% \centering
% \caption{$m_{\mathrm{CV}}$ vs.\ $m_{\mathrm{oracle}}$ for SDE systems. Each entry reports $m_{\mathrm{CV}}/m_{\mathrm{oracle}}$.}
% \label{tab:sde_cv_oracle_degree}
% \small
% \setlength{\tabcolsep}{6pt}
% \begin{tabular}{lcc}
% \toprule
% \textbf{Setting} & $\bm{p=0}$ & $\bm{p=0.5}$ \\
% \midrule
% Exp.    & $1/1$ & $1/1$ \\
% Damp.   & $1/1$ & $1/1$ \\
% LV      & $4/2$ & $2/3$ \\
% Lorenz  & $2/4$ & --    \\
% \bottomrule
% \end{tabular}
% \end{table}

\begin{table*}[h]
\centering
\caption{$m_{\mathrm{CV}}$ vs.\ $m_{\mathrm{oracle}}$ for ODE and SDE systems. Each entry reports $m_{\mathrm{CV}}/m_{\mathrm{oracle}}$.}
\label{tab:cv_oracle_ode_sde}
\vspace{-0.05in}
\small

\begin{minipage}[t]{0.62\textwidth}
\centering
\textbf{(a) ODE}
\vspace{0.05in}

\setlength{\tabcolsep}{5pt}
\begin{tabular}{lcccc}
\toprule
\textbf{Setting} & $\bm{p=0}$ & $\bm{p=0.25}$ & $\bm{p=0.5}$ & $\bm{p=0.75}$ \\
\midrule
Exp.    & $4/4$ & $4/4$ & $4/4$ & $4/5$ \\
Harm.   & $3/2$ & $3/3$ & $3/3$ & $5/4$ \\
Damp.   & $4/3$ & $4/5$ & $4/5$ & $5/4$ \\
LV      & $4/3$ & $5/5$ & $5/3$ & $5/5$ \\
Hopper  & $1/1$ & $1/1$ & $1/1$ & $1/1$ \\
Lorenz  & $4/3$ & --    & --    & --    \\
\bottomrule
\end{tabular}
\end{minipage}
\hfill
\begin{minipage}[t]{0.34\textwidth}
\centering
\textbf{(b) SDE}
\vspace{0.05in}

\setlength{\tabcolsep}{6pt}
\begin{tabular}{lcc}
\toprule
\textbf{Setting} & $\bm{p=0}$ & $\bm{p=0.5}$ \\
\midrule
Exp.    & $1/1$ & $1/1$ \\
Damp.   & $1/1$ & $1/1$ \\
LV      & $4/2$ & $2/3$ \\
Lorenz  & $2/4$ & --    \\
\bottomrule
\end{tabular}
\end{minipage}

\vspace{-0.05in}
\end{table*}

\begin{table}[t]
\centering
\caption{Ablation over $\sigma$ for SDE systems}
\label{tab:sde_sigma_ablation}
\small
\setlength{\tabcolsep}{5pt}

\begin{tabular}{llcccccc}
\toprule
$\bm{p}$ & \textbf{Dataset}
& $\bm{10^{-4}}$
& $\bm{10^{-3}}$
& $\bm{10^{-2}}$
& $\bm{10^{-1}}$
& $\bm{10^{0}}$
& $\bm{10^{1}}$ \\
\midrule

$0$ & Exp.
& $0.496$
& $0.477$
& $0.447$
& $\mathbf{0.253}$
& $6.124$
& $116.150$ \\

$0$ & Damp.
& $0.473$
& $0.513$
& $\mathbf{0.377}$
& $0.402$
& $5.478$
& $85.405$ \\

$0$ & LV
& $2.954$
& $7.641$
& $1.163$
& $\mathbf{0.346}$
& $0.476$
& $0.898$ \\

$0$ & Lorenz
& $0.201$
& $0.210$
& $\mathbf{0.185}$
& $0.219$
& $0.587$
& $3.379$ \\

\midrule

$0.5$ & Exp.
& $0.426$
& $0.382$
& $0.392$
& $\mathbf{0.142}$
& $7.031$
& $114.271$ \\

$0.5$ & Damp.
& $0.499$
& $0.490$
& $0.502$
& $\mathbf{0.278}$
& $6.406$
& $89.495$ \\

$0.5$ & LV
& $16.251$
& $7.649$
& $4.577$
& $\mathbf{0.342}$
& $0.428$
& $0.863$ \\

\bottomrule
\end{tabular}
\vspace{-0.05in}
\end{table}

\shortsection{Constant Diffusion Schedule $\sigma$} Consider trajectory $X^{i}$, $i\in \{1,2,\cdots,N\}$, sampled from the SDE evolution with constant diffusion schedule mentioned in Equation~\ref{eq: SDE}: $dx_{t}= u_{t}(x_{t})dt + \sigma dw_{t}$, with $x \in \mathbb{R}^{d}$. Let trajectory $i$ contain $n_{i}$ observations $\{x_{i}({t^{i}_{0}}), \cdots, x_{i}({t^{i}_{n_{i}}})\}$, then the constant diffusion term $\sigma$ can then be estimated using a discrete estimator below:
\begin{equation}
\label{eq:pooled_sigma_est}
\hat{\sigma}^{2}
=
\frac{1}{d \sum_{i=1}^{N} T_i}
\sum_{i=1}^{N}
\sum_{k=0}^{n_i-1}
\bigl\lVert
x_{i}({t^{i}_{k+1}}) - x_{i}({t^{i}_{k}})
\bigr\rVert_2^{2},
\end{equation}
where $T_{i}=\sum_{k=1}^{n_{i}}(t_{k+1}-t_{k})$ is the sum of time differences.
This estimator follows from classical results on quadratic variation of Ito diffusions, which imply that the sum of squared increments converges to the integrated diffusion coefficient \citep{43_sigmaestimate_sorenson,44_sigmaestimate_karatzas}.
Note that the estimator is meant to provide a guide for selecting the hyperparameter $\sigma$; however, the input ($\sigma$) to Algorithm~\ref{alg:splineflow} should still be at the discretion of the user. 
In addition, we conduct an ablation over the diffusion scale $\sigma$ for SDE under $\mathcal{W}_{2}$ with SplineFlow using the degree selected by training-only CV ($m_{CV}$). As shown in the Table~\ref{tab:sde_sigma_ablation} below, the best performance is generally obtained at moderate diffusion scales within $[0.01, 0.1]$. Compared to Table~\ref{tab:sde_results_condensed}, the results obtained using our estimated values are close to the best values in the ablation, thus supporting the efficacy of the estimation rule.

\section{Training Details}
\label{apdx:training}
  
\subsection{Hyperparameters}
 
We use similar neural network architectures and training configurations across SplineFlow and all baselines.
Specifically, we set:
\begin{itemize}
    \item Learning Rate (LR): $0.0005$
    \item LR scheduler: Cosine
    \item Epochs: $10,000$
    \item NN: MLP
    \item Depth: $4$
    \item Width for ODE experiments: $256$ ($2048$ for Lorenz and HopperPhysics owing to dimensionality and complexity)
    \item Width for SDE experiments: $2 \times 256$ ($2 \times 2048$ for Lorenz owing to complexity) for both velocity field ($u_\theta$) and score network ($s_\phi$)
    \item Width for cellular trajectory experiments: $64$
\end{itemize}
 
\subsection{Loss Curves}
\label{apdx:loss_curves}
 
Figures~\ref{fig:loss_ode} show training loss curves for training SplineFlow on ODE dynamics. SplineFlow consistently converges to a near-zero training loss
validating that more sophisticated supervision signals do not incur additional training overhead.

\begin{figure}[t]
    \centering
    % ---- Row 1: two ODE systems where SplineFlow wins ----
    \begin{subfigure}[t]{0.32\linewidth}
        \centering
        \includegraphics[width=\linewidth]{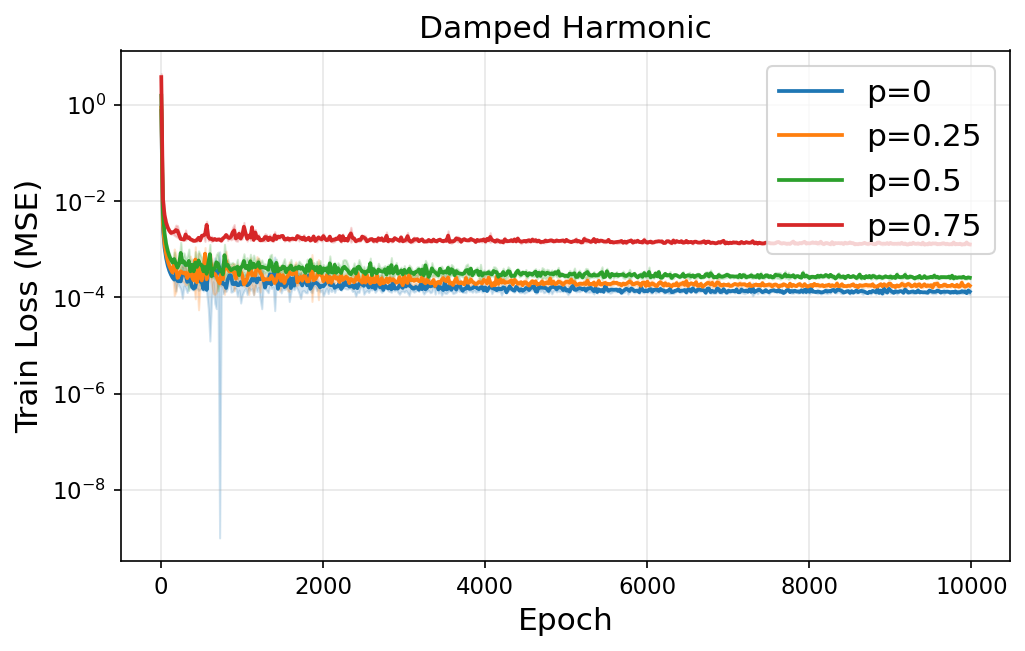}
        \caption{Damped Harmonic}
    \end{subfigure}
    \hfill
    \begin{subfigure}[t]{0.32\linewidth}
        \centering
        \includegraphics[width=\linewidth]{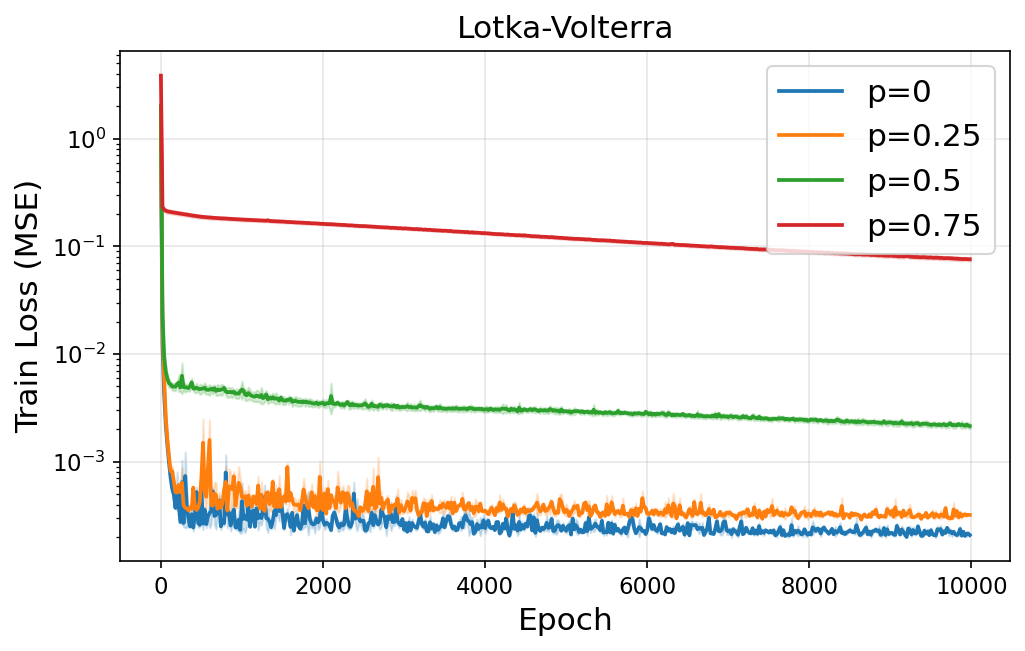}
        \caption{Lotka--Volterra}
    \end{subfigure}
    \hfill
    \begin{subfigure}[t]{0.32\linewidth}
        \centering
        \includegraphics[width=\linewidth]{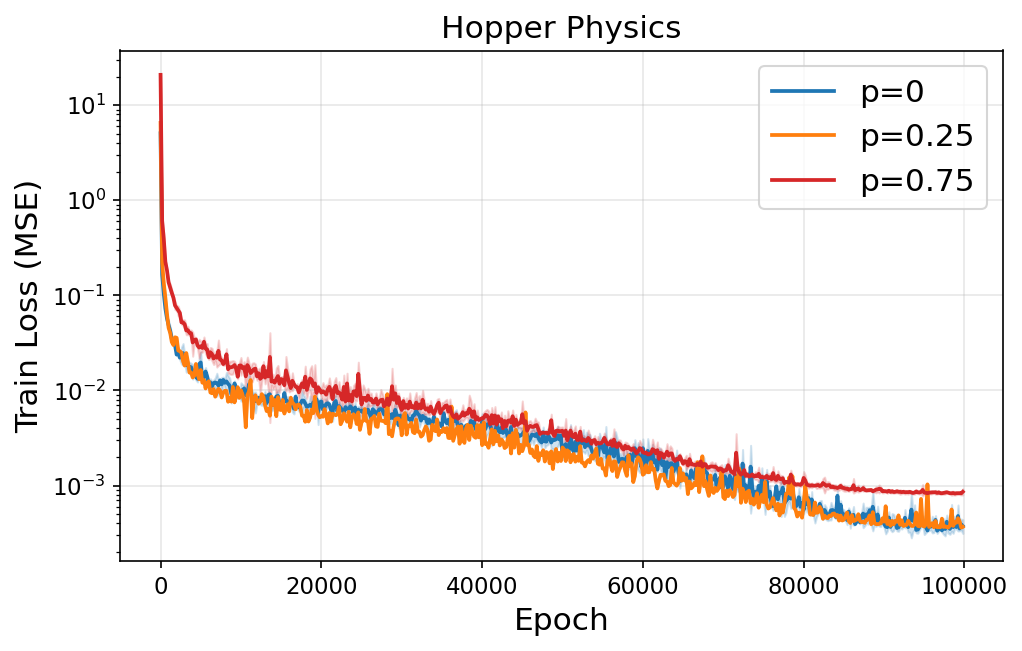}
        \caption{HopperPhysics}
    \end{subfigure}
    % ---- Row 2: Lorenz ODE ----
    \vspace{0.5em}
    \begin{subfigure}[t]{0.32\linewidth}
        \centering
        \includegraphics[width=\linewidth]{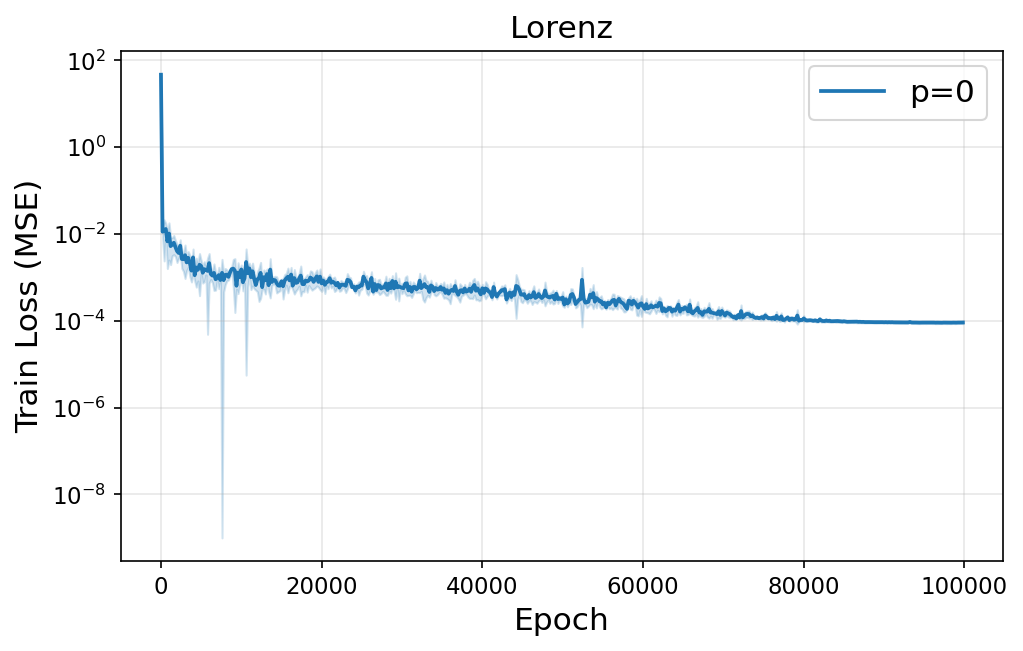}
        \caption{Lorenz}
    \end{subfigure}
    \hfill
    \begin{subfigure}[t]{0.32\linewidth}
        \centering
        \includegraphics[width=\linewidth]{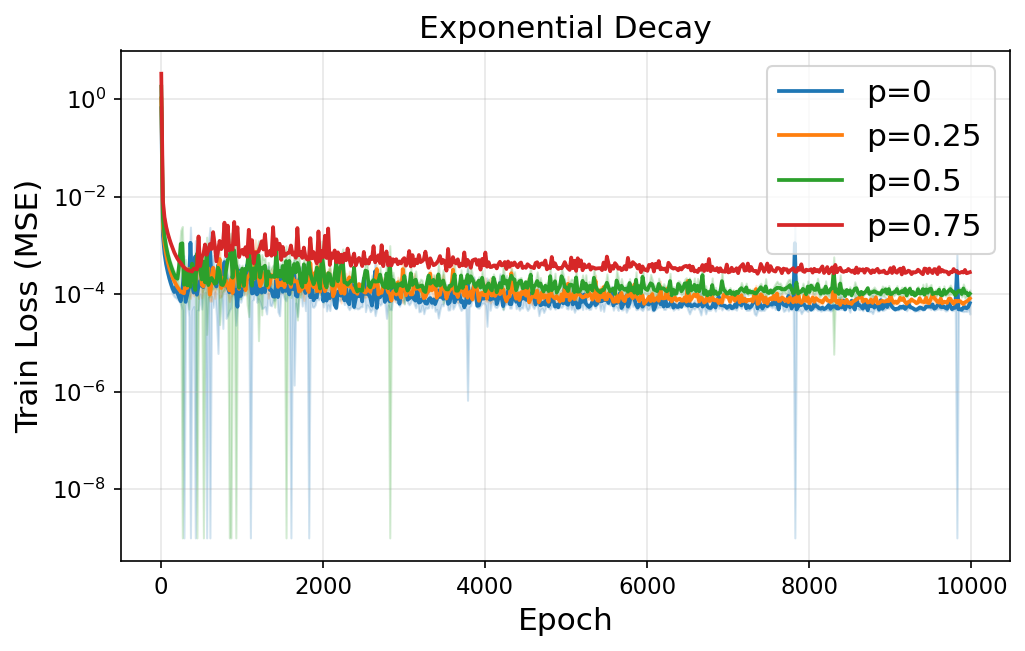}
        \caption{Exponential Decay}
    \end{subfigure}
    \hfill
    \begin{subfigure}[t]{0.32\linewidth}
        \centering
        \includegraphics[width=\linewidth]{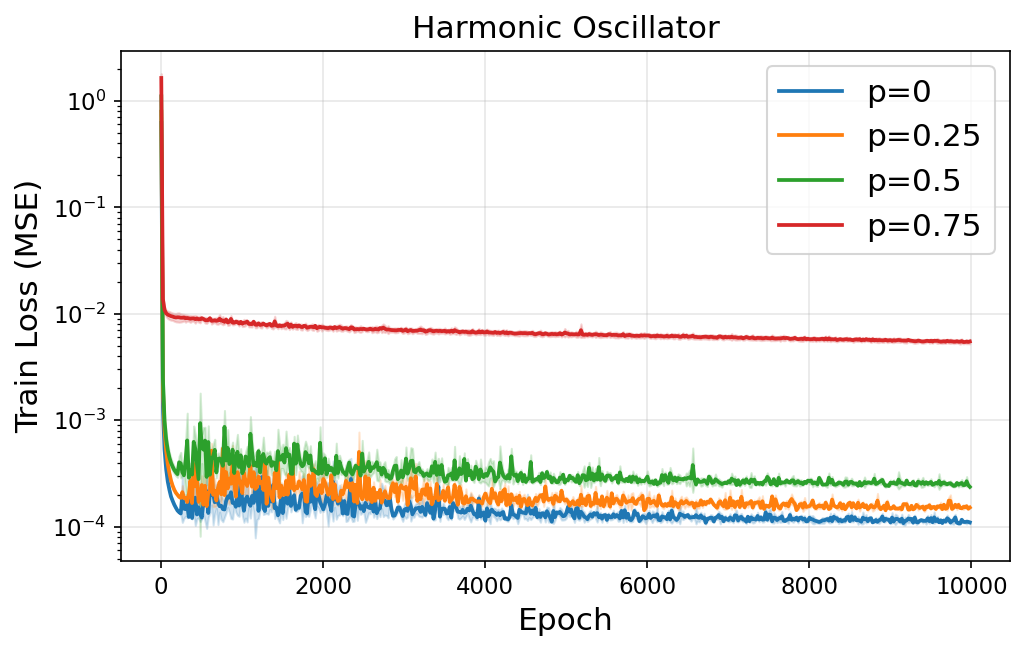}
        \caption{Harmonic Oscillator}
    \end{subfigure}
    \vspace{-0.1in}
    \caption{%
        Training loss (log scale) vs.\ epoch for SplineFlow on ODE
        systems ($p=0$, regular sampling).  Shaded bands show the standard
        deviation across five runs.
    }
    \label{fig:loss_ode}
\end{figure}

\section{Full Experimental Results}
\label{apdx:expanded_results}

This section presents the full experimental results and summarized heatmap visualizations for SplineFlow with the best-selected spline degree parameter $m$, benchmarked against baseline methods across all evaluated datasets and configurations.

\begin{table}[H]
\centering
\caption{Comparisons of different methods for modeling ODE dynamics in MSE metric across $5$ deterministic dynamical systems and varying sampling irregularity $p\in\{0, 0.25, 0.5, 0.75\}$.}
\label{tab: apdx_ode_combined_all_datasets}
\vspace{0.01in}
% \small
\setlength{\tabcolsep}{8pt}
\renewcommand{\arraystretch}{1.05}
\resizebox{\linewidth}{!}{
\begin{tabular}{llcccccc}
\toprule
$p$ & 
Model &
Exp-Decay &
Harmonic  &
Damped Harmonic  &
Lotka--Volterra &
HopperPhysics & Lorenz \\

\midrule
\multirow{5}{*}{$0$} &
NeuralODE &
$0.999 \pm 0.126$ &
$6.542 \pm 5.965$ &
$3.474 \pm 0.366$ &
$1.98\mathrm{e}{4} \pm 1.24\mathrm{e}{4}$ &
$2.947 \pm 0.183$ &
$1.491 \pm 0.359$ \\

& LatentODE &
$9.8\mathrm{e}{-5} \pm 4.7\mathrm{e}{-5}$ &
$8.2\mathrm{e}{-4} \pm 4.2\mathrm{e}{-4}$ &
$1.2\mathrm{e}{-4} \pm 6.3\mathrm{e}{-4}$ &
$\mathbf{1.6\mathrm{e}{-3} \pm 1.2\mathrm{e}{-3}}$ &
$1.894 \pm 0.523$ &
$1.157 \pm 0.204$ \\

& MMFM &
$1.4\mathrm{e}{-3} \pm 1.0\mathrm{e}{-4}$ &
$0.014 \pm 0.009$ &
$0.002 \pm 0.002$ &
$0.487 \pm 0.384$ &
$18.179 \pm 4.049$ &
$0.802 \pm 0.008$ \\

& TFM &
$6.6\mathrm{e}{-4} \pm 8.3\mathrm{e}{-5}$ &
$0.181 \pm 3.8\mathrm{e}{-4}$ &
$0.015 \pm 2.7\mathrm{e}{-4}$ &
$0.184 \pm 0.104$ &
$1.554 \pm 0.297$ &
$2.442 \pm 0.076$ \\

& SplineFlow &
$\mathbf{8.9\mathrm{e}{-5} \pm 3.3\mathrm{e}{-5}}$ &
$\mathbf{3.7\mathrm{e}{-4} \pm 3.1\mathrm{e}{-4}}$ &
$\mathbf{9.5\mathrm{e}{-5} \pm 1.0\mathrm{e}{-4}}$ &
$0.086 \pm 0.043$ &
$\mathbf{1.410 \pm 0.172}$ &
$\mathbf{0.639 \pm 0.004}$ \\

\midrule
\multirow{4}{*}{$0.25$} 

& LatentODE &
$9.8\mathrm{e}{-5} \pm 3.7\mathrm{e}{-5}$ &
$0.001 \pm 0.001$ &
$3.4\mathrm{e}{-4} \pm 5.8\mathrm{e}{-4}$ &
$\mathbf{9e-4 \pm 3.3\mathrm{e}{-4}}$ &
$1.605 \pm 0.019$ &
-- \\

& MMFM &
$2.0\mathrm{e}{-3} \pm 8.0\mathrm{e}{-4}$ &
$0.014 \pm 0.006$ &
-- &
$0.372 \pm 0.213$ &
$31.294 \pm 8.263$ &
-- \\

& TFM &
$6.9\mathrm{e}{-4} \pm 9.6\mathrm{e}{-5}$ &
$0.174 \pm 0.008$ &
$0.015 \pm 2.6\mathrm{e}{-4}$ &
$0.259 \pm 0.172$ &
$1.512 \pm 0.030$ &
-- \\

& SplineFlow &
$\mathbf{8.2\mathrm{e}{-5} \pm 2.9\mathrm{e}{-5}}$ &
$\mathbf{6.5\mathrm{e}{-4} \pm 2.1\mathrm{e}{-4}}$ &
$\mathbf{3.0\mathrm{e}{-5} \pm 4.8\mathrm{e}{-6}}$ &
$0.100 \pm 0.057$ &
$\mathbf{1.433 \pm 0.030}$ &
-- \\

\midrule
\multirow{4}{*}{$0.5$} 

& LatentODE &
$1.6\mathrm{e}{-4} \pm 4.5\mathrm{e}{-5}$ &
$7.6\mathrm{e}{-4} \pm 6.7\mathrm{e}{-5}$ &
$7.3\mathrm{e}{-4} \pm 3.6\mathrm{e}{-4}$ &
$\mathbf{1.5e-3 \pm 0.001}$ &
$\mathbf{1.649 \pm 0.073}$ &
-- \\

& MMFM &
$1.7\mathrm{e}{-3} \pm 3.0\mathrm{e}{-4}$ &
$0.012 \pm 0.002$ &
$0.003 \pm 0.001$ &
$4.41 \pm 0.223$ &
$47.836 \pm 4.026$ &
-- \\

& TFM &
$9.4\mathrm{e}{-4} \pm 1.7\mathrm{e}{-4}$ &
$0.177 \pm 0.018$ &
$0.016 \pm 0.001$ &
$4.259 \pm 5.690$ &
$1.718 \pm 0.092$ &
-- \\

& SplineFlow &
$\mathbf{1.1\mathrm{e}{-4} \pm 6.8\mathrm{e}{-5}}$ &
$\mathbf{5.5\mathrm{e}{-4} \pm 2.2\mathrm{e}{-4}}$ &
$\mathbf{4.7\mathrm{e}{-5} \pm 2.1\mathrm{e}{-5}}$ &
$0.705 \pm 1.210$ &
$1.825 \pm 0.602$ &
-- \\

\midrule
\multirow{4}{*}{$0.75$} 

& LatentODE &
$9.6\mathrm{e}{-5} \pm 5.0\mathrm{e}{-6}$ &
$\mathbf{7.9\mathrm{e}{-4} \pm 2.5\mathrm{e}{-4}}$ &
$9.7\mathrm{e}{-4} \pm 3.7\mathrm{e}{-4}$ &
$\mathbf{3e-3 \pm 0.001}$ &
$3.563 \pm 0.148$ &
-- \\

& MMFM &
$2.4\mathrm{e}{-3} \pm 4.0\mathrm{e}{-4}$ &
$0.010 \pm 0.002$ &
$0.004 \pm 0.0005$ &
$8.62 \pm 0.038$ &
$61.110 \pm 19.287$ &
-- \\

& TFM &
$1.4\mathrm{e}{-3} \pm 2.0\mathrm{e}{-4}$ &
$0.205$ &
$0.031 \pm 0.002$ &
$1.976 \pm 1.810$ &
$\mathbf{2.810 \pm 0.165}$ &
-- \\

& SplineFlow &
$\mathbf{8.6\mathrm{e}{-5} \pm 6.4\mathrm{e}{-6}}$ &
$0.001 \pm 0.001$ &
$\mathbf{4.8\mathrm{e}{-5} \pm 2.3\mathrm{e}{-5}}$ &
$1.783 \pm 1.430$ &
$3.336 \pm 0.430$ &
-- \\

\bottomrule
\end{tabular}
}
\vspace{-0.05in}
\end{table}

\begin{table}[H]
\centering
\caption{Performance across SDE dynamics datasets for different metrics under regular sampling}
\vspace{0.01in}
\label{tab:apdx_sde_regular}
\renewcommand{\arraystretch}{1.1}
\resizebox{\linewidth}{!}{
\begin{tabular}{llccccc}
\toprule
\textbf{Dataset} & \textbf{Model} & \textbf{MSE (P-ODE)} & \textbf{MSE (SDE)} & \textbf{Wasserstein} & \textbf{MMD} & \textbf{Energy} \\
\midrule
\multirow{2}{*}{Exp.}
& SF2M
& $3.91\mathrm{e}{-01} \pm 2.26\mathrm{e}{-03}$
& $4.77\mathrm{e}{-01} \pm 9.90\mathrm{e}{-03}$
& $\mathbf{3.05\mathrm{e}{-01} \pm 4.85\mathrm{e}{-03}}$
& $\mathbf{2.40\mathrm{e}{-02} \pm 1.00\mathrm{e}{-03}}$
& $\mathbf{6.10\mathrm{e}{-02} \pm 2.62\mathrm{e}{-03}}$ \\
& SplineFlow
& $\mathbf{3.90\mathrm{e}{-01} \pm 4.23\mathrm{e}{-03}}$
& $\mathbf{4.55\mathrm{e}{-01} \pm 8.27\mathrm{e}{-03}}$
& $3.30\mathrm{e}{-01} \pm 2.83\mathrm{e}{-03}$
& $2.70\mathrm{e}{-02} \pm 2.48\mathrm{e}{-03}$
& $6.94\mathrm{e}{-02} \pm 1.08\mathrm{e}{-03}$ \\

\midrule
\multirow{2}{*}{Damped.}
& SF2M
& $\mathbf{8.44\mathrm{e}{-01} \pm 1.54\mathrm{e}{-02}}$
& $\mathbf{9.66\mathrm{e}{-01} \pm 3.29\mathrm{e}{-02}}$
& $4.29\mathrm{e}{-01} \pm 7.40\mathrm{e}{-03}$
& $4.84\mathrm{e}{-02} \pm 2.96\mathrm{e}{-03}$
& $1.25\mathrm{e}{-01} \pm 3.74\mathrm{e}{-03}$ \\
& SplineFlow
& $8.94\mathrm{e}{-01} \pm 5.12\mathrm{e}{-03}$
& $9.69\mathrm{e}{-01} \pm 9.41\mathrm{e}{-03}$
& $\mathbf{4.10\mathrm{e}{-01} \pm 4.41\mathrm{e}{-03}}$
& $\mathbf{4.19\mathrm{e}{-02} \pm 1.54\mathrm{e}{-03}}$
& $\mathbf{1.12\mathrm{e}{-01} \pm 7.12\mathrm{e}{-03}}$ \\

\midrule
\multirow{2}{*}{LV}
& SF2M
& $4.58\mathrm{e}{-01} \pm 9.74\mathrm{e}{-02}$
& $4.60\mathrm{e}{-01} \pm 9.67\mathrm{e}{-02}$
& $4.43\mathrm{e}{-01} \pm 4.99\mathrm{e}{-02}$
& $1.25\mathrm{e}{-01} \pm 2.35\mathrm{e}{-02}$
& $4.65\mathrm{e}{-01} \pm 9.44\mathrm{e}{-02}$ \\
& SplineFlow
& $\mathbf{2.68\mathrm{e}{-01} \pm 1.44\mathrm{e}{-02}}$
& $\mathbf{2.73\mathrm{e}{-01} \pm 1.37\mathrm{e}{-02}}$
& $\mathbf{2.93\mathrm{e}{-01} \pm 3.16\mathrm{e}{-03}}$
& $\mathbf{4.21\mathrm{e}{-02} \pm 3.45\mathrm{e}{-03}}$
& $\mathbf{1.56\mathrm{e}{-01} \pm 1.33\mathrm{e}{-02}}$ \\

\midrule
\multirow{2}{*}{Lorenz}
& SF2M
& $2.17\mathrm{e}{+00} \pm 7.94\mathrm{e}{-03}$
& $2.15\mathrm{e}{+00} \pm 8.35\mathrm{e}{-03}$
& $6.32\mathrm{e}{-01} \pm 7.61\mathrm{e}{-03}$
& $1.29\mathrm{e}{-01} \pm 2.56\mathrm{e}{-03}$
& $4.04\mathrm{e}{-01} \pm 1.08\mathrm{e}{-02}$ \\
& SplineFlow
& $\mathbf{1.23\mathrm{e}{+00} \pm 4.01\mathrm{e}{-02}}$
& $\mathbf{1.22\mathrm{e}{+00} \pm 5.62\mathrm{e}{-02}}$
& $\mathbf{1.79\mathrm{e}{-01} \pm 1.47\mathrm{e}{-02}}$
& $\mathbf{5.00\mathrm{e}{-03} \pm 1.79\mathrm{e}{-03}}$
& $\mathbf{3.66\mathrm{e}{-02} \pm 7.65\mathrm{e}{-03}}$ \\
\bottomrule

\end{tabular}
}
\end{table}

\begin{table}[H]
\centering
\caption{Performance across SDE dynamics datasets for different metrics under irregular sampling.}
\vspace{0.01in}
\renewcommand{\arraystretch}{1.1}

\label{tab:apdx_sde_irregular}
\resizebox{\linewidth}{!}{
\begin{tabular}{lllccccc}
\toprule
\textbf{Dataset} & $\bm{p}$ & \textbf{Model} &
\textbf{MSE (P-ODE)} &
\textbf{MSE (SDE)} &
\textbf{Wasserstein} &
\textbf{MMD} &
\textbf{Energy} \\
\midrule

\multirow{6}{*}{Exp.}
& \multirow{2}{*}{$0.25$}
& SF2M
& $\mathbf{3.64\mathrm{e}{-01} \pm 1.27\mathrm{e}{-03}}$
& $\mathbf{4.65\mathrm{e}{-01} \pm 6.30\mathrm{e}{-03}}$
& $\mathbf{2.45\mathrm{e}{-01} \pm 7.37\mathrm{e}{-03}}$
& $\mathbf{1.46\mathrm{e}{-02} \pm 1.15\mathrm{e}{-03}}$
& $\mathbf{4.01\mathrm{e}{-02} \pm 2.76\mathrm{e}{-03}}$ \\
&& SplineFlow
& $3.66\mathrm{e}{-01} \pm 6.08\mathrm{e}{-03}$
& $4.69\mathrm{e}{-01} \pm 5.90\mathrm{e}{-03}$
& $2.50\mathrm{e}{-01} \pm 5.53\mathrm{e}{-03}$
& $1.51\mathrm{e}{-02} \pm 1.03\mathrm{e}{-03}$
& $4.12\mathrm{e}{-02} \pm 3.14\mathrm{e}{-03}$ \\
\cmidrule{2-8}

& \multirow{2}{*}{$0.5$}
& SF2M
& $3.77\mathrm{e}{-01} \pm 5.88\mathrm{e}{-03}$
& $\mathbf{5.23\mathrm{e}{-01} \pm 1.07\mathrm{e}{-02}}$
& $2.04\mathrm{e}{-01} \pm 1.16\mathrm{e}{-02}$
& $9.09\mathrm{e}{-03} \pm 1.93\mathrm{e}{-03}$
& $2.83\mathrm{e}{-02} \pm 4.68\mathrm{e}{-03}$ \\
&& SplineFlow
& $\mathbf{3.76\mathrm{e}{-01} \pm 3.30\mathrm{e}{-03}}$
& $5.30\mathrm{e}{-01} \pm 6.59\mathrm{e}{-03}$
& $\mathbf{1.94\mathrm{e}{-01} \pm 1.42\mathrm{e}{-03}}$
& $\mathbf{7.42\mathrm{e}{-03} \pm 3.31\mathrm{e}{-04}}$
& $\mathbf{2.39\mathrm{e}{-02} \pm 1.36\mathrm{e}{-03}}$ \\
\cmidrule{2-8}

& \multirow{2}{*}{$0.75$}
& SF2M
& $\mathbf{4.05\mathrm{e}{-01} \pm 6.73\mathrm{e}{-03}}$
& $\mathbf{6.46\mathrm{e}{-01} \pm 1.07\mathrm{e}{-02}}$
& $\mathbf{1.36\mathrm{e}{-01} \pm 1.20\mathrm{e}{-02}}$
& $\mathbf{1.47\mathrm{e}{-03} \pm 9.50\mathrm{e}{-04}}$
& $\mathbf{1.27\mathrm{e}{-02} \pm 3.13\mathrm{e}{-03}}$ \\
&& SplineFlow
& $4.18\mathrm{e}{-01} \pm 5.87\mathrm{e}{-03}$
& $6.76\mathrm{e}{-01} \pm 8.23\mathrm{e}{-03}$
& $1.42\mathrm{e}{-01} \pm 1.72\mathrm{e}{-02}$
& $2.05\mathrm{e}{-03} \pm 1.55\mathrm{e}{-03}$
& $1.51\mathrm{e}{-02} \pm 5.15\mathrm{e}{-03}$ \\

\midrule
\multirow{6}{*}{Damped.}
& \multirow{2}{*}{$0.25$}
& SF2M
& $8.35\mathrm{e}{-01} \pm 3.23\mathrm{e}{-02}$
& $9.50\mathrm{e}{-01} \pm 2.01\mathrm{e}{-02}$
& $3.71\mathrm{e}{-01} \pm 3.46\mathrm{e}{-02}$
& $\mathbf{3.29\mathrm{e}{-02} \pm 7.30\mathrm{e}{-03}}$
& $9.31\mathrm{e}{-02} \pm 1.77\mathrm{e}{-02}$ \\
&& SplineFlow
& $\mathbf{7.98\mathrm{e}{-01} \pm 1.57\mathrm{e}{-02}}$
& $\mathbf{9.13\mathrm{e}{-01} \pm 2.28\mathrm{e}{-02}}$
& $\mathbf{3.64\mathrm{e}{-01} \pm 2.11\mathrm{e}{-02}}$
& $3.42\mathrm{e}{-02} \pm 4.34\mathrm{e}{-03}$
& $\mathbf{9.16\mathrm{e}{-02} \pm 1.07\mathrm{e}{-02}}$ \\
\cmidrule{2-8}

& \multirow{2}{*}{$0.5$}
& SF2M
& $8.28\mathrm{e}{-01} \pm 3.29\mathrm{e}{-02}$
& $\mathbf{9.72\mathrm{e}{-01} \pm 2.54\mathrm{e}{-02}}$
& $\mathbf{3.19\mathrm{e}{-01} \pm 7.43\mathrm{e}{-03}}$
& $2.20\mathrm{e}{-02} \pm 2.68\mathrm{e}{-03}$
& $\mathbf{6.56\mathrm{e}{-02} \pm 5.04\mathrm{e}{-03}}$ \\
&& SplineFlow
& $\mathbf{8.05\mathrm{e}{-01} \pm 8.73\mathrm{e}{-03}}$
& $9.81\mathrm{e}{-01} \pm 3.99\mathrm{e}{-02}$
& $3.21\mathrm{e}{-01} \pm 2.84\mathrm{e}{-02}$
& $\mathbf{2.17\mathrm{e}{-02} \pm 5.59\mathrm{e}{-03}}$
& $6.87\mathrm{e}{-02} \pm 1.21\mathrm{e}{-02}$ \\
\cmidrule{2-8}

& \multirow{2}{*}{$0.75$}
& SF2M
& $\mathbf{9.12\mathrm{e}{-01} \pm 2.70\mathrm{e}{-02}}$
& $\mathbf{1.15\mathrm{e}{+00} \pm 3.93\mathrm{e}{-02}}$
& $\mathbf{2.75\mathrm{e}{-01} \pm 2.17\mathrm{e}{-02}}$
& $\mathbf{1.09\mathrm{e}{-02} \pm 2.14\mathrm{e}{-03}}$
& $\mathbf{4.65\mathrm{e}{-02} \pm 6.44\mathrm{e}{-03}}$ \\
&& SplineFlow
& $9.37\mathrm{e}{-01} \pm 2.89\mathrm{e}{-02}$
& $1.17\mathrm{e}{+00} \pm 1.17\mathrm{e}{-02}$
& $2.85\mathrm{e}{-01} \pm 1.10\mathrm{e}{-02}$
& $1.24\mathrm{e}{-02} \pm 7.65\mathrm{e}{-04}$
& $5.27\mathrm{e}{-02} \pm 4.89\mathrm{e}{-03}$ \\

\midrule
\multirow{6}{*}{LV}
& \multirow{2}{*}{$0.25$}
& SF2M
& $4.07\mathrm{e}{-01} \pm 6.92\mathrm{e}{-02}$
& $4.04\mathrm{e}{-01} \pm 7.33\mathrm{e}{-02}$
& $4.00\mathrm{e}{-01} \pm 2.56\mathrm{e}{-02}$
& $9.24\mathrm{e}{-02} \pm 8.00\mathrm{e}{-03}$
& $3.45\mathrm{e}{-01} \pm 3.23\mathrm{e}{-02}$ \\
&& SplineFlow
& $\mathbf{3.18\mathrm{e}{-01} \pm 3.58\mathrm{e}{-02}}$
& $\mathbf{3.21\mathrm{e}{-01} \pm 3.21\mathrm{e}{-02}}$
& $\mathbf{3.18\mathrm{e}{-01} \pm 1.21\mathrm{e}{-02}}$
& $\mathbf{4.44\mathrm{e}{-02} \pm 5.03\mathrm{e}{-03}}$
& $\mathbf{1.74\mathrm{e}{-01} \pm 2.60\mathrm{e}{-02}}$ \\
\cmidrule{2-8}

& \multirow{2}{*}{$0.5$}
& SF2M
& $3.50\mathrm{e}{-01} \pm 2.88\mathrm{e}{-02}$
& $3.54\mathrm{e}{-01} \pm 3.12\mathrm{e}{-02}$
& $3.74\mathrm{e}{-01} \pm 3.69\mathrm{e}{-02}$
& $8.03\mathrm{e}{-02} \pm 2.88\mathrm{e}{-02}$
& $2.93\mathrm{e}{-01} \pm 9.23\mathrm{e}{-02}$ \\
&& SplineFlow
& $\mathbf{2.87\mathrm{e}{-01} \pm 2.32\mathrm{e}{-02}}$
& $\mathbf{2.85\mathrm{e}{-01} \pm 2.54\mathrm{e}{-02}}$
& $\mathbf{3.10\mathrm{e}{-01} \pm 1.16\mathrm{e}{-02}}$
& $\mathbf{4.41\mathrm{e}{-02} \pm 3.80\mathrm{e}{-03}}$
& $\mathbf{1.63\mathrm{e}{-01} \pm 2.32\mathrm{e}{-02}}$ \\
\cmidrule{2-8}

& \multirow{2}{*}{$0.75$}
& SF2M
& $5.24\mathrm{e}{-01} \pm 1.24\mathrm{e}{-01}$
& $5.26\mathrm{e}{-01} \pm 1.21\mathrm{e}{-01}$
& $4.46\mathrm{e}{-01} \pm 9.53\mathrm{e}{-02}$
& $1.06\mathrm{e}{-01} \pm 3.98\mathrm{e}{-02}$
& $4.15\mathrm{e}{-01} \pm 1.58\mathrm{e}{-01}$ \\
&& SplineFlow
& $\mathbf{3.77\mathrm{e}{-01} \pm 5.73\mathrm{e}{-02}}$
& $\mathbf{3.78\mathrm{e}{-01} \pm 5.48\mathrm{e}{-02}}$
& $\mathbf{3.60\mathrm{e}{-01} \pm 4.04\mathrm{e}{-02}}$
& $\mathbf{5.57\mathrm{e}{-02} \pm 9.43\mathrm{e}{-03}}$
& $\mathbf{2.22\mathrm{e}{-01} \pm 4.77\mathrm{e}{-02}}$ \\

\bottomrule
\end{tabular}
}
\end{table}

\begin{table}[H]
\caption{Performance of models on holdout interpolation and extrapolation settings for the Brain Regeneration dataset.}
\vspace{0.01in}
\label{tab:apdx_stereoseq_condensed}
\centering
% \small
\setlength{\tabcolsep}{5pt}
\resizebox{0.95\textwidth}{!}{
\begin{tabular}{lll ccc}
\toprule
\textbf{Embedding} & \textbf{Task} & \textbf{Model}
& \textbf{Wasserstein} & \textbf{MMD} & \textbf{Energy} \\
\midrule

\multirow{6}{*}{PHATE} 
&
\multirow{3}{*}{Interpolation}
& MOTFM
& $7.94e{-1}\!\pm\!2.64e{-2}$
& $5.03e{-1}\!\pm\!1.60e{-2}$
& $1.44e{+0}\!\pm\!6.21e{-2}$ \\
& & SF2M
& $7.53e{-1}\!\pm\!2.58e{-2}$
& $4.57e{-1}\!\pm\!8.30e{-3}$
& $1.31e{+0}\!\pm\!3.55e{-2}$ \\
& & SplineFlow
& $\mathbf{6.07e{-1}\!\pm\!2.71e{-2}}$
& $\mathbf{2.85e{-1}\!\pm\!6.60e{-3}}$
& $\mathbf{8.15e{-1}\!\pm\!2.14e{-2}}$ \\

\cmidrule{2-6}
& \multirow{3}{*}{Extrapolation}
& MOTFM
& $1.57e{+0}\!\pm\!4.50e{-3}$
& $1.01e{+0}\!\pm\!1.15e{-2}$
& $3.89e{+0}\!\pm\!5.07e{-2}$ \\
& & SF2M
& $\mathbf{1.53e{+0}\!\pm\!4.72e{-2}}$
& $9.72e{-1}\!\pm\!9.20e{-3}$
& $3.77e{+0}\!\pm\!7.24e{-2}$ \\
& & SplineFlow
& $1.59e{+0}\!\pm\!1.92e{-2}$
& $\mathbf{9.55e{-1}\!\pm\!8.70e{-3}}$
& $\mathbf{3.74e{+0}\!\pm\!4.24e{-2}}$ \\

\midrule
\multirow{6}{*}{PCA} & 
\multirow{3}{*}{Interpolation}
& MOTFM
& $9.01e{-1}\!\pm\!2.16e{-2}$
& $7.12e{-2}\!\pm\!4.40e{-3}$
& $1.16e{+0}\!\pm\!5.24e{-2}$ \\
& & SF2M
& $8.88e{-1}\!\pm\!7.40e{-3}$
& $6.68e{-2}\!\pm\!7.00e{-4}$
& $1.14e{+0}\!\pm\!1.24e{-2}$ \\
& & SplineFlow
& $\mathbf{8.38e{-1}\!\pm\!1.37e{-2}}$
& $\mathbf{5.82e{-2}\!\pm\!2.50e{-3}}$
& $\mathbf{1.02e{+0}\!\pm\!3.93e{-2}}$ \\

\cmidrule{2-6}
& \multirow{3}{*}{Extrapolation}
& MOTFM
& $\mathbf{7.99e{-1}\!\pm\!3.10e{-3}}$
& $1.35e{-1}\!\pm\!6.90e{-3}$
& $1.65e{+0}\!\pm\!4.28e{-2}$ \\
& & SF2M
& $8.65e{-1}\!\pm\!1.54e{-2}$
& $1.77e{-1}\!\pm\!2.50e{-3}$
& $1.92e{+0}\!\pm\!1.27e{-2}$ \\
& & SplineFlow
& $8.04e{-1}\!\pm\!3.72e{-2}$
& $\mathbf{1.29e{-1}\!\pm\!6.80e{-3}}$
& $\mathbf{1.65e{+0}\!\pm\!5.66e{-2}}$ \\

\bottomrule
\end{tabular}
}
\end{table}

\begin{table}[H]
\caption{Performance of models across holdout interpolation and extrapolation settings for the Embryoid Evolution dataset.}
\vspace{0.01in}
\label{tab:apdx_citeseq_condensed}
\centering
% \small
\setlength{\tabcolsep}{5pt}
\renewcommand{\arraystretch}{1.05}
\resizebox{0.95\textwidth}{!}{
\begin{tabular}{lll ccc}
\toprule
\textbf{Embedding} & \textbf{Task} & \textbf{Model}
& \textbf{Wasserstein} & \textbf{MMD} & \textbf{Energy} \\
\midrule

\multirow{6}{*}{PHATE} &
\multirow{3}{*}{Interpolation}
& MOTFM
& $3.78e{-1}\!\pm\!1.11e{-2}$
& $6.01e{-2}\!\pm\!2.52e{-3}$
& $1.75e{-1}\!\pm\!4.98e{-3}$ \\
& & SF2M
& $2.72e{-1}\!\pm\!2.36e{-2}$
& $2.46e{-2}\!\pm\!4.03e{-3}$
& $8.63e{-2}\!\pm\!1.33e{-2}$ \\
& & SplineFlow
& $\mathbf{2.07e{-1}\!\pm\!1.69e{-2}}$
& $\mathbf{1.26e{-2}\!\pm\!5.87e{-4}}$
& $\mathbf{3.96e{-2}\!\pm\!1.40e{-3}}$ \\

\cmidrule{2-6}
& \multirow{3}{*}{Extrapolation}
& MOTFM
& $4.48e{-1}\!\pm\!5.05e{-2}$
& $8.25e{-2}\!\pm\!9.05e{-3}$
& $2.39e{-1}\!\pm\!7.18e{-3}$ \\
& & SF2M
& $\mathbf{4.39e{-1}\!\pm\!4.40e{-2}}$
& $\mathbf{5.79e{-2}\!\pm\!1.08e{-2}}$
& $\mathbf{1.94e{-1}\!\pm\!1.57e{-2}}$ \\
& & SplineFlow
& $4.48e{-1}\!\pm\!4.41e{-2}$
& $6.90e{-2}\!\pm\!1.70e{-3}$
& $2.07e{-1}\!\pm\!2.96e{-2}$ \\

\midrule
\multirow{6}{*}{PCA} &
\multirow{3}{*}{Interpolation}
& MOTFM
& $2.86e{-1}\!\pm\!1.37e{-2}$
& $1.14e{-2}\!\pm\!9.33e{-4}$
& $1.65e{-1}\!\pm\!1.23e{-2}$ \\
& & SF2M
& $\mathbf{2.74e{-1}\!\pm\!1.05e{-2}}$
& $\mathbf{9.12e{-3}\!\pm\!8.42e{-4}}$
& $\mathbf{1.55e{-1}\!\pm\!1.25e{-2}}$ \\
& & SplineFlow
& $2.96e{-1}\!\pm\!1.46e{-2}$
& $1.27e{-2}\!\pm\!3.11e{-4}$
& $1.99e{-1}\!\pm\!9.34e{-3}$ \\

\cmidrule{2-6}
& \multirow{3}{*}{Extrapolation}
& MOTFM
& $7.89e{-1}\!\pm\!1.29e{-1}$
& $3.75e{-2}\!\pm\!2.11e{-3}$
& $7.21e{-1}\!\pm\!5.09e{-2}$ \\
& & SF2M
& $5.92e{-1}\!\pm\!2.02e{-2}$
& $3.77e{-2}\!\pm\!3.98e{-4}$
& $7.47e{-1}\!\pm\!1.18e{-2}$ \\
& & SplineFlow
& $\mathbf{5.39e{-1}\!\pm\!3.35e{-2}}$
& $\mathbf{3.58e{-2}\!\pm\!4.42e{-3}}$
& $\mathbf{6.89e{-1}\!\pm\!1.01e{-1}}$ \\

\bottomrule
\end{tabular}
}
\end{table}

\begin{figure}[H]
\centering
\includegraphics[width=\linewidth]{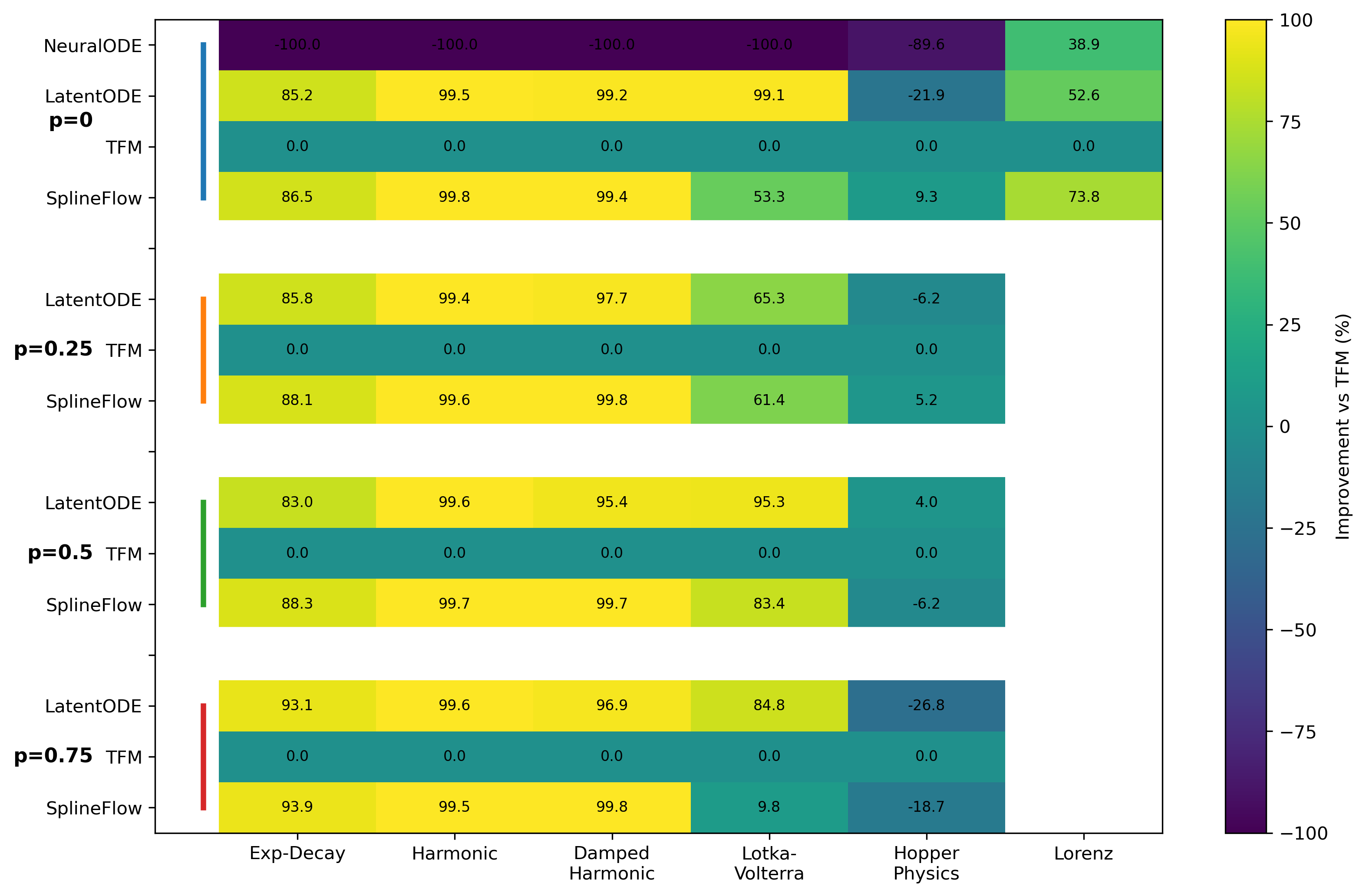}
\vspace{-0.1in}
\caption{Heatmap visualization of relative improvements of LatentODE and SplineFlow in MSE with reference to TFM, summarized from the ODE Table~\ref{tab: apdx_ode_combined_all_datasets}. 
% ($\dfrac{\text{TFM}- \text{model}}{\text{TFM}}\times100$)
}
\vspace{-0.1in}
\label{fig: heatmap_ode}
\end{figure}

\section{Visualizations of Predicted Dynamics}
\label{apdx: viz_dynamics}

In this section, we visualize the predicted trajectories alongside the ground truth in different spaces to further examine the learning performance of different methods (apart from quantitative metric performance). For irregularly sampled settings, we choose $p=0.75$ to reflect the most challenging scenarios. The trajectories predicted by our SplineFlow method can recover the ground-truth dynamics well, particularly outperforming other methods for the chaotic Lorenz system. 

\subsection{ODE Dynamics}

\subsubsection{Harmonic Oscillator ($p=0.75$)}

\begin{figure}[H]
    \centering
    \includegraphics[width=0.95\linewidth]{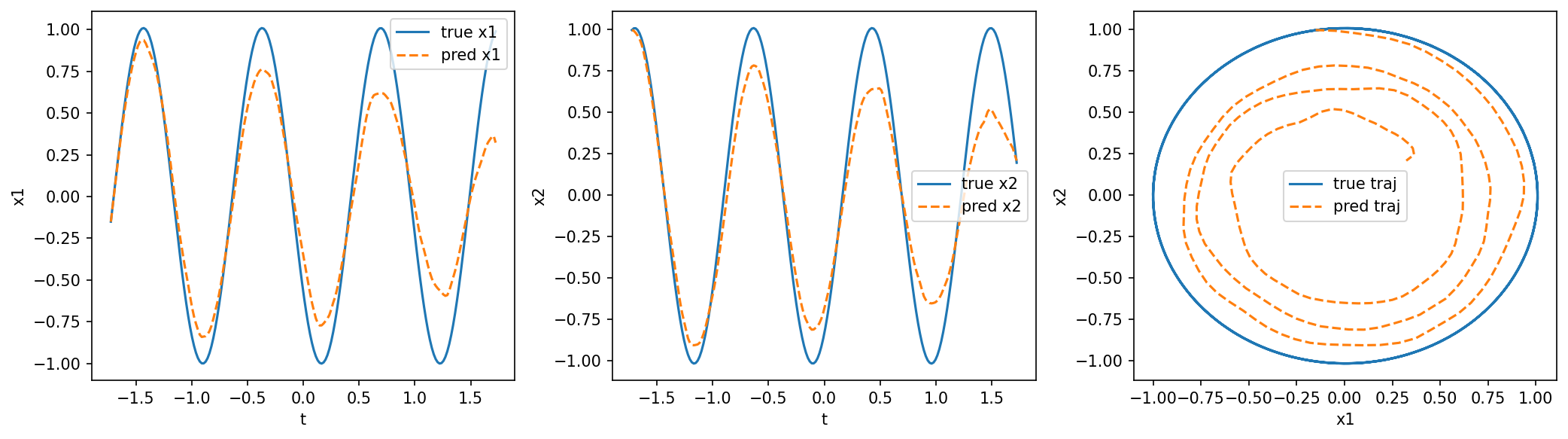}
    \vspace{-0.05in}
    \caption{TFM evaluated path for Harmonic Oscillator.}
    \vspace{-0.1in}
    \label{fig:harmonic_ode_tfm_p75}
\end{figure}

\begin{figure}[H]
    \centering
    \includegraphics[width=0.95\linewidth]{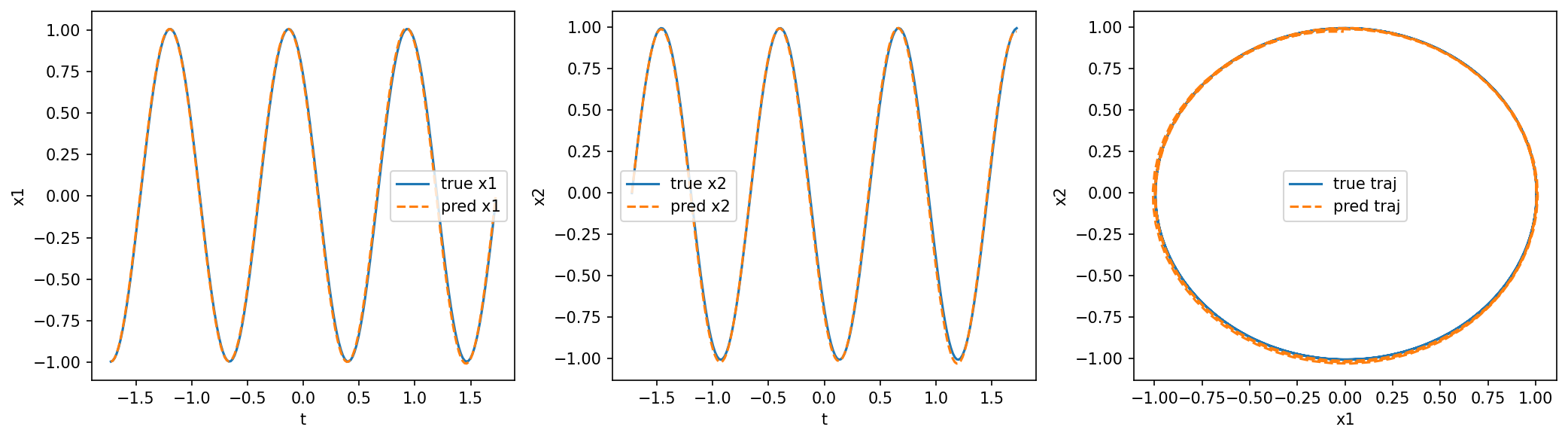}
    \vspace{-0.05in}
    \caption{SplineFlow evaluated path for Harmonic Oscillator.}
    \vspace{-0.1in}
    \label{fig:harmonic_ode_splineflow_p75}
\end{figure}

\begin{figure}[H]
    \centering
    \includegraphics[width=0.8\linewidth]{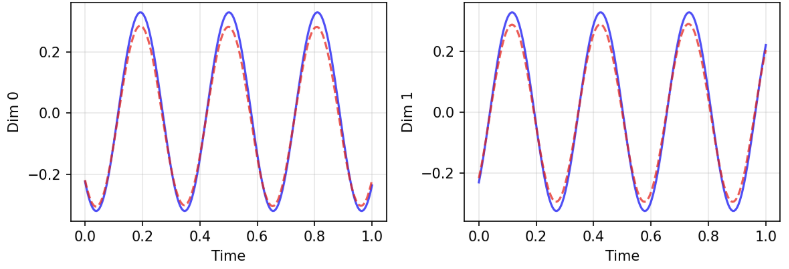}
    \vspace{-0.05in}
    \caption{LatentODE evaluated path for Harmonic Oscillator.}
    \vspace{-0.1in}
    \label{fig:harmonic_ode_latentode_p75}
\end{figure}

\subsubsection{Damped Harmonic Oscillator ($p=0.75$)}

\begin{figure}[H]
    \centering
    \includegraphics[width=0.95\linewidth]{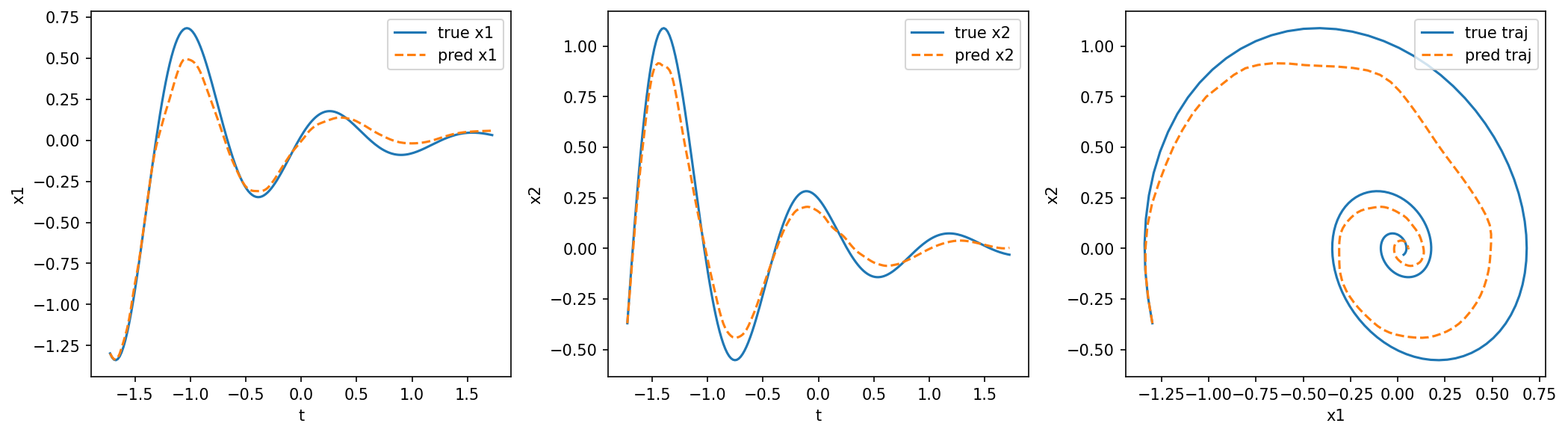}
    \vspace{-0.05in}
    \caption{TFM evaluated path for Damped Harmonic Oscillator.}
    \vspace{-0.1in}
    \label{fig:damped_harmonic_ode_tfm_p75}
\end{figure}

\begin{figure}[H]
    \centering
    \includegraphics[width=0.95\linewidth]{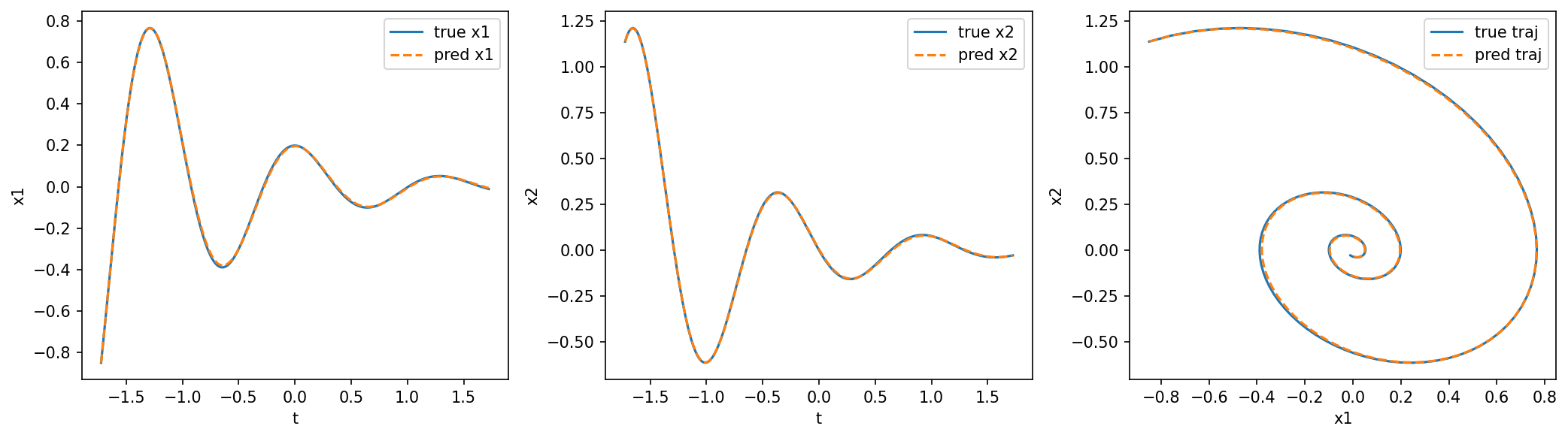}
    \vspace{-0.05in}
    \caption{SplineFlow evaluated path for Damped Harmonic Oscillator.}
    \vspace{-0.1in}
    \label{fig:damped_harmonic_ode_splineflow_p75}
\end{figure}

\begin{figure}[H]
    \centering
    \includegraphics[width=0.85\linewidth]{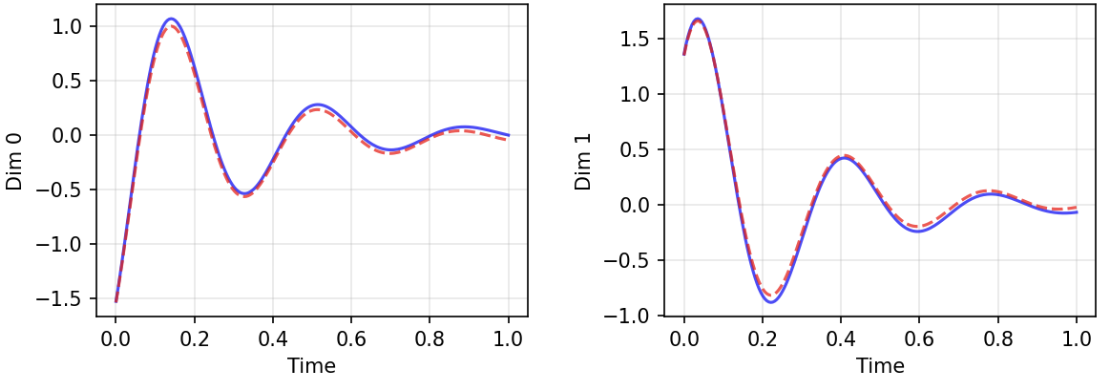}
    \vspace{-0.05in}
    \caption{LatentODE evaluated path for Damped Harmonic Oscillator.}
    \vspace{-0.1in}
    \label{fig:harmonic_ode_latentode_p75}
\end{figure}

\subsubsection{Lotka--Volterra ($p=0.75$)}

\begin{figure}[H]
    \centering
    \includegraphics[width=0.95\linewidth]{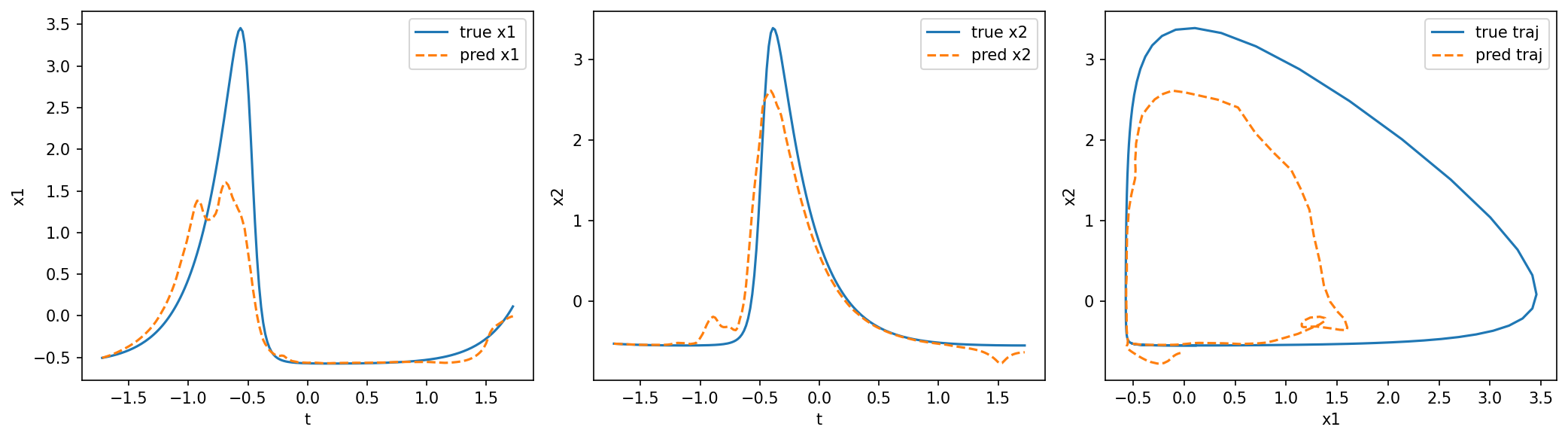}
    \vspace{-0.05in}
    \caption{TFM evaluated path for Lotka--Volterra System.}
    \vspace{-0.1in}
    \label{fig:lv_ode_tfm_p75}
\end{figure}

\begin{figure}[H]
    \centering
    \includegraphics[width=0.95\linewidth]{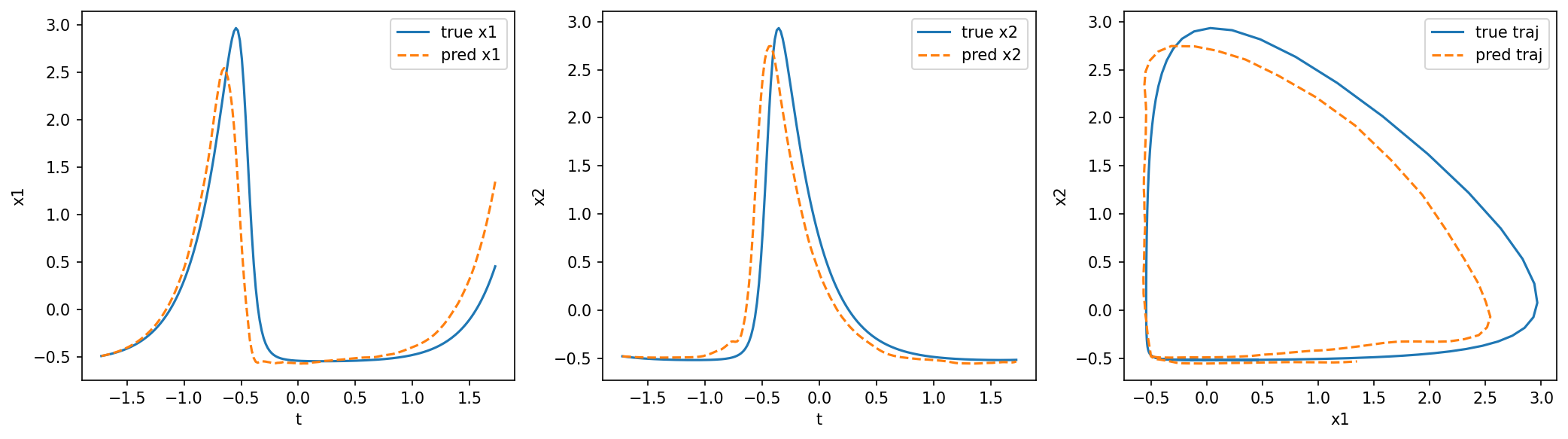}
    \vspace{-0.05in}
    \caption{SplineFlow evaluated path for Lotka--Volterra System.}
    \vspace{-0.1in}
    \label{fig:lv_ode_splineflow_p75}
\end{figure}

\subsection{Chaotic Systems (Lorenz)}

\subsubsection{Adjoint-Methods ODE Dynamics}
\begin{figure}[H]
    \centering
    \includegraphics[width=0.95\linewidth]{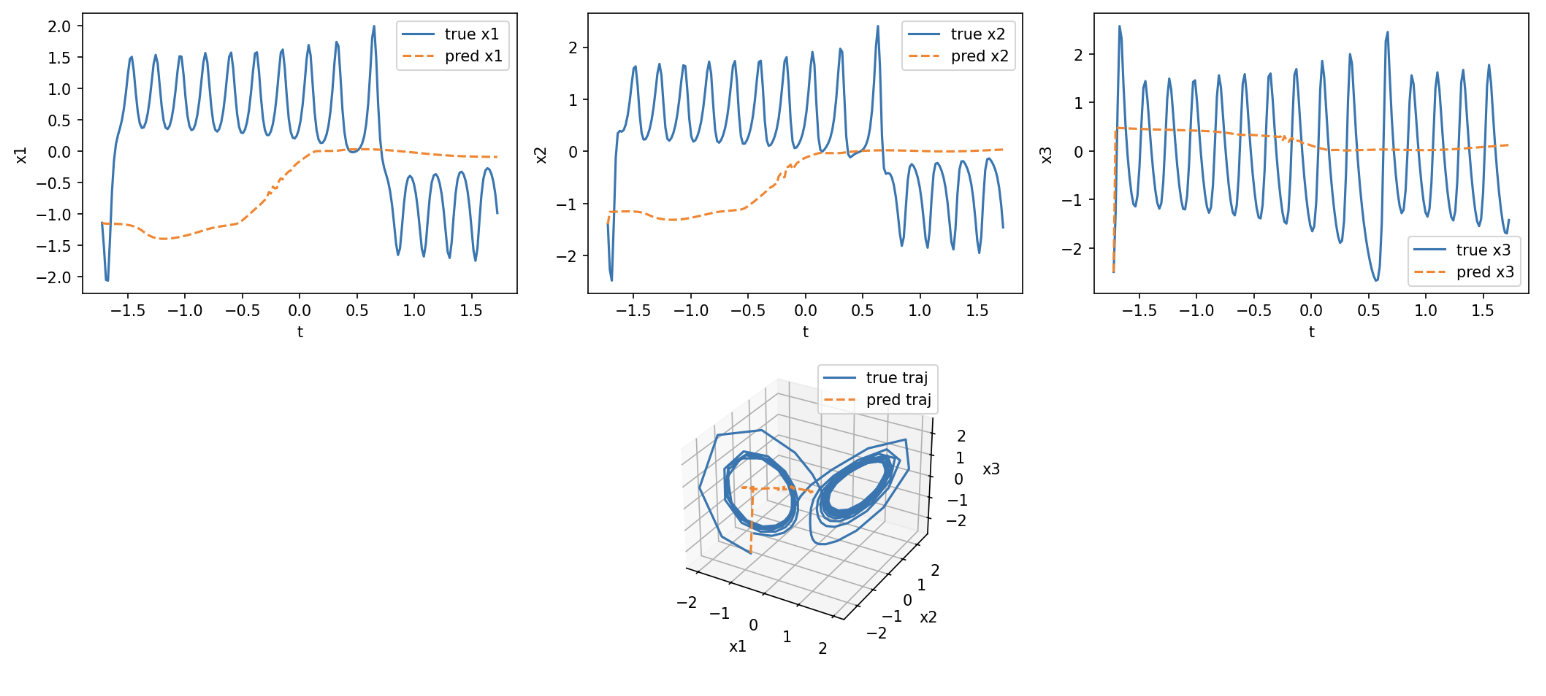}
    \vspace{-0.05in}
    \caption{NeuralODE evaluated path for Lorenz System.}
    \vspace{-0.1in}
    \label{fig:lorenz_ode_neuralode_p0}
\end{figure}

\begin{figure}[H]
    \centering
    \includegraphics[width=0.95\linewidth]{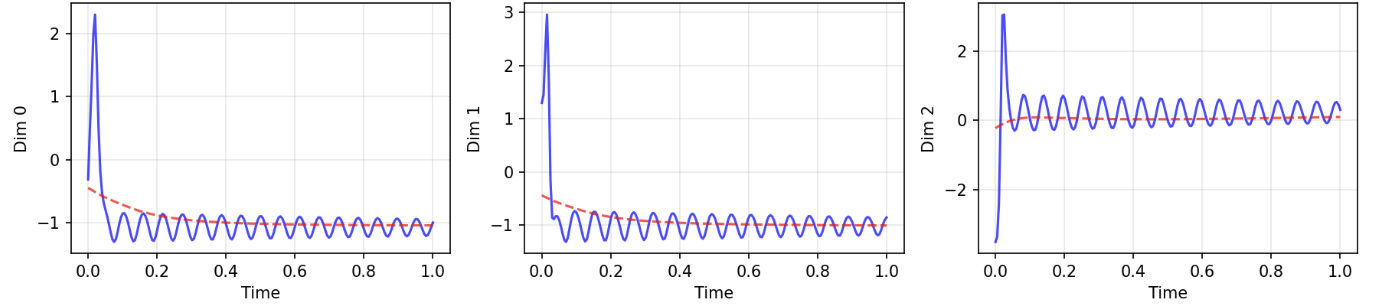}
    \vspace{-0.05in}
    \caption{LatentODE evaluated path for Lorenz System.}
    \vspace{-0.1in}
    \label{fig:lorenz_ode_latentode_p0}
\end{figure}

\subsubsection{Trajectory Flow Matching ODE Dynamics}

\begin{figure}[H]
    \centering
    \includegraphics[width=0.95\linewidth]{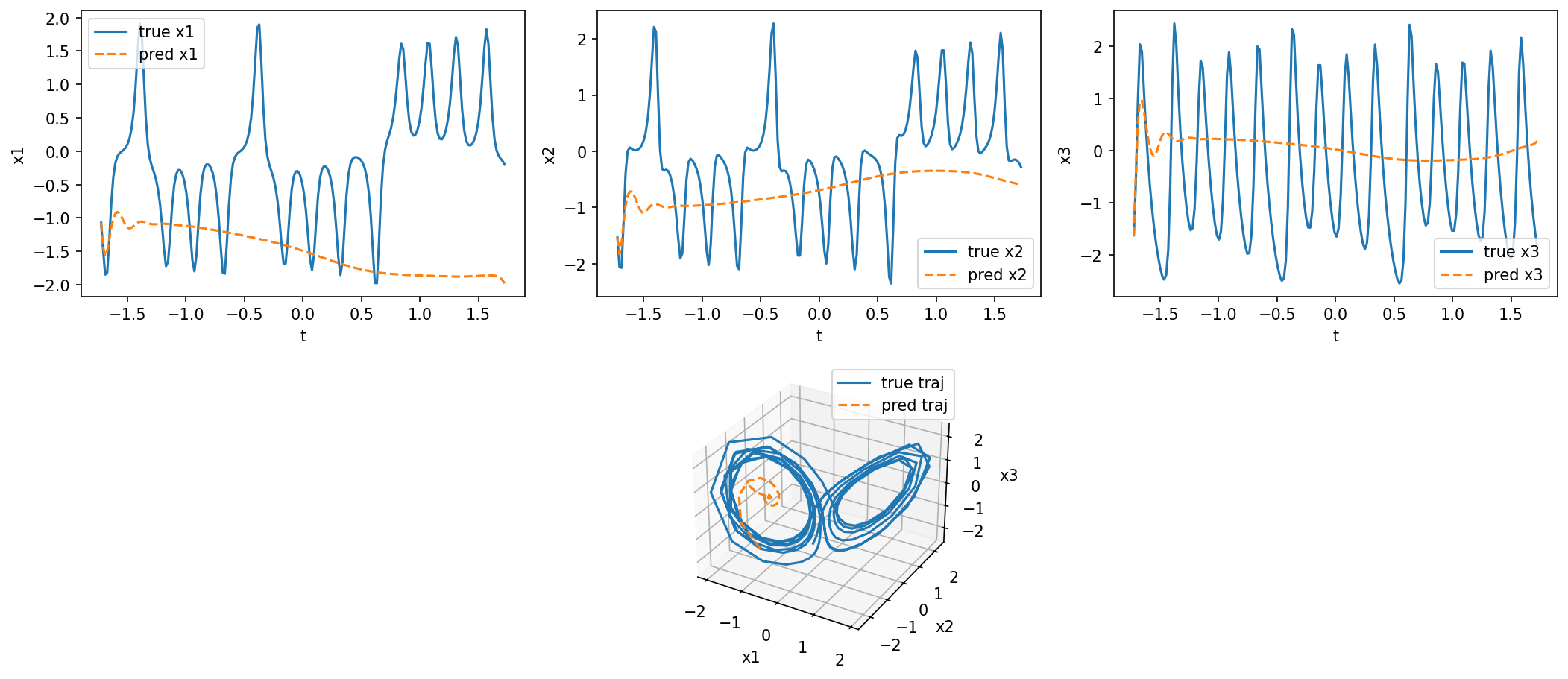}
    \vspace{-0.05in}
    \caption{TFM evaluated path for Lorenz System.}
    \vspace{-0.1in}
    \label{fig:lorenz_ode_tfm_img2}
\end{figure}

\begin{figure}[H]
    \centering
    \includegraphics[width=0.95\linewidth]{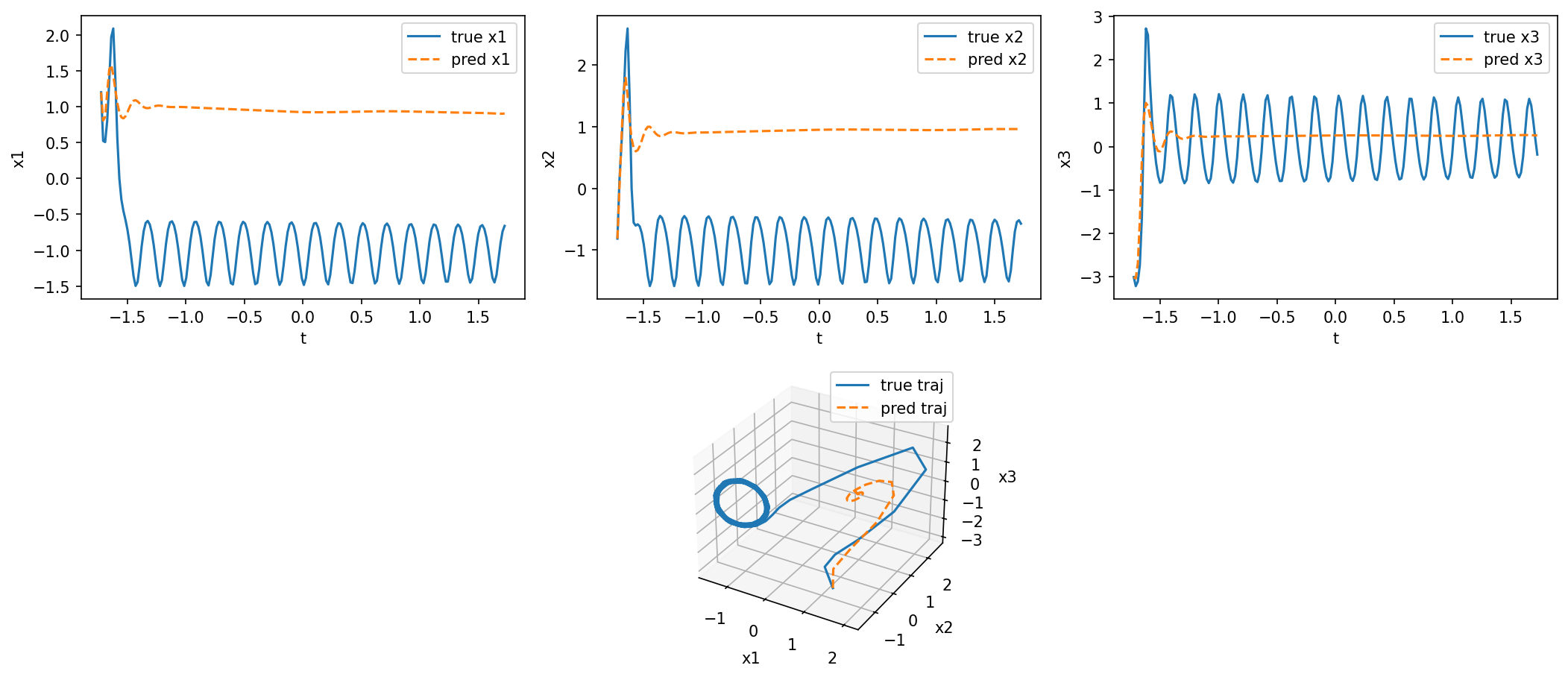}
    \vspace{-0.05in}
    \caption{TFM evaluated path for Lorenz System.}
    \vspace{-0.1in}
    \label{fig:lorenz_ode_tfm_img3}
\end{figure}

\subsubsection{SplineFlow ODE Dynamics}

\begin{figure}[H]
    \centering
    \includegraphics[width=0.95\linewidth]{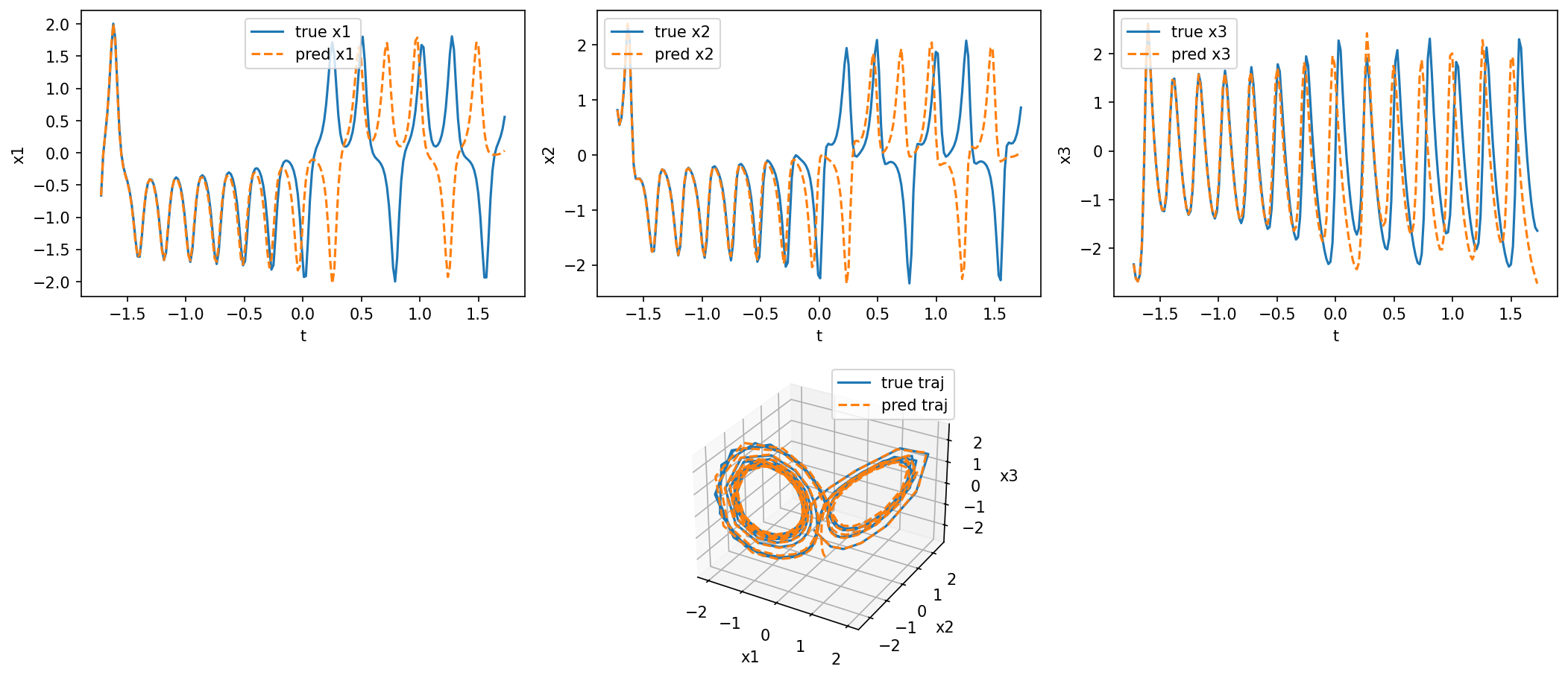}
    \vspace{-0.05in}
    \caption{SplineFlow evaluated path for Lorenz System.}
    \vspace{-0.1in}
    \label{fig:lorenz_ode_splineflow_img2}
\end{figure}

\begin{figure}[H]
    \centering
    \includegraphics[width=0.95\linewidth]{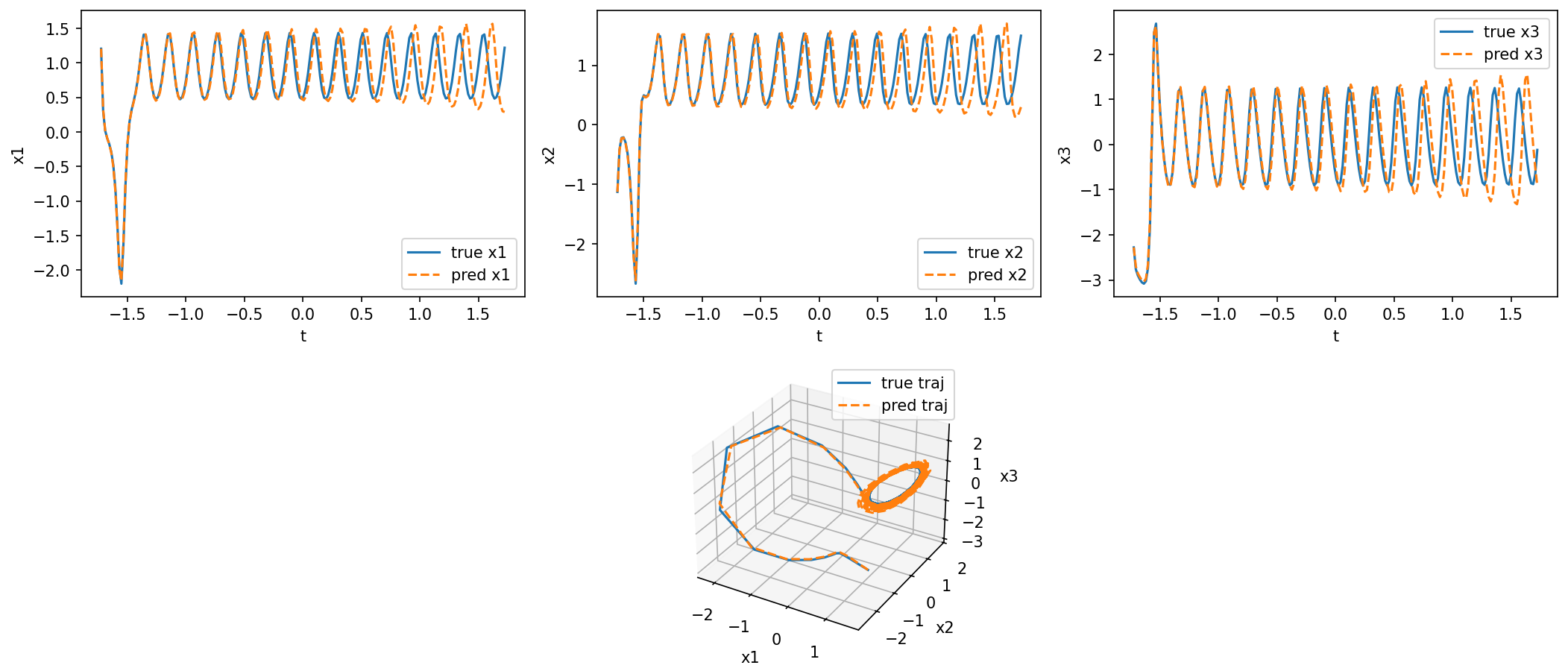}
    \vspace{-0.05in}
    \caption{SplineFlow evaluated path for Lorenz System.}
    \vspace{-0.1in}
    \label{fig:lorenz_ode_splineflow_img3}
\end{figure}

\subsubsection{$\text{SF2M}$ SDE Dynamics}

\begin{figure}[H]
    \centering
    \includegraphics[width=0.95\linewidth]{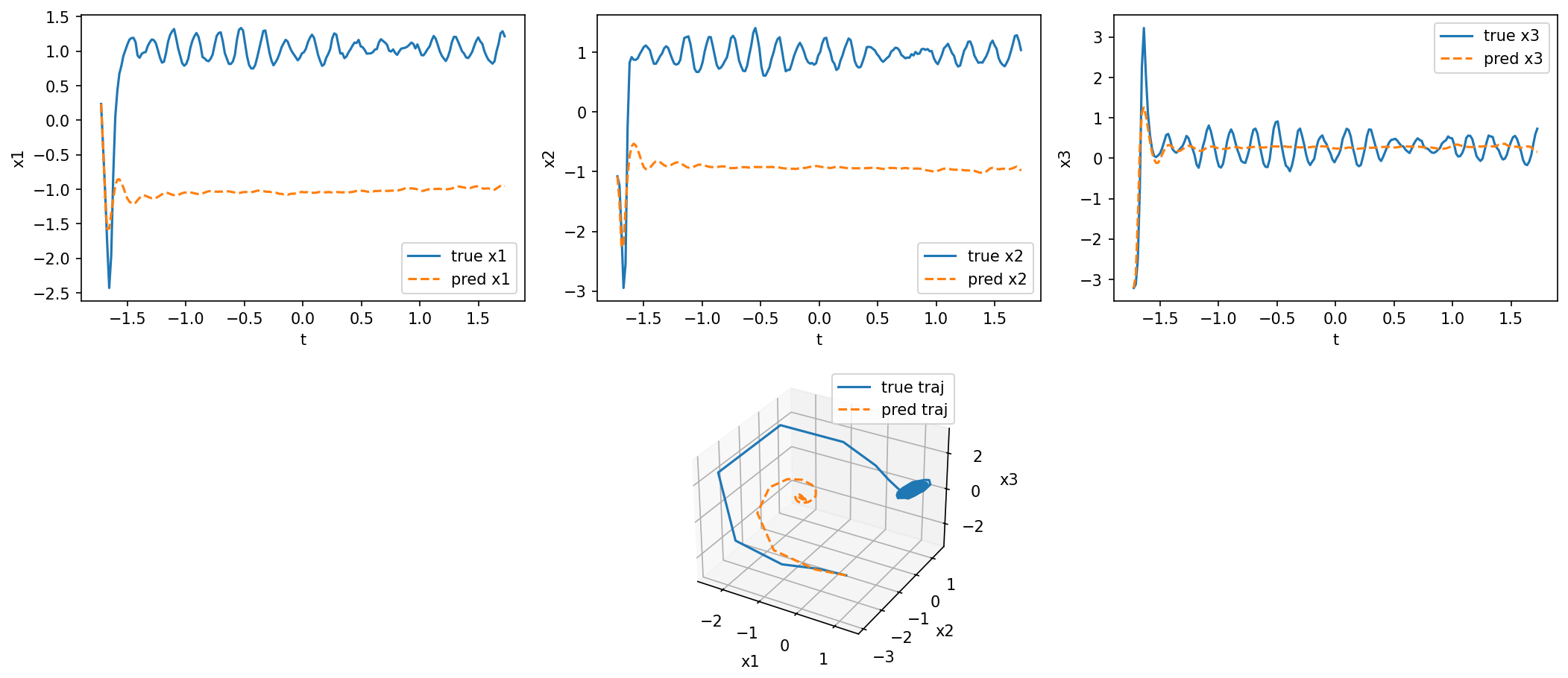}
    \vspace{-0.05in}
    \caption{$\text{SF2M}$ evaluated path for Lorenz System.}
    \vspace{-0.1in}
    \label{fig:lorenz_sde_sf2m_img1}
\end{figure}

\begin{figure}[H]
    \centering
    \includegraphics[width=0.95\linewidth]{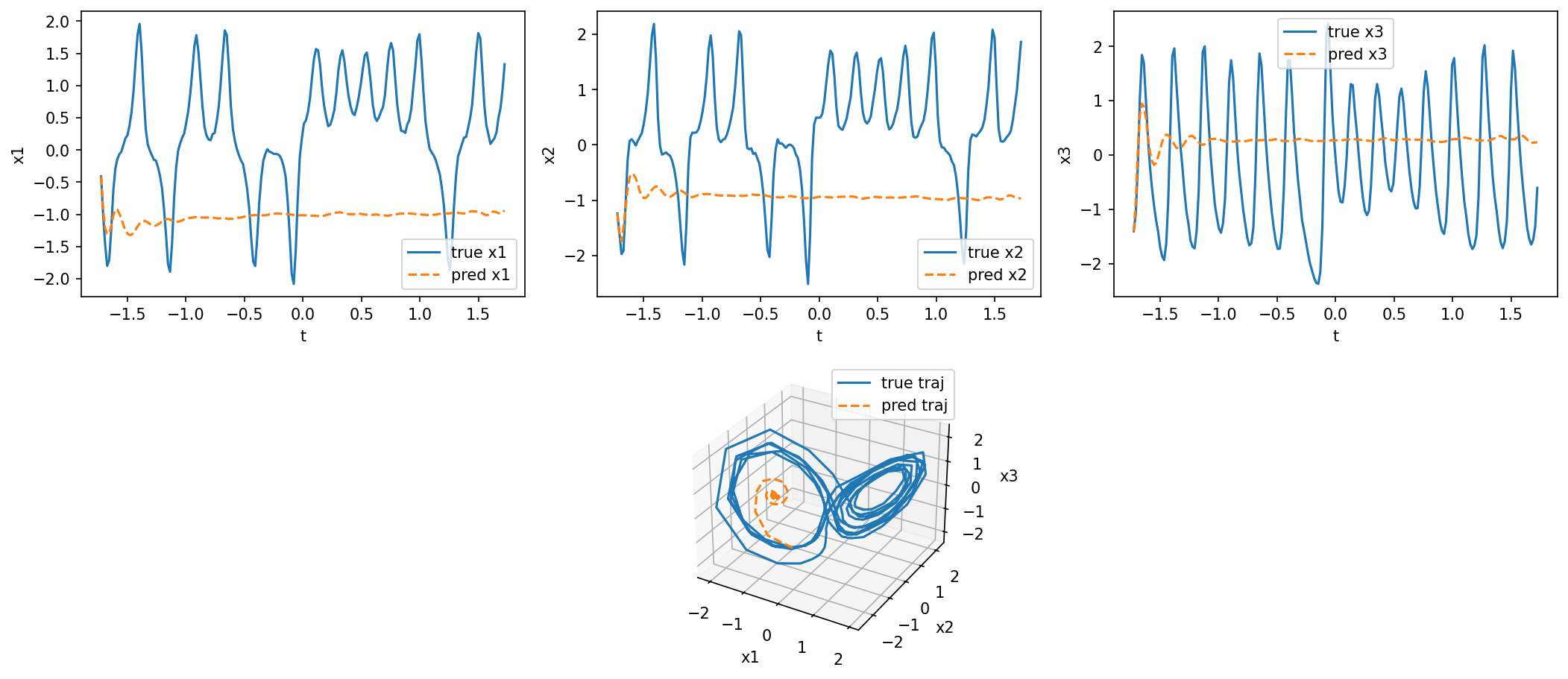}
    \vspace{-0.05in}
    \caption{$\text{SF2M}$ evaluated path for Lorenz System.}
    \vspace{-0.1in}
    \label{fig:lorenz_sde_sf2m_img2}
\end{figure}

\subsubsection{SplineFlow SDE Dynamics}

\begin{figure}[H]
    \centering
    \includegraphics[width=0.95\linewidth]{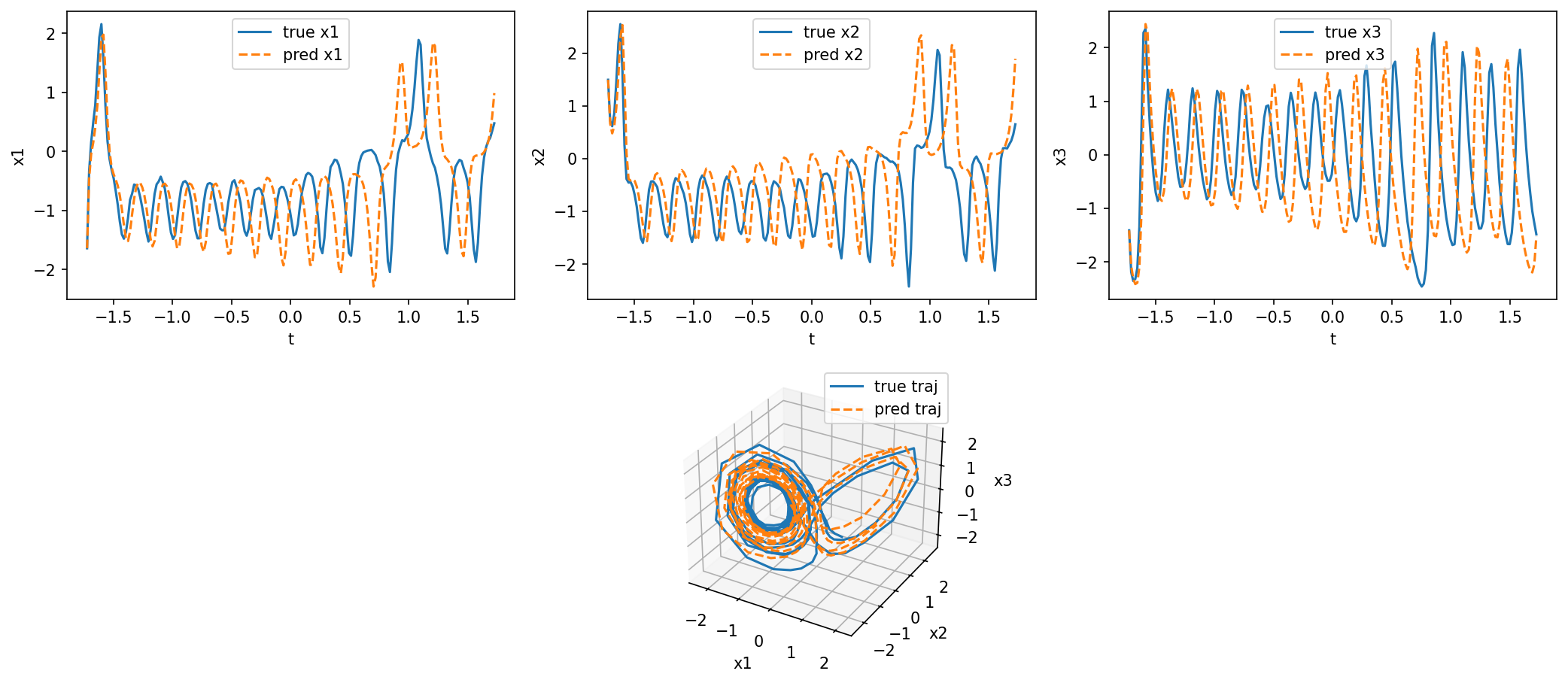}
    \vspace{-0.05in}
    \caption{SplineFlow evaluated path for Lorenz System.}
    \vspace{-0.1in}
    \label{fig:lorenz_sde_splineflow_img1}
\end{figure}

\begin{figure}[H]
    \centering
    \includegraphics[width=0.95\linewidth]{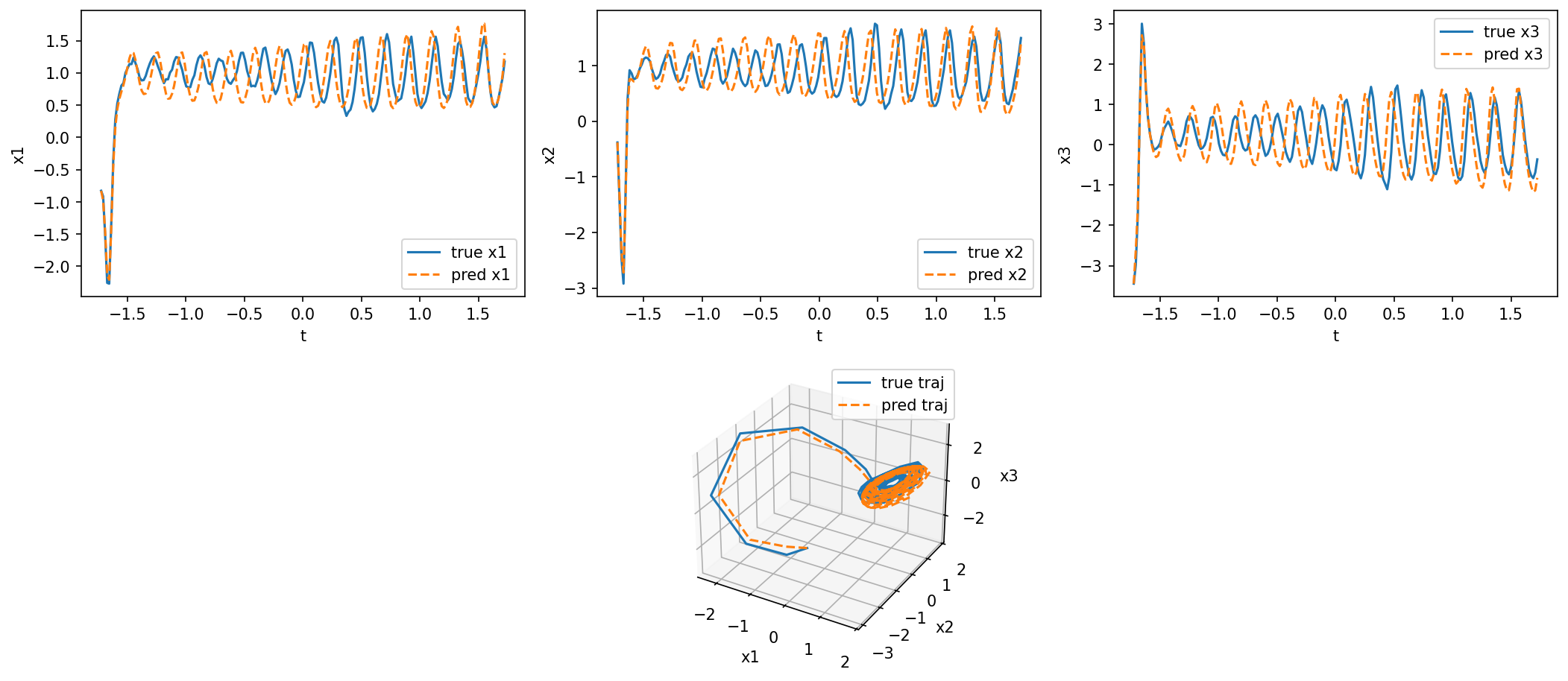}
    \vspace{-0.05in}
    \caption{SplineFlow evaluated path for Lorenz System.}
    \vspace{-0.1in}
    \label{fig:lorenz_sde_splineflow_img2}
\end{figure}

% \subsection{Cellular State Evolution}

% \subsubsection{Brain Regeneration}

% \begin{figure}[H]
%   \centering
%     \includegraphics[width=\linewidth]{figures/cellular_trajectories/temporal_progression_stereoseq.png} % replace with your file
%   \caption{Temporal progression of spatial distribution of different cell types for Brain Regeneration.}
%   \label{fig:temporal_progression_stereoseq}
% \end{figure}

\section{Full Ablations on Irregularity and Spline Degree}
\label{apdx:ablations_all}

\subsection{ODE Dynamics (ablations with order $k$) for non-linear and oscillatory systems}

From the ablation studies, we observe that higher degree B-splines are the most useful when modeling at high irregularity ($p>0$) regimes, for example, Damped Harmonic and Lotka--Volterra. For systems with linear dynamics, such as HopperPhysics, linear degrees perform best. 

\subsubsection{Damped Harmonic Oscillator}

\begin{table}[H]
\caption{Effect of spline degree on MSE performance of SplineFlow on damped harmonic system across missingness probabilities $p$.}
%\vspace{-0.05in}
\centering
\small
\setlength{\tabcolsep}{6pt}
\renewcommand{\arraystretch}{1.05}
\begin{tabular}{lcccc}
\toprule
\textbf{Degree} & $\bm{p=0}$ & $\bm{p=0.25}$ & $\bm{p=0.5}$ & $\bm{p=0.75}$ \\
\midrule
1 & $1.97e{-2}\!\pm\!4.83e{-4}$ & $2.05e{-2}\!\pm\!2.69e{-3}$ & $3.19e{-2}\!\pm\!4.18e{-3}$ & $1.71e{-1}\!\pm\!3.71e{-2}$ \\
2 & $6.5e{-5}\!\pm\!3.47e{-5}$ & $6.2e{-5}\!\pm\!2.69e{-5}$ & $5.4e{-5}\!\pm\!3.44e{-5}$ & $4.15e{-4}\!\pm\!8.13e{-5}$ \\
3 & $3.1e{-5}\!\pm\!8.47e{-6}$ & $3.4e{-5}\!\pm\!9.69e{-6}$ & $1.00e{-4}\!\pm\!8.68e{-5}$ & $1.84e{-4}\!\pm\!4.29e{-5}$ \\
4 & $7.2e{-5}\!\pm\!3.02e{-5}$ & $5.1e{-5}\!\pm\!3.44e{-5}$ & $4.7e{-5}\!\pm\!3.37e{-5}$ & $3.4e{-5}\!\pm\!4.54e{-6}$ \\
5 & $9.5e{-5}\!\pm\!1.02e{-4}$ & $3.0e{-5}\!\pm\!4.77e{-6}$ & $4.7e{-5}\!\pm\!2.14e{-5}$ & $4.8e{-5}\!\pm\!2.30e{-5}$ \\
\bottomrule
\end{tabular}
\label{tab:order_vs_p}
\end{table}

\subsubsection{Lotka--Volterra}

\begin{table}[H]
\caption{Effect of spline degree on MSE performance of SplineFlow on Lotka--Volterra system across missingness probabilities $p$.}
%\vspace{-0.05in}
\centering
\small
\setlength{\tabcolsep}{6pt}
\renewcommand{\arraystretch}{1.05}
\begin{tabular}{lcccc}
\toprule
\textbf{Degree} & $\bm{p=0}$ & $\bm{p=0.25}$ & $\bm{p=0.5}$ & $\bm{p=0.75}$ \\
\midrule
1 & $4.19e{+0}\!\pm\!7.51e{+0}$ & $4.29e{+0}\!\pm\!6.00e{+0}$ & $1.22e{+1}\!\pm\!8.84e{+0}$ & $2.15e{+2}\!\pm\!1.28e{+2}$ \\
2 & $5.83e{-2}\!\pm\!4.14e{-2}$ & $2.56e{-1}\!\pm\!2.05e{-1}$ & $3.56e{-1}\!\pm\!2.24e{-1}$ & $6.82e{+0}\!\pm\!3.98e{+0}$ \\
3 & $4.50e{-2}\!\pm\!2.78e{-2}$ & $1.87e{-1}\!\pm\!1.33e{-1}$ & $2.75e{-1}\!\pm\!1.89e{-1}$ & $8.88e{+0}\!\pm\!8.82e{+0}$ \\
4 & $5.53e{-2}\!\pm\!2.36e{-2}$ & $1.83e{-1}\!\pm\!1.41e{-1}$ & $4.20e{+0}\!\pm\!7.96e{+0}$ & $4.31e{+0}\!\pm\!4.40e{+0}$ \\
5 & $8.55e{-2}\!\pm\!4.28e{-2}$ & $1.00e{-1}\!\pm\!5.72e{-2}$ & $7.05e{-1}\!\pm\!1.21e{+0}$ & $1.78e{+0}\!\pm\!1.43e{+0}$ \\
\bottomrule
\end{tabular}
\label{tab:lv_order_vs_p}
\end{table}

\subsubsection{HopperPhysics}

\begin{table}[H]
\caption{Effect of spline degree on MSE performance of SplineFlow on hopperphysics simulations across missingness probabilities $p$.}
%\vspace{-0.05in}
\centering
\small
\setlength{\tabcolsep}{6pt}
\renewcommand{\arraystretch}{1.05}
\begin{tabular}{lcccc}
\toprule
\textbf{Degree} & $\bm{p=0}$ & $\bm{p=0.25}$ & $\bm{p=0.5}$ & $\bm{p=0.75}$ \\
\midrule
1 & $1.41e{+0}\!\pm\!1.72e{-1}$ & $1.43e{+0}\!\pm\!2.98e{-2}$ & $1.83e{+0}\!\pm\!6.19e{-2}$ & $3.34e{+0}\!\pm\!4.30e{-1}$ \\
2 & $6.60e{+0}\!\pm\!2.35e{+0}$ & $5.39e{+1}\!\pm\!2.47e{+1}$ & $2.80e{+2}\!\pm\!5.93e{+1}$ & $1.00e{+3}\!\pm\!6.07e{+1}$ \\
3 & $1.25e{+1}\!\pm\!8.86e{+0}$ & $3.91e{+2}\!\pm\!1.80e{+2}$ & $6.38e{+4}\!\pm\!1.24e{+4}$ & $1.73e{+4}\!\pm\!1.56e{+3}$ \\
4 & $1.63e{+1}\!\pm\!1.29e{+1}$ & $1.41e{+5}\!\pm\!2.11e{+4}$ & $4.64e{+5}\!\pm\!1.50e{+5}$ & $1.20e{+5}\!\pm\!6.57e{+4}$ \\
5 & $2.26e{+1}\!\pm\!6.33e{+0}$ & $8.95e{+5}\!\pm\!1.36e{+5}$ & $6.71e{+5}\!\pm\!1.30e{+5}$ & $9.25e{+4}\!\pm\!2.44e{+4}$ \\
\bottomrule
\end{tabular}
\label{tab:hopper_order_vs_p}
\end{table}

\subsection{SDE Dynamics (ablations with order $k$) for linear, non-linear and oscillatory systems}

From the ablations in this section, we observe that higher-degree B-splines are useful for modeling nonlinear systems such as the Lotka--Volterra system in the stochastic case.

\subsubsection{Exponential Decay ($p=0$)}

\begin{table}[H]
\caption{Effect of spline order on Exp-Decay performance.}
%\vspace{-0.05in}
\centering
\small
\setlength{\tabcolsep}{4pt}
\renewcommand{\arraystretch}{1.05}
\resizebox{1.0\textwidth}{!}{
\begin{tabular}{lccccc}
\toprule
\textbf{Degree} & \textbf{MSE$_\text{PODE}$} & \textbf{MSE$_\text{SDE}$} & \textbf{Wasserstein} & \textbf{MMD} & \textbf{Energy} \\
\midrule
1 & $3.91e{-1}\!\pm\!4.23e{-3}$ & $4.56e{-1}\!\pm\!8.27e{-3}$ & $3.31e{-1}\!\pm\!2.83e{-3}$ & $2.73e{-2}\!\pm\!2.48e{-3}$ & $6.94e{-2}\!\pm\!1.08e{-3}$ \\
2 & $1.25e{+0}\!\pm\!4.04e{-1}$ & $2.38e{+0}\!\pm\!5.77e{-1}$ & $7.94e{-1}\!\pm\!1.78e{-1}$ & $3.38e{-2}\!\pm\!3.08e{-3}$ & $1.12e{-1}\!\pm\!1.67e{-2}$ \\
3 & $8.22e{-1}\!\pm\!6.78e{-2}$ & $2.03e{+0}\!\pm\!4.47e{-1}$ & $6.87e{-1}\!\pm\!1.70e{-1}$ & $3.85e{-2}\!\pm\!1.01e{-2}$ & $1.25e{-1}\!\pm\!3.65e{-2}$ \\
4 & $1.20e{+0}\!\pm\!6.28e{-1}$ & $4.74e{+0}\!\pm\!4.43e{+0}$ & $1.17e{+0}\!\pm\!8.69e{-1}$ & $3.46e{-2}\!\pm\!1.40e{-2}$ & $1.30e{-1}\!\pm\!7.66e{-2}$ \\
5 & $1.27e{+0}\!\pm\!3.89e{-1}$ & $9.89e{+0}\!\pm\!1.17e{+1}$ & $1.49e{+0}\!\pm\!1.31e{+0}$ & $6.13e{-2}\!\pm\!1.22e{-2}$ & $2.26e{-1}\!\pm\!9.57e{-2}$ \\
\bottomrule
\end{tabular}
}
\label{tab:expdecay_order_ablation}
\end{table}

\subsubsection{Exponential Decay ($p=0.25$)}

\begin{table}[H]
\caption{Effect of B-spline degree on Exp-decay performance for $p=0.25$.}
%\vspace{-0.05in}
\centering
\small
\setlength{\tabcolsep}{4pt}
\renewcommand{\arraystretch}{1.05}
\resizebox{1.0\textwidth}{!}{
\begin{tabular}{lccccc}
\toprule
\textbf{Degree} & \textbf{MSE$_\text{PODE}$} & \textbf{MSE$_\text{SDE}$} & \textbf{Wasserstein} & \textbf{MMD} & \textbf{Energy} \\
\midrule
1 & $3.66e{-1}\!\pm\!6.08e{-3}$ & $4.70e{-1}\!\pm\!5.90e{-3}$ & $2.51e{-1}\!\pm\!5.53e{-3}$ & $1.51e{-2}\!\pm\!1.03e{-3}$ & $4.12e{-2}\!\pm\!3.14e{-3}$ \\
2 & $7.30e{-1}\!\pm\!1.41e{-1}$ & $1.49e{+0}\!\pm\!3.84e{-2}$ & $5.52e{-1}\!\pm\!3.67e{-2}$ & $4.28e{-2}\!\pm\!1.87e{-2}$ & $1.43e{-1}\!\pm\!6.38e{-2}$ \\
3 & $1.37e{+0}\!\pm\!8.59e{-1}$ & $2.55e{+0}\!\pm\!9.96e{-1}$ & $7.75e{-1}\!\pm\!2.18e{-1}$ & $4.36e{-2}\!\pm\!1.19e{-2}$ & $1.47e{-1}\!\pm\!4.29e{-2}$ \\
4 & $8.48e{-1}\!\pm\!2.68e{-1}$ & $2.73e{+0}\!\pm\!1.64e{+0}$ & $8.58e{-1}\!\pm\!3.85e{-1}$ & $4.18e{-2}\!\pm\!7.33e{-3}$ & $1.38e{-1}\!\pm\!3.88e{-2}$ \\
5 & $3.12e{+0}\!\pm\!1.10e{+0}$ & $4.41e{+0}\!\pm\!8.47e{-1}$ & $1.17e{+0}\!\pm\!1.97e{-1}$ & $6.01e{-2}\!\pm\!1.21e{-2}$ & $2.60e{-1}\!\pm\!4.65e{-2}$ \\
\bottomrule
\end{tabular}
}
\label{tab:expdecay_p025_degree}
\end{table}

\subsubsection{Exponential Decay ($p=0.5$)}

\begin{table}[H]
\caption{Effect of B-spline degree on Exp-decay performance for $p=0.5$.}
%\vspace{-0.05in}
\centering
\small
\setlength{\tabcolsep}{4pt}
\renewcommand{\arraystretch}{1.05}
\resizebox{1.0\textwidth}{!}{
\begin{tabular}{lccccc}
\toprule
\textbf{Degree} & \textbf{MSE$_\text{PODE}$} & \textbf{MSE$_\text{SDE}$} & \textbf{Wasserstein} & \textbf{MMD} & \textbf{Energy} \\
\midrule
1 & $3.76e{-1}\!\pm\!3.30e{-3}$ & $5.30e{-1}\!\pm\!6.59e{-3}$ & $1.95e{-1}\!\pm\!1.42e{-3}$ & $7.42e{-3}\!\pm\!3.31e{-4}$ & $2.40e{-2}\!\pm\!1.36e{-3}$ \\
2 & $6.44e{-1}\!\pm\!5.12e{-2}$ & $1.80e{+0}\!\pm\!5.27e{-1}$ & $5.92e{-1}\!\pm\!1.49e{-1}$ & $2.92e{-2}\!\pm\!8.18e{-3}$ & $8.90e{-2}\!\pm\!1.62e{-2}$ \\
3 & $6.52e{-1}\!\pm\!9.32e{-2}$ & $1.61e{+0}\!\pm\!9.35e{-2}$ & $5.57e{-1}\!\pm\!4.34e{-2}$ & $2.34e{-2}\!\pm\!9.21e{-3}$ & $7.47e{-2}\!\pm\!2.55e{-2}$ \\
4 & $8.61e{-1}\!\pm\!3.17e{-1}$ & $2.78e{+0}\!\pm\!1.86e{+0}$ & $8.49e{-1}\!\pm\!4.06e{-1}$ & $3.21e{-2}\!\pm\!1.25e{-2}$ & $1.34e{-1}\!\pm\!5.61e{-2}$ \\
5 & $2.75e{+0}\!\pm\!2.96e{+0}$ & $3.53e{+0}\!\pm\!3.21e{+0}$ & $9.04e{-1}\!\pm\!6.04e{-1}$ & $7.77e{-2}\!\pm\!5.26e{-2}$ & $3.93e{-1}\!\pm\!3.49e{-1}$ \\
\bottomrule
\end{tabular}
}
\label{tab:expdecay_p05_degree}
\end{table}

\subsubsection{Exponential Decay ($p=0.75$)}

\begin{table}[H]
\caption{Effect of B-spline degree on Exp-decay performance for $p=0.75$.}
%\vspace{-0.05in}
\centering
\small
\setlength{\tabcolsep}{4pt}
\renewcommand{\arraystretch}{1.05}
\resizebox{1.0\textwidth}{!}{
\begin{tabular}{lccccc}
\toprule
\textbf{Degree} & \textbf{MSE$_\text{PODE}$} & \textbf{MSE$_\text{SDE}$} & \textbf{Wasserstein} & \textbf{MMD} & \textbf{Energy} \\
\midrule
1 & $4.19e{-1}\!\pm\!5.87e{-3}$ & $6.77e{-1}\!\pm\!8.23e{-3}$ & $1.42e{-1}\!\pm\!1.72e{-2}$ & $2.06e{-3}\!\pm\!1.55e{-3}$ & $1.51e{-2}\!\pm\!5.15e{-3}$ \\
2 & $7.83e{-1}\!\pm\!4.41e{-2}$ & $2.24e{+0}\!\pm\!1.00e{+0}$ & $7.53e{-1}\!\pm\!2.53e{-1}$ & $3.91e{-2}\!\pm\!9.25e{-3}$ & $1.43e{-1}\!\pm\!5.00e{-2}$ \\
3 & $9.82e{-1}\!\pm\!7.21e{-2}$ & $5.56e{+0}\!\pm\!3.84e{+0}$ & $1.39e{+0}\!\pm\!6.96e{-1}$ & $4.30e{-2}\!\pm\!1.32e{-2}$ & $1.69e{-1}\!\pm\!3.47e{-2}$ \\
4 & $1.12e{+0}\!\pm\!9.37e{-1}$ & $2.61e{+0}\!\pm\!5.93e{-1}$ & $9.23e{-1}\!\pm\!1.44e{-1}$ & $5.42e{-2}\!\pm\!2.47e{-2}$ & $2.26e{-1}\!\pm\!1.46e{-1}$ \\
5 & $3.18e{+5}\!\pm\!4.50e{+5}$ & $3.03e{+5}\!\pm\!4.23e{+5}$ & $1.99e{+2}\!\pm\!2.47e{+2}$ & $2.38e{-1}\!\pm\!2.83e{-1}$ & $2.61e{+2}\!\pm\!3.68e{+2}$ \\
\bottomrule
\end{tabular}
}
\label{tab:expdecay_p075_degree}
\end{table}

\subsubsection{Lotka--Volterra ($p=0$)}

\begin{table}[H]
\caption{Effect of spline order on Lotka–Volterra performance.}
%\vspace{-0.05in}
\centering
\small
\setlength{\tabcolsep}{4pt}
\renewcommand{\arraystretch}{1.05}
\resizebox{1.0\textwidth}{!}{
\begin{tabular}{lccccc}
\toprule
\textbf{Degree} & \textbf{MSE$_\text{PODE}$} & \textbf{MSE$_\text{SDE}$} & \textbf{Wasserstein} & \textbf{MMD} & \textbf{Energy} \\
\midrule
1 & $4.70e{-1}\!\pm\!9.54e{-2}$ & $4.69e{-1}\!\pm\!9.23e{-2}$ & $4.27e{-1}\!\pm\!7.98e{-2}$ & $1.12e{-1}\!\pm\!6.87e{-2}$ & $4.26e{-1}\!\pm\!2.38e{-1}$ \\
2 & $2.68e{-1}\!\pm\!1.44e{-2}$ & $2.73e{-1}\!\pm\!1.37e{-2}$ & $2.94e{-1}\!\pm\!3.16e{-3}$ & $4.21e{-2}\!\pm\!3.45e{-3}$ & $1.56e{-1}\!\pm\!1.33e{-2}$ \\
3 & $3.03e{-1}\!\pm\!6.56e{-3}$ & $3.03e{-1}\!\pm\!1.10e{-2}$ & $3.21e{-1}\!\pm\!8.22e{-3}$ & $4.64e{-2}\!\pm\!3.21e{-3}$ & $1.74e{-1}\!\pm\!1.36e{-2}$ \\
4 & $3.47e{-1}\!\pm\!3.19e{-2}$ & $3.47e{-1}\!\pm\!3.10e{-2}$ & $3.60e{-1}\!\pm\!2.25e{-2}$ & $6.97e{-2}\!\pm\!1.63e{-2}$ & $2.60e{-1}\!\pm\!5.59e{-2}$ \\
5 & $3.28e{-1}\!\pm\!2.62e{-2}$ & $3.30e{-1}\!\pm\!2.62e{-2}$ & $3.50e{-1}\!\pm\!3.36e{-2}$ & $6.87e{-2}\!\pm\!2.85e{-2}$ & $2.52e{-1}\!\pm\!9.45e{-2}$ \\
\bottomrule
\end{tabular}
}
\label{tab:lv_order_ablation}
\end{table}

\subsubsection{Lotka--Volterra ($p=0.25$)}

\begin{table}[H]
\caption{Effect of B-spline degree on Lotka--Volterra performance for $p=0.25$.}
%\vspace{-0.05in}
\centering
\small
\setlength{\tabcolsep}{4pt}
\renewcommand{\arraystretch}{1.05}
\resizebox{1.0\textwidth}{!}{
\begin{tabular}{lccccc}
\toprule
\textbf{Degree} & \textbf{MSE$_\text{PODE}$} & \textbf{MSE$_\text{SDE}$} & \textbf{Wasserstein} & \textbf{MMD} & \textbf{Energy} \\
\midrule
1 & $5.22e{-1}\!\pm\!1.72e{-1}$ & $5.24e{-1}\!\pm\!1.71e{-1}$ & $4.56e{-1}\!\pm\!1.33e{-1}$ & $1.32e{-1}\!\pm\!1.11e{-1}$ & $4.89e{-1}\!\pm\!3.77e{-1}$ \\
2 & $3.40e{-1}\!\pm\!5.04e{-2}$ & $3.38e{-1}\!\pm\!4.76e{-2}$ & $3.20e{-1}\!\pm\!1.01e{-2}$ & $4.57e{-2}\!\pm\!2.87e{-3}$ & $1.82e{-1}\!\pm\!1.18e{-2}$ \\
3 & $3.34e{-1}\!\pm\!5.05e{-2}$ & $3.39e{-1}\!\pm\!4.86e{-2}$ & $3.37e{-1}\!\pm\!1.39e{-2}$ & $6.22e{-2}\!\pm\!1.12e{-2}$ & $2.47e{-1}\!\pm\!5.83e{-2}$ \\
4 & $3.18e{-1}\!\pm\!3.58e{-2}$ & $3.21e{-1}\!\pm\!3.21e{-2}$ & $3.19e{-1}\!\pm\!1.21e{-2}$ & $4.45e{-2}\!\pm\!5.03e{-3}$ & $1.75e{-1}\!\pm\!2.60e{-2}$ \\
5 & $2.81e{-1}\!\pm\!1.70e{-2}$ & $2.83e{-1}\!\pm\!1.54e{-2}$ & $3.19e{-1}\!\pm\!5.33e{-3}$ & $4.83e{-2}\!\pm\!3.97e{-3}$ & $1.83e{-1}\!\pm\!1.57e{-2}$ \\
\bottomrule
\end{tabular}
}
\label{tab:lv_p025_degree}
\end{table}

\subsubsection{Lotka--Volterra ($p=0.5$)}

\begin{table}[H]
\caption{Effect of B-spline degree on Lotka--Volterra performance for $p=0.5$.}
%\vspace{-0.05in}
\centering
\small
\setlength{\tabcolsep}{4pt}
\renewcommand{\arraystretch}{1.05}
\resizebox{1.0\textwidth}{!}{
\begin{tabular}{lccccc}
\toprule
\textbf{Degree} & \textbf{MSE$_\text{PODE}$} & \textbf{MSE$_\text{SDE}$} & \textbf{Wasserstein} & \textbf{MMD} & \textbf{Energy} \\
\midrule
1 & $3.82e{-1}\!\pm\!6.10e{-2}$ & $3.83e{-1}\!\pm\!5.94e{-2}$ & $3.70e{-1}\!\pm\!1.68e{-2}$ & $7.09e{-2}\!\pm\!4.52e{-3}$ & $2.75e{-1}\!\pm\!1.38e{-2}$ \\
2 & $4.11e{-1}\!\pm\!4.68e{-2}$ & $4.15e{-1}\!\pm\!4.79e{-2}$ & $3.74e{-1}\!\pm\!3.81e{-2}$ & $6.24e{-2}\!\pm\!9.27e{-3}$ & $2.59e{-1}\!\pm\!3.66e{-2}$ \\
3 & $2.88e{-1}\!\pm\!2.32e{-2}$ & $2.85e{-1}\!\pm\!2.54e{-2}$ & $3.10e{-1}\!\pm\!1.16e{-2}$ & $4.42e{-2}\!\pm\!3.80e{-3}$ & $1.64e{-1}\!\pm\!2.32e{-2}$ \\
4 & $3.70e{-1}\!\pm\!1.10e{-1}$ & $3.72e{-1}\!\pm\!1.11e{-1}$ & $3.48e{-1}\!\pm\!4.99e{-2}$ & $5.39e{-2}\!\pm\!2.09e{-2}$ & $2.22e{-1}\!\pm\!1.02e{-1}$ \\
5 & $3.19e{-1}\!\pm\!2.96e{-2}$ & $3.23e{-1}\!\pm\!3.05e{-2}$ & $3.11e{-1}\!\pm\!1.57e{-2}$ & $4.50e{-2}\!\pm\!2.94e{-3}$ & $1.79e{-1}\!\pm\!7.45e{-3}$ \\
\bottomrule
\end{tabular}
}
\label{tab:lv_p05_degree}
\end{table}

\subsubsection{Lotka--Volterra ($p=0.75$)}

\begin{table}[H]
\caption{Effect of B-spline degree on Lotka--Volterra performance for $p=0.75$.}
%\vspace{-0.05in}
\centering
\small
\setlength{\tabcolsep}{4pt}
\renewcommand{\arraystretch}{1.05}
\resizebox{1.0\textwidth}{!}{
\begin{tabular}{lccccc}
\toprule
\textbf{Degree} & \textbf{MSE$_\text{PODE}$} & \textbf{MSE$_\text{SDE}$} & \textbf{Wasserstein} & \textbf{MMD} & \textbf{Energy} \\
\midrule
1 & $4.98e{-1}\!\pm\!1.36e{-1}$ & $4.99e{-1}\!\pm\!1.35e{-1}$ & $4.53e{-1}\!\pm\!4.31e{-2}$ & $1.21e{-1}\!\pm\!1.03e{-2}$ & $4.53e{-1}\!\pm\!7.16e{-2}$ \\
2 & $4.30e{-1}\!\pm\!7.67e{-2}$ & $4.33e{-1}\!\pm\!7.35e{-2}$ & $3.97e{-1}\!\pm\!6.21e{-2}$ & $8.44e{-2}\!\pm\!3.38e{-2}$ & $3.30e{-1}\!\pm\!1.15e{-1}$ \\
3 & $3.77e{-1}\!\pm\!5.73e{-2}$ & $3.78e{-1}\!\pm\!5.48e{-2}$ & $3.61e{-1}\!\pm\!4.04e{-2}$ & $5.57e{-2}\!\pm\!9.43e{-3}$ & $2.23e{-1}\!\pm\!4.77e{-2}$ \\
4 & $4.49e{-1}\!\pm\!1.40e{-1}$ & $4.43e{-1}\!\pm\!1.29e{-1}$ & $3.55e{-1}\!\pm\!4.84e{-2}$ & $6.56e{-2}\!\pm\!1.72e{-2}$ & $2.75e{-1}\!\pm\!9.15e{-2}$ \\
5 & $6.03e{-1}\!\pm\!2.33e{-1}$ & $5.98e{-1}\!\pm\!2.31e{-1}$ & $4.60e{-1}\!\pm\!7.29e{-2}$ & $1.26e{-1}\!\pm\!4.30e{-2}$ & $5.06e{-1}\!\pm\!1.91e{-1}$ \\
\bottomrule
\end{tabular}
}
\label{tab:lv_p075_degree}
\end{table}

\subsubsection{Damped Harmonic Oscillator ($p=0$)}

\begin{table}[H]
\caption{Effect of spline order on damped harmonic oscillator performance.}
%\vspace{-0.05in}
\centering
\small
\setlength{\tabcolsep}{4pt}
\renewcommand{\arraystretch}{1.05}
\resizebox{1.0\textwidth}{!}{
\begin{tabular}{lccccc}
\toprule
\textbf{Degree} & \textbf{MSE$_\text{PODE}$} & \textbf{MSE$_\text{SDE}$} & \textbf{Wasserstein} & \textbf{MMD} & \textbf{Energy} \\
\midrule
1 & $8.95e{-1}\!\pm\!5.12e{-3}$ & $9.70e{-1}\!\pm\!9.41e{-3}$ & $4.10e{-1}\!\pm\!4.41e{-3}$ & $4.19e{-2}\!\pm\!1.54e{-3}$ & $1.12e{-1}\!\pm\!7.12e{-3}$ \\
2 & $1.78e{+0}\!\pm\!1.96e{-1}$ & $8.02e{+0}\!\pm\!4.45e{+0}$ & $1.58e{+0}\!\pm\!6.03e{-1}$ & $7.48e{-3}\!\pm\!1.30e{-3}$ & $8.23e{-2}\!\pm\!3.24e{-2}$ \\
3 & $5.51e{+0}\!\pm\!2.89e{+0}$ & $1.30e{+1}\!\pm\!7.52e{+0}$ & $2.07e{+0}\!\pm\!6.98e{-1}$ & $2.08e{-2}\!\pm\!1.41e{-2}$ & $2.06e{-1}\!\pm\!1.72e{-1}$ \\
4 & $9.23e{+0}\!\pm\!2.11e{+0}$ & $1.56e{+1}\!\pm\!2.84e{+0}$ & $2.43e{+0}\!\pm\!2.60e{-1}$ & $1.69e{-2}\!\pm\!7.13e{-3}$ & $1.90e{-1}\!\pm\!4.45e{-2}$ \\
5 & $8.53e{+0}\!\pm\!5.52e{+0}$ & $2.21e{+1}\!\pm\!1.21e{+1}$ & $3.04e{+0}\!\pm\!7.89e{-1}$ & $2.13e{-2}\!\pm\!1.77e{-3}$ & $2.67e{-1}\!\pm\!1.09e{-1}$ \\
\bottomrule
\end{tabular}
}
\label{tab:damped_order_ablation}
\end{table}

\subsubsection{Damped Harmonic Oscillator ($p=0.25$)}

\begin{table}[H]
\caption{Effect of B-spline degree on Damped-Harmonic performance for $p=0.25$.}
%\vspace{-0.05in}
\centering
\small
\setlength{\tabcolsep}{4pt}
\renewcommand{\arraystretch}{1.05}
\resizebox{1.0\textwidth}{!}{
\begin{tabular}{lccccc}
\toprule
\textbf{Degree} & \textbf{MSE$_\text{PODE}$} & \textbf{MSE$_\text{SDE}$} & \textbf{Wasserstein} & \textbf{MMD} & \textbf{Energy} \\
\midrule
1 & $7.99e{-1}\!\pm\!1.57e{-2}$ & $9.14e{-1}\!\pm\!2.28e{-2}$ & $3.65e{-1}\!\pm\!2.11e{-2}$ & $3.42e{-2}\!\pm\!4.34e{-3}$ & $9.17e{-2}\!\pm\!1.07e{-2}$ \\
2 & $2.22e{+0}\!\pm\!8.18e{-1}$ & $3.48e{+0}\!\pm\!7.77e{-1}$ & $8.66e{-1}\!\pm\!1.83e{-1}$ & $1.43e{-2}\!\pm\!2.24e{-3}$ & $9.77e{-2}\!\pm\!7.13e{-3}$ \\
3 & $1.49e{+0}\!\pm\!3.66e{-2}$ & $6.24e{+0}\!\pm\!2.46e{+0}$ & $1.33e{+0}\!\pm\!4.36e{-1}$ & $7.18e{-3}\!\pm\!3.63e{-3}$ & $7.58e{-2}\!\pm\!2.36e{-2}$ \\
4 & $4.11e{+0}\!\pm\!8.82e{-1}$ & $9.91e{+0}\!\pm\!3.36e{+0}$ & $1.76e{+0}\!\pm\!2.72e{-1}$ & $2.62e{-2}\!\pm\!2.77e{-3}$ & $1.96e{-1}\!\pm\!2.59e{-2}$ \\
5 & $7.85e{+0}\!\pm\!3.86e{+0}$ & $1.52e{+1}\!\pm\!7.33e{+0}$ & $2.34e{+0}\!\pm\!7.83e{-1}$ & $3.32e{-2}\!\pm\!2.33e{-2}$ & $3.66e{-1}\!\pm\!1.77e{-1}$ \\
\bottomrule
\end{tabular}
}
\label{tab:damped_p025_degree}
\end{table}

\subsubsection{Damped Harmonic Oscillator ($p=0.5$)}

\begin{table}[H]
\caption{Effect of B-spline degree on Damped-Harmonic performance for $p=0.5$.}
%\vspace{-0.05in}
\centering
\small
\setlength{\tabcolsep}{4pt}
\renewcommand{\arraystretch}{1.05}
\resizebox{1.0\textwidth}{!}{
\begin{tabular}{lccccc}
\toprule
\textbf{Degree} & \textbf{MSE$_\text{PODE}$} & \textbf{MSE$_\text{SDE}$} & \textbf{Wasserstein} & \textbf{MMD} & \textbf{Energy} \\
\midrule
1 & $8.05e{-1}\!\pm\!8.73e{-3}$ & $9.81e{-1}\!\pm\!3.99e{-2}$ & $3.21e{-1}\!\pm\!2.84e{-2}$ & $2.18e{-2}\!\pm\!5.59e{-3}$ & $6.87e{-2}\!\pm\!1.21e{-2}$ \\
2 & $1.61e{+0}\!\pm\!4.11e{-1}$ & $9.25e{+0}\!\pm\!7.31e{+0}$ & $1.52e{+0}\!\pm\!9.81e{-1}$ & $1.23e{-2}\!\pm\!7.06e{-3}$ & $1.04e{-1}\!\pm\!6.55e{-2}$ \\
3 & $1.34e{+0}\!\pm\!2.35e{-1}$ & $3.23e{+0}\!\pm\!1.62e{+0}$ & $7.52e{-1}\!\pm\!3.50e{-1}$ & $1.05e{-2}\!\pm\!4.37e{-3}$ & $7.49e{-2}\!\pm\!2.28e{-2}$ \\
4 & $1.64e{+0}\!\pm\!4.51e{-1}$ & $1.38e{+1}\!\pm\!1.04e{+1}$ & $2.08e{+0}\!\pm\!1.07e{+0}$ & $1.02e{-2}\!\pm\!3.90e{-3}$ & $9.98e{-2}\!\pm\!4.27e{-2}$ \\
5 & $4.26e{+0}\!\pm\!2.48e{+0}$ & $2.05e{+1}\!\pm\!2.01e{+1}$ & $2.71e{+0}\!\pm\!1.93e{+0}$ & $5.87e{-3}\!\pm\!2.05e{-3}$ & $6.62e{-2}\!\pm\!2.93e{-2}$ \\
\bottomrule
\end{tabular}
}
\label{tab:damped_p05_degree}
\end{table}

\subsubsection{Damped Harmonic Oscillator ($p=0.75$)}

\begin{table}[H]
\caption{Effect of B-spline degree on Damped-Harmonic performance for $p=0.75$.}
%\vspace{-0.05in}
\centering
\small
\setlength{\tabcolsep}{4pt}
\renewcommand{\arraystretch}{1.05}
\resizebox{1.0\textwidth}{!}{
\begin{tabular}{lccccc}
\toprule
\textbf{Degree} & \textbf{MSE$_\text{PODE}$} & \textbf{MSE$_\text{SDE}$} & \textbf{Wasserstein} & \textbf{MMD} & \textbf{Energy} \\
\midrule
1 & $9.38e{-1}\!\pm\!2.89e{-2}$ & $1.17e{+0}\!\pm\!1.17e{-2}$ & $2.86e{-1}\!\pm\!1.10e{-2}$ & $1.24e{-2}\!\pm\!7.65e{-4}$ & $5.28e{-2}\!\pm\!4.89e{-3}$ \\
2 & $2.24e{+0}\!\pm\!6.32e{-1}$ & $7.16e{+0}\!\pm\!5.03e{-1}$ & $1.56e{+0}\!\pm\!1.18e{-1}$ & $1.29e{-2}\!\pm\!4.25e{-3}$ & $1.02e{-1}\!\pm\!3.60e{-2}$ \\
3 & $3.34e{+0}\!\pm\!2.64e{+0}$ & $8.62e{+0}\!\pm\!5.79e{+0}$ & $1.43e{+0}\!\pm\!5.99e{-1}$ & $1.13e{-2}\!\pm\!1.98e{-3}$ & $9.04e{-2}\!\pm\!1.63e{-2}$ \\
4 & $7.48e{+0}\!\pm\!8.12e{+0}$ & $3.67e{+1}\!\pm\!2.28e{+1}$ & $3.83e{+0}\!\pm\!1.39e{+0}$ & $3.72e{-2}\!\pm\!3.14e{-2}$ & $2.98e{-1}\!\pm\!2.59e{-1}$ \\
5 & $8.79e{+1}\!\pm\!1.22e{+2}$ & $4.87e{+1}\!\pm\!5.88e{+1}$ & $3.11e{+0}\!\pm\!2.19e{+0}$ & $3.30e{-2}\!\pm\!3.02e{-2}$ & $3.06e{-1}\!\pm\!3.25e{-1}$ \\
\bottomrule
\end{tabular}
}
\label{tab:damped_p075_degree}
\end{table}

\subsection{Chaotic Lorenz Systems}
In this section, we conduct ablation studies on the Lorenz system, a highly chaotic system.
We observe that higher-degree interpolants are especially effective for modeling stochastic chaotic dynamics, whereas all degree $m>2$ variants perform better than the linear baseline in the deterministic cases.

\subsubsection{Lorenz ODE ($p=0$)}

\begin{table}[H]
\caption{Lorenz system results for regularly sampled observations $p=0$.}
%\vspace{-0.05in}
\centering
\small
\begin{tabular}{lc}
\toprule
\textbf{Degree} & \textbf{MSE} \\
\midrule
1 & $1.380 \pm 1.890\mathrm{e}{-02}$ \\
2 & $6.800\mathrm{e}{-01} \pm 1.520\mathrm{e}{-02}$ \\
3 & $6.400\mathrm{e}{-01} \pm 3.790\mathrm{e}{-03}$ \\
4 & $7.800\mathrm{e}{-01} \pm 1.320\mathrm{e}{-01}$ \\
5 & $9.200\mathrm{e}{-01} \pm 7.360\mathrm{e}{-02}$ \\

\bottomrule
\end{tabular}
\label{tab:lorenz_order0}
\end{table}

\subsubsection{Lorenz SDE ($p=0$)}

\begin{table}[H]
\caption{Effect of spline order on stochastic Lorenz system performance.}
%\vspace{-0.05in}
\centering
\small
\setlength{\tabcolsep}{4pt}
\renewcommand{\arraystretch}{1.05}
\resizebox{1.0\textwidth}{!}{
\begin{tabular}{lccccc}
\toprule
\textbf{Degree} & \textbf{MSE$_\text{PODE}$} & \textbf{MSE$_\text{SDE}$} & \textbf{Wasserstein} & \textbf{MMD} & \textbf{Energy} \\
\midrule
1 & $2.21e{+0}\!\pm\!9.71e{-3}$ & $2.21e{+0}\!\pm\!9.67e{-3}$ & $6.31e{-1}\!\pm\!2.03e{-3}$ & $1.30e{-1}\!\pm\!1.22e{-3}$ & $4.01e{-1}\!\pm\!3.50e{-3}$ \\
2 & $2.19e{+0}\!\pm\!1.38e{-2}$ & $2.19e{+0}\!\pm\!1.88e{-2}$ & $6.30e{-1}\!\pm\!1.90e{-3}$ & $1.29e{-1}\!\pm\!2.21e{-4}$ & $3.98e{-1}\!\pm\!1.52e{-3}$ \\
3 & $4.82e{+0}\!\pm\!4.96e{+0}$ & $3.71e{+0}\!\pm\!3.37e{+0}$ & $9.31e{-1}\!\pm\!9.85e{-1}$ & $7.44e{-3}\!\pm\!6.53e{-3}$ & $5.61e{-2}\!\pm\!3.44e{-2}$ \\
4 & $1.23e{+0}\!\pm\!4.01e{-2}$ & $1.23e{+0}\!\pm\!5.62e{-2}$ & $1.80e{-1}\!\pm\!1.47e{-2}$ & $4.51e{-3}\!\pm\!1.79e{-3}$ & $3.66e{-2}\!\pm\!7.65e{-3}$ \\
5 & $1.36e{+0}\!\pm\!6.41e{-2}$ & $1.50e{+0}\!\pm\!2.46e{-1}$ & $3.07e{-1}\!\pm\!1.49e{-1}$ & $6.44e{-3}\!\pm\!1.73e{-3}$ & $4.47e{-2}\!\pm\!1.08e{-2}$ \\
\bottomrule
\end{tabular}
}
\label{tab:lorenz_order_ablation}
\end{table}

\subsection{Experiments on Transcriptomics Datasets}
\label{apdx:PHATE cellular dynamics}

\subsubsection{Brain Regeneration}
From the ablation study, we observe that higher-degree variants perform especially well with PHATE embeddings. While the degree $m=2$ is often the optimal variant for PCA embedding. This aligns with PHATE embeddings, which are designed to capture the structural organization of datasets and thus benefit from B-splines. 

\subsubsection{PHATE Embedding: Interpolation}

\begin{table}[H]
\caption{Effect of B-spline degree on PHATE interpolation performance.}
%\vspace{-0.05in}
\centering
\small
\setlength{\tabcolsep}{4pt}
\renewcommand{\arraystretch}{1.05}
\begin{tabular}{lccc}
\toprule
\textbf{Degree} & \textbf{ Wasserstein} & \textbf{ MMD} & \textbf{ Energy} \\
\midrule
1 & $7.714\mathrm{e}{-01} \pm 1.250\mathrm{e}{-02}$ & $4.813\mathrm{e}{-01} \pm 1.740\mathrm{e}{-02}$ & $1.391\mathrm{e}{+00} \pm 7.250\mathrm{e}{-02}$ \\
2 & $6.421\mathrm{e}{-01} \pm 5.280\mathrm{e}{-02}$ & $3.088\mathrm{e}{-01} \pm 2.570\mathrm{e}{-02}$ & $9.172\mathrm{e}{-01} \pm 9.580\mathrm{e}{-02}$ \\
3 & $6.519\mathrm{e}{-01} \pm 2.100\mathrm{e}{-02}$ & $3.174\mathrm{e}{-01} \pm 5.200\mathrm{e}{-03}$ & $9.277\mathrm{e}{-01} \pm 2.030\mathrm{e}{-02}$ \\
4 & $6.548\mathrm{e}{-01} \pm 2.930\mathrm{e}{-02}$ & $3.120\mathrm{e}{-01} \pm 2.020\mathrm{e}{-02}$ & $9.189\mathrm{e}{-01} \pm 7.080\mathrm{e}{-02}$ \\
5 & $6.066\mathrm{e}{-01} \pm 2.710\mathrm{e}{-02}$ & $2.846\mathrm{e}{-01} \pm 6.600\mathrm{e}{-03}$ & $8.152\mathrm{e}{-01} \pm 2.140\mathrm{e}{-02}$ \\
\bottomrule
\end{tabular}
\label{tab:phate_interp_degree}
\end{table}

\subsubsection{PHATE Embedding: Extrapolation}

\begin{table}[H]
\caption{Effect of B-spline degree on PHATE extrapolation performance.}
%\vspace{-0.05in}
\centering
\small
\setlength{\tabcolsep}{4pt}
\renewcommand{\arraystretch}{1.05}
\begin{tabular}{lccc}
\toprule
\textbf{Degree} & \textbf{ Wasserstein} & \textbf{MMD} & \textbf{ Energy} \\
\midrule
1 & $1.588 \pm 2.270\mathrm{e}{-02}$ & $9.796\mathrm{e}{-01} \pm 4.000\mathrm{e}{-03}$ & $3.843 \pm 3.700\mathrm{e}{-02}$ \\
2 & $1.584 \pm 3.670\mathrm{e}{-02}$ & $9.654\mathrm{e}{-01} \pm 6.300\mathrm{e}{-03}$ & $3.787 \pm 6.170\mathrm{e}{-02}$ \\
3 & $1.587 \pm 1.700\mathrm{e}{-02}$ & $9.623\mathrm{e}{-01} \pm 2.700\mathrm{e}{-03}$ & $3.761 \pm 2.500\mathrm{e}{-02}$ \\
4 & $1.536 \pm 1.920\mathrm{e}{-02}$ & $9.550\mathrm{e}{-01} \pm 8.700\mathrm{e}{-03}$ & $3.740 \pm 4.240\mathrm{e}{-02}$ \\
5 & $1.545 \pm 2.720\mathrm{e}{-02}$ & $9.644\mathrm{e}{-01} \pm 3.200\mathrm{e}{-03}$ & $3.751 \pm 1.290\mathrm{e}{-02}$ \\
\bottomrule
\end{tabular}
\label{tab:phate_extra_degree}
\end{table}

\subsubsection{PCA Embedding: Interpolation}

\begin{table}[H]
\caption{Effect of B-spline degree on PCA interpolation performance.}
%\vspace{-0.05in}
\centering
\small
\setlength{\tabcolsep}{4pt}
\renewcommand{\arraystretch}{1.05}
\begin{tabular}{lccc}
\toprule
\textbf{Degree} & \textbf{ Wasserstein} & \textbf{ MMD} & \textbf{ Energy} \\
\midrule
1 & $8.719\mathrm{e}{-01} \pm 3.180\mathrm{e}{-02}$ & $6.320\mathrm{e}{-02} \pm 3.600\mathrm{e}{-03}$ & $1.120 \pm 4.510\mathrm{e}{-02}$ \\
2 & $8.380\mathrm{e}{-01} \pm 1.370\mathrm{e}{-02}$ & $5.820\mathrm{e}{-02} \pm 2.500\mathrm{e}{-03}$ & $1.019 \pm 3.930\mathrm{e}{-02}$ \\
3 & $8.346\mathrm{e}{-01} \pm 1.670\mathrm{e}{-02}$ & $6.030\mathrm{e}{-02} \pm 1.200\mathrm{e}{-03}$ & $1.035 \pm 1.660\mathrm{e}{-02}$ \\
4 & $8.502\mathrm{e}{-01} \pm 1.110\mathrm{e}{-02}$ & $6.140\mathrm{e}{-02} \pm 3.000\mathrm{e}{-03}$ & $1.063 \pm 3.660\mathrm{e}{-02}$ \\
5 & $8.512\mathrm{e}{-01} \pm 1.940\mathrm{e}{-02}$ & $6.150\mathrm{e}{-02} \pm 2.600\mathrm{e}{-03}$ & $1.063 \pm 3.940\mathrm{e}{-02}$ \\
\bottomrule
\end{tabular}
\label{tab:pca_interp_degree}
\end{table}

\subsubsection{PCA Embedding: Extrapolation}

\begin{table}[H]
\caption{Effect of B-spline degree on PCA extrapolation performance.}
%\vspace{-0.05in}
\centering
\small
\setlength{\tabcolsep}{4pt}
\renewcommand{\arraystretch}{1.05}
\begin{tabular}{lccc}
\toprule
\textbf{Degree} & \textbf{ Wasserstein} & \textbf{ MMD} & \textbf{ Energy} \\
\midrule
1 & $8.434\mathrm{e}{-01} \pm 2.500\mathrm{e}{-02}$ & $1.730\mathrm{e}{-01} \pm 8.800\mathrm{e}{-03}$ & $1.869 \pm 2.100\mathrm{e}{-02}$ \\
2 & $8.040\mathrm{e}{-01} \pm 3.720\mathrm{e}{-02}$ & $1.292\mathrm{e}{-01} \pm 6.800\mathrm{e}{-03}$ & $1.645 \pm 5.660\mathrm{e}{-02}$ \\
3 & $8.622\mathrm{e}{-01} \pm 4.030\mathrm{e}{-02}$ & $1.415\mathrm{e}{-01} \pm 7.900\mathrm{e}{-03}$ & $1.738 \pm 7.230\mathrm{e}{-02}$ \\
4 & $8.206\mathrm{e}{-01} \pm 4.100\mathrm{e}{-02}$ & $1.355\mathrm{e}{-01} \pm 1.180\mathrm{e}{-02}$ & $1.736 \pm 8.500\mathrm{e}{-02}$ \\
5 & $8.234\mathrm{e}{-01} \pm 1.070\mathrm{e}{-02}$ & $1.336\mathrm{e}{-01} \pm 1.310\mathrm{e}{-02}$ & $1.690 \pm 1.068\mathrm{e}{-01}$ \\
\bottomrule
\end{tabular}
\label{tab:pca_extra_degree}
\end{table}

\subsubsection{Embryoid Evolution}

From the ablation study, we observe that, in almost all cases, the better-performing variant degree is greater than linear for both PHATE and PCA embeddings. The higher-degree variants are especially suitable for extrapolation settings with PCA embeddings. 

\subsubsection{PHATE Embedding: Interpolation}

\begin{table}[H]
\caption{Effect of B-spline degree on PHATE interpolation performance.}
%\vspace{-0.05in}
\centering
\small
\setlength{\tabcolsep}{4pt}
\renewcommand{\arraystretch}{1.05}
\begin{tabular}{lccc}
\toprule
\textbf{Degree} & \textbf{ Wasserstein} & \textbf{ MMD} & \textbf{ Energy} \\
\midrule
1 & $2.355\mathrm{e}{-01} \pm 3.280\mathrm{e}{-02}$ & $2.783\mathrm{e}{-02} \pm 5.360\mathrm{e}{-03}$ & $7.954\mathrm{e}{-02} \pm 1.660\mathrm{e}{-02}$ \\
2 & $2.745\mathrm{e}{-01} \pm 3.700\mathrm{e}{-02}$ & $1.768\mathrm{e}{-02} \pm 3.220\mathrm{e}{-03}$ & $6.331\mathrm{e}{-02} \pm 1.540\mathrm{e}{-02}$ \\
3 & $2.071\mathrm{e}{-01} \pm 1.690\mathrm{e}{-02}$ & $1.260\mathrm{e}{-02} \pm 5.870\mathrm{e}{-04}$ & $3.963\mathrm{e}{-02} \pm 1.400\mathrm{e}{-03}$ \\
4 & $2.560\mathrm{e}{-01} \pm 6.300\mathrm{e}{-02}$ & $1.592\mathrm{e}{-02} \pm 4.450\mathrm{e}{-03}$ & $5.462\mathrm{e}{-02} \pm 1.620\mathrm{e}{-02}$ \\
5 & $2.217\mathrm{e}{-01} \pm 3.530\mathrm{e}{-02}$ & $1.408\mathrm{e}{-02} \pm 3.100\mathrm{e}{-03}$ & $4.467\mathrm{e}{-02} \pm 9.410\mathrm{e}{-03}$ \\
\bottomrule
\end{tabular}
\label{tab:phate_interp_degree2}
\end{table}

\subsubsection{PHATE Embedding: Extrapolation}

\begin{table}[H]
\caption{Effect of B-spline degree on PHATE extrapolation performance.}
%\vspace{-0.05in}
\centering
\small
\setlength{\tabcolsep}{4pt}
\renewcommand{\arraystretch}{1.05}
\begin{tabular}{lccc}
\toprule
\textbf{Degree} & \textbf{ Wasserstein} & \textbf{ MMD} & \textbf{ Energy} \\
\midrule
1 & $5.765\mathrm{e}{-01} \pm 5.860\mathrm{e}{-02}$ & $7.572\mathrm{e}{-02} \pm 1.110\mathrm{e}{-02}$ & $2.833\mathrm{e}{-01} \pm 4.370\mathrm{e}{-02}$ \\
2 & $4.483\mathrm{e}{-01} \pm 4.410\mathrm{e}{-02}$ & $6.902\mathrm{e}{-02} \pm 1.700\mathrm{e}{-03}$ & $2.073\mathrm{e}{-01} \pm 2.960\mathrm{e}{-02}$ \\
3 & $4.774\mathrm{e}{-01} \pm 6.720\mathrm{e}{-02}$ & $6.540\mathrm{e}{-02} \pm 9.270\mathrm{e}{-03}$ & $2.367\mathrm{e}{-01} \pm 5.200\mathrm{e}{-02}$ \\
4 & $5.502\mathrm{e}{-01} \pm 7.380\mathrm{e}{-02}$ & $7.318\mathrm{e}{-02} \pm 1.650\mathrm{e}{-02}$ & $2.916\mathrm{e}{-01} \pm 8.900\mathrm{e}{-02}$ \\
5 & $4.733\mathrm{e}{-01} \pm 6.450\mathrm{e}{-02}$ & $7.417\mathrm{e}{-02} \pm 3.690\mathrm{e}{-03}$ & $2.423\mathrm{e}{-01} \pm 2.410\mathrm{e}{-02}$ \\

\bottomrule
\end{tabular}
\label{tab:phate_extra_degree2}
\end{table}

\subsubsection{PCA Embedding: Interpolation}

\begin{table}[H]
\caption{Effect of B-spline degree on PCA interpolation performance.}
%\vspace{-0.05in}
\centering
\small
\setlength{\tabcolsep}{4pt}
\renewcommand{\arraystretch}{1.05}
\begin{tabular}{lccc}
\toprule
\textbf{Degree} & \textbf{ Wasserstein} & \textbf{ MMD} & \textbf{ Energy} \\
\midrule
1 & $3.008\mathrm{e}{-01} \pm 6.130\mathrm{e}{-03}$ & $1.341\mathrm{e}{-02} \pm 7.060\mathrm{e}{-04}$ & $1.913\mathrm{e}{-01} \pm 8.150\mathrm{e}{-03}$ \\
2 & $3.073\mathrm{e}{-01} \pm 1.260\mathrm{e}{-02}$ & $1.231\mathrm{e}{-02} \pm 5.310\mathrm{e}{-04}$ & $2.070\mathrm{e}{-01} \pm 9.860\mathrm{e}{-03}$ \\
3 & $3.080\mathrm{e}{-01} \pm 1.510\mathrm{e}{-02}$ & $1.328\mathrm{e}{-02} \pm 1.680\mathrm{e}{-03}$ & $2.134\mathrm{e}{-01} \pm 2.560\mathrm{e}{-02}$ \\
4 & $2.959\mathrm{e}{-01} \pm 1.460\mathrm{e}{-02}$ & $1.267\mathrm{e}{-02} \pm 3.110\mathrm{e}{-04}$ & $1.994\mathrm{e}{-01} \pm 9.340\mathrm{e}{-03}$ \\
5 & $3.248\mathrm{e}{-01} \pm 9.690\mathrm{e}{-03}$ & $1.358\mathrm{e}{-02} \pm 4.650\mathrm{e}{-04}$ & $2.230\mathrm{e}{-01} \pm 1.100\mathrm{e}{-02}$ \\

\bottomrule
\end{tabular}
\label{tab:pca_interp_degree2}
\end{table}

\subsubsection{PCA Embedding: Extrapolation}

\begin{table}[H]
\caption{Effect of B-spline degree on PCA extrapolation performance.}
%\vspace{-0.05in}
\centering
\small
\setlength{\tabcolsep}{4pt}
\renewcommand{\arraystretch}{1.05}
\begin{tabular}{lccc}
\toprule
\textbf{Degree} & \textbf{ Wasserstein} & \textbf{ MMD} & \textbf{ Energy} \\
\midrule
1 & $5.652\mathrm{e}{-01} \pm 2.770\mathrm{e}{-02}$ & $3.352\mathrm{e}{-02} \pm 1.930\mathrm{e}{-03}$ & $6.546\mathrm{e}{-01} \pm 5.310\mathrm{e}{-02}$ \\
2 & $5.861\mathrm{e}{-01} \pm 2.590\mathrm{e}{-02}$ & $3.790\mathrm{e}{-02} \pm 2.500\mathrm{e}{-03}$ & $7.244\mathrm{e}{-01} \pm 5.470\mathrm{e}{-02}$ \\
3 & $5.387\mathrm{e}{-01} \pm 3.350\mathrm{e}{-02}$ & $3.583\mathrm{e}{-02} \pm 4.420\mathrm{e}{-03}$ & $6.895\mathrm{e}{-01} \pm 1.010\mathrm{e}{-01}$ \\
4 & $5.389\mathrm{e}{-01} \pm 2.420\mathrm{e}{-02}$ & $3.469\mathrm{e}{-02} \pm 2.410\mathrm{e}{-03}$ & $6.588\mathrm{e}{-01} \pm 5.070\mathrm{e}{-02}$ \\
5 & $5.871\mathrm{e}{-01} \pm 1.240\mathrm{e}{-02}$ & $3.842\mathrm{e}{-02} \pm 3.320\mathrm{e}{-03}$ & $7.504\mathrm{e}{-01} \pm 5.620\mathrm{e}{-02}$ \\

\bottomrule
\end{tabular}
\label{tab:pca_extra_degree2}
\end{table}

\subsection{Comparison between SplineFlow with $m=3$ and MMFM \citep{12_rohbeck_FM}}

\begin{table}[H]
\centering
\caption{Comparison of MMFM and SplineFlow  across ODE systems and sampling irregularity $p$.}
\vspace{0.05in}
\label{tab:sde_mmfm_splineflow_m3}
\small
\setlength{\tabcolsep}{6pt}
\begin{tabular}{llcccc}
\toprule
$\bm{p}$ & \textbf{Model} & \textbf{Exp.} & \textbf{Damp.} & \textbf{LV} & \textbf{Lorenz} \\
\midrule
\multirow{2}{*}{$0$}
& MMFM
& $1.40\mathrm{e}{-3}$
& $2.00\mathrm{e}{-3}$
& $4.87\mathrm{e}{-1}$
& $8.02\mathrm{e}{-1}$ \\
& SplineFlow ($m=3$)
& $\mathbf{9.23\mathrm{e}{-5}}$
& $\mathbf{3.10\mathrm{e}{-5}}$
& $\mathbf{4.50\mathrm{e}{-2}}$
& $\mathbf{6.39\mathrm{e}{-1}}$ \\
\midrule
\multirow{2}{*}{$0.25$}
& MMFM
& $2.00\mathrm{e}{-3}$
& $1.80\mathrm{e}{-3}$
& $3.72\mathrm{e}{-1}$
& -- \\
& SplineFlow ($m=3$)
& $\mathbf{8.42\mathrm{e}{-5}}$
& $\mathbf{2.90\mathrm{e}{-5}}$
& $\mathbf{1.87\mathrm{e}{-1}}$
& -- \\
\midrule
\multirow{2}{*}{$0.5$}
& MMFM
& $1.70\mathrm{e}{-3}$
& $3.00\mathrm{e}{-3}$
& $4.41$
& -- \\
& SplineFlow ($m=3$)
& $\mathbf{1.13\mathrm{e}{-4}}$
& $\mathbf{1.00\mathrm{e}{-4}}$
& $\mathbf{2.75\mathrm{e}{-1}}$
& -- \\
\midrule
\multirow{2}{*}{$0.75$}
& MMFM
& $2.40\mathrm{e}{-3}$
& $4.00\mathrm{e}{-3}$
& $\mathbf{8.62}$
& -- \\
& SplineFlow ($m=3$)
& $\mathbf{9.00\mathrm{e}{-5}}$
& $\mathbf{1.84\mathrm{e}{-4}}$
& $8.88$
& -- \\
\bottomrule
\end{tabular}
\end{table}

\clearpage
\newpage

\section*{NeurIPS Paper Checklist}

\begin{enumerate}

\item {\bf Claims}
    \item[] Question: Do the main claims made in the abstract and introduction accurately reflect the paper's contributions and scope?
    \item[] Answer: \answerYes{}% Replace by \answerYes{}, \answerNo{}, or \answerNA{}.
    \item[] Justification: Our main claims in the abstract and introduction revolve around the limitations of existing flow matching techniques that use linear interpolants to model dynamical systems and the proposal of SplineFlow to mitigate those limitations. In Sections~\ref{sec:linear vs. B-spline} and \ref{sec:splineflow algorithm}, we derive theoretical results showing that using b-splines leads to lower function approximation error, which in turn yields better approximation of the underlying dynamics by the velocity fields and better sample quality. Our exhaustive experiments in Section~\ref{sec:experiments} demonstrate the benefits of using higher-order B-splines across ODE, SDE, and cellular dynamics settings.
    \item[] Guidelines:
    \begin{itemize}
        \item The answer \answerNA{} means that the abstract and introduction do not include the claims made in the paper.
        \item The abstract and/or introduction should clearly state the claims made, including the contributions made in the paper and important assumptions and limitations. A \answerNo{} or \answerNA{} answer to this question will not be perceived well by the reviewers. 
        \item The claims made should match theoretical and experimental results, and reflect how much the results can be expected to generalize to other settings. 
        \item It is fine to include aspirational goals as motivation as long as it is clear that these goals are not attained by the paper. 
    \end{itemize}

\item {\bf Limitations}
    \item[] Question: Does the paper discuss the limitations of the work performed by the authors?
    \item[] Answer: \answerYes{} % Replace by \answerYes{}, \answerNo{}, or \answerNA{}.
    \item[] Justification: We discuss limitations in terms of more clarity with respect to optimal knot placement in splines, and lack of provision for non-additive diffusion terms for modeling the SDE dynamics. 
    \item[] Guidelines:
    \begin{itemize}
        \item The answer \answerNA{} means that the paper has no limitation while the answer \answerNo{} means that the paper has limitations, but those are not discussed in the paper. 
        \item The authors are encouraged to create a separate ``Limitations'' section in their paper.
        \item The paper should point out any strong assumptions and how robust the results are to violations of these assumptions (e.g., independence assumptions, noiseless settings, model well-specification, asymptotic approximations only holding locally). The authors should reflect on how these assumptions might be violated in practice and what the implications would be.
        \item The authors should reflect on the scope of the claims made, e.g., if the approach was only tested on a few datasets or with a few runs. In general, empirical results often depend on implicit assumptions, which should be articulated.
        \item The authors should reflect on the factors that influence the performance of the approach. For example, a facial recognition algorithm may perform poorly when image resolution is low or images are taken in low lighting. Or a speech-to-text system might not be used reliably to provide closed captions for online lectures because it fails to handle technical jargon.
        \item The authors should discuss the computational efficiency of the proposed algorithms and how they scale with dataset size.
        \item If applicable, the authors should discuss possible limitations of their approach to address problems of privacy and fairness.
        \item While the authors might fear that complete honesty about limitations might be used by reviewers as grounds for rejection, a worse outcome might be that reviewers discover limitations that aren't acknowledged in the paper. The authors should use their best judgment and recognize that individual actions in favor of transparency play an important role in developing norms that preserve the integrity of the community. Reviewers will be specifically instructed to not penalize honesty concerning limitations.
    \end{itemize}

\item {\bf Theory assumptions and proofs}
    \item[] Question: For each theoretical result, does the paper provide the full set of assumptions and a complete (and correct) proof?
    \item[] Answer: \answerYes{}% Replace by \answerYes{}, \answerNo{}, or \answerNA{}.
    \item[] Justification: Yes, the formal statements of our theorems are mentioned in Appendix~\ref{section: Mathematical Proofs}, and we discuss all the necessary assumptions and conditions under which the guarantees hold.
    \item[] Guidelines:
    \begin{itemize}
        \item The answer \answerNA{} means that the paper does not include theoretical results. 
        \item All the theorems, formulas, and proofs in the paper should be numbered and cross-referenced.
        \item All assumptions should be clearly stated or referenced in the statement of any theorems.
        \item The proofs can either appear in the main paper or the supplemental material, but if they appear in the supplemental material, the authors are encouraged to provide a short proof sketch to provide intuition. 
        \item Inversely, any informal proof provided in the core of the paper should be complemented by formal proofs provided in appendix or supplemental material.
        \item Theorems and Lemmas that the proof relies upon should be properly referenced. 
    \end{itemize}

    \item {\bf Experimental result reproducibility}
    \item[] Question: Does the paper fully disclose all the information needed to reproduce the main experimental results of the paper to the extent that it affects the main claims and/or conclusions of the paper (regardless of whether the code and data are provided or not)?
    \item[] Answer: \answerYes{} % Replace by \answerYes{}, \answerNo{}, or \answerNA{}.
    \item[] Justification: We've included all the implementation details, such as network architecture, hyperparameter settings and runtime, in Appendix~\ref{apdx:training}.
    \item[] Guidelines:
    \begin{itemize}
        \item The answer \answerNA{} means that the paper does not include experiments.
        \item If the paper includes experiments, a \answerNo{} answer to this question will not be perceived well by the reviewers: Making the paper reproducible is important, regardless of whether the code and data are provided or not.
        \item If the contribution is a dataset and\slash or model, the authors should describe the steps taken to make their results reproducible or verifiable. 
        \item Depending on the contribution, reproducibility can be accomplished in various ways. For example, if the contribution is a novel architecture, describing the architecture fully might suffice, or if the contribution is a specific model and empirical evaluation, it may be necessary to either make it possible for others to replicate the model with the same dataset, or provide access to the model. In general. releasing code and data is often one good way to accomplish this, but reproducibility can also be provided via detailed instructions for how to replicate the results, access to a hosted model (e.g., in the case of a large language model), releasing of a model checkpoint, or other means that are appropriate to the research performed.
        \item While NeurIPS does not require releasing code, the conference does require all submissions to provide some reasonable avenue for reproducibility, which may depend on the nature of the contribution. For example
        \begin{enumerate}
            \item If the contribution is primarily a new algorithm, the paper should make it clear how to reproduce that algorithm.
            \item If the contribution is primarily a new model architecture, the paper should describe the architecture clearly and fully.
            \item If the contribution is a new model (e.g., a large language model), then there should either be a way to access this model for reproducing the results or a way to reproduce the model (e.g., with an open-source dataset or instructions for how to construct the dataset).
            \item We recognize that reproducibility may be tricky in some cases, in which case authors are welcome to describe the particular way they provide for reproducibility. In the case of closed-source models, it may be that access to the model is limited in some way (e.g., to registered users), but it should be possible for other researchers to have some path to reproducing or verifying the results.
        \end{enumerate}
    \end{itemize}

\item {\bf Open access to data and code}
    \item[] Question: Does the paper provide open access to the data and code, with sufficient instructions to faithfully reproduce the main experimental results, as described in supplemental material?
    \item[] Answer: \answerYes{} % Replace by \answerYes{}, \answerNo{}, or \answerNA{}.
    \item[] Justification: Yes, we have released the code; the link is in the Availability section.
    \item[] Guidelines:
    \begin{itemize}
        \item The answer \answerNA{} means that paper does not include experiments requiring code.
        \item Please see the NeurIPS code and data submission guidelines (\url{https://neurips.cc/public/guides/CodeSubmissionPolicy}) for more details.
        \item While we encourage the release of code and data, we understand that this might not be possible, so \answerNo{} is an acceptable answer. Papers cannot be rejected simply for not including code, unless this is central to the contribution (e.g., for a new open-source benchmark).
        \item The instructions should contain the exact command and environment needed to run to reproduce the results. See the NeurIPS code and data submission guidelines (\url{https://neurips.cc/public/guides/CodeSubmissionPolicy}) for more details.
        \item The authors should provide instructions on data access and preparation, including how to access the raw data, preprocessed data, intermediate data, and generated data, etc.
        \item The authors should provide scripts to reproduce all experimental results for the new proposed method and baselines. If only a subset of experiments are reproducible, they should state which ones are omitted from the script and why.
        \item At submission time, to preserve anonymity, the authors should release anonymized versions (if applicable).
        \item Providing as much information as possible in supplemental material (appended to the paper) is recommended, but including URLs to data and code is permitted.
    \end{itemize}

\item {\bf Experimental setting/details}
    \item[] Question: Does the paper specify all the training and test details (e.g., data splits, hyperparameters, how they were chosen, type of optimizer) necessary to understand the results?
    \item[] Answer: \answerYes{} % Replace by \answerYes{}, \answerNo{}, or \answerNA{}.
    \item[] Justification:  We've included the training details, such as network architecture, hyperparameter settings, runtime, etc., in Appendix~\ref{apdx:training}.
    \item[] Guidelines:
    \begin{itemize}
        \item The answer \answerNA{} means that the paper does not include experiments.
        \item The experimental setting should be presented in the core of the paper to a level of detail that is necessary to appreciate the results and make sense of them.
        \item The full details can be provided either with the code, in appendix, or as supplemental material.
    \end{itemize}

\item {\bf Experiment statistical significance}
    \item[] Question: Does the paper report error bars suitably and correctly defined or other appropriate information about the statistical significance of the experiments?
    \item[] Answer: \answerYes{} % Replace by \answerYes{}, \answerNo{}, or \answerNA{}.
    \item[] Justification: We've provided mean $\pm$ std across several seeds for all the metrics, and full ablation studies in Appendix~\ref{apdx:ablations_all}, including the condensed results in Appendix~\ref{apdx:expanded_results}.
    \item[] Guidelines:
    \begin{itemize}
        \item The answer \answerNA{} means that the paper does not include experiments.
        \item The authors should answer \answerYes{} if the results are accompanied by error bars, confidence intervals, or statistical significance tests, at least for the experiments that support the main claims of the paper.
        \item The factors of variability that the error bars are capturing should be clearly stated (for example, train/test split, initialization, random drawing of some parameter, or overall run with given experimental conditions).
        \item The method for calculating the error bars should be explained (closed form formula, call to a library function, bootstrap, etc.)
        \item The assumptions made should be given (e.g., Normally distributed errors).
        \item It should be clear whether the error bar is the standard deviation or the standard error of the mean.
        \item It is OK to report 1-sigma error bars, but one should state it. The authors should preferably report a 2-sigma error bar than state that they have a 96\% CI, if the hypothesis of Normality of errors is not verified.
        \item For asymmetric distributions, the authors should be careful not to show in tables or figures symmetric error bars that would yield results that are out of range (e.g., negative error rates).
        \item If error bars are reported in tables or plots, the authors should explain in the text how they were calculated and reference the corresponding figures or tables in the text.
    \end{itemize}

\item {\bf Experiments compute resources}
    \item[] Question: For each experiment, does the paper provide sufficient information on the computer resources (type of compute workers, memory, time of execution) needed to reproduce the experiments?
    \item[] Answer: \answerYes{} % Replace by \answerYes{}, \answerNo{}, or \answerNA{}.
    \item[] Justification: We've included the details (network architecture, hyperparameter settings, and runtime) in Appendix~\ref{apdx:training}. All the experiments were run on a single Nvidia A100 GPU.
    
    \item[] Guidelines:
    \begin{itemize}
        \item The answer \answerNA{} means that the paper does not include experiments.
        \item The paper should indicate the type of compute workers CPU or GPU, internal cluster, or cloud provider, including relevant memory and storage.
        \item The paper should provide the amount of compute required for each of the individual experimental runs as well as estimate the total compute. 
        \item The paper should disclose whether the full research project required more compute than the experiments reported in the paper (e.g., preliminary or failed experiments that didn't make it into the paper). 
    \end{itemize}
    
\item {\bf Code of ethics}
    \item[] Question: Does the research conducted in the paper conform, in every respect, with the NeurIPS Code of Ethics \url{https://neurips.cc/public/EthicsGuidelines}?
    \item[] Answer: \answerYes{} % Replace by \answerYes{}, \answerNo{}, or \answerNA{}.
    \item[] Justification: We've read the guidelines and confirm that we're not in violation of any code of ethics.
    \item[] Guidelines:
    \begin{itemize}
        \item The answer \answerNA{} means that the authors have not reviewed the NeurIPS Code of Ethics.
        \item If the authors answer \answerNo, they should explain the special circumstances that require a deviation from the Code of Ethics.
        \item The authors should make sure to preserve anonymity (e.g., if there is a special consideration due to laws or regulations in their jurisdiction).
    \end{itemize}

\item {\bf Broader impacts}
    \item[] Question: Does the paper discuss both potential positive societal impacts and negative societal impacts of the work performed?
    \item[] Answer: \answerNo{}% Replace by \answerYes{}, \answerNo{}, or \answerNA{}.
    \item[] Justification: The paper present a general principled methodology to extend flow matching for dynamical systems, rather than specific applications with direct societal implications.
    \item[] Guidelines:
    \begin{itemize}
        \item The answer \answerNA{} means that there is no societal impact of the work performed.
        \item If the authors answer \answerNA{} or \answerNo, they should explain why their work has no societal impact or why the paper does not address societal impact.
        \item Examples of negative societal impacts include potential malicious or unintended uses (e.g., disinformation, generating fake profiles, surveillance), fairness considerations (e.g., deployment of technologies that could make decisions that unfairly impact specific groups), privacy considerations, and security considerations.
        \item The conference expects that many papers will be foundational research and not tied to particular applications, let alone deployments. However, if there is a direct path to any negative applications, the authors should point it out. For example, it is legitimate to point out that an improvement in the quality of generative models could be used to generate Deepfakes for disinformation. On the other hand, it is not needed to point out that a generic algorithm for optimizing neural networks could enable people to train models that generate Deepfakes faster.
        \item The authors should consider possible harms that could arise when the technology is being used as intended and functioning correctly, harms that could arise when the technology is being used as intended but gives incorrect results, and harms following from (intentional or unintentional) misuse of the technology.
        \item If there are negative societal impacts, the authors could also discuss possible mitigation strategies (e.g., gated release of models, providing defenses in addition to attacks, mechanisms for monitoring misuse, mechanisms to monitor how a system learns from feedback over time, improving the efficiency and accessibility of ML).
    \end{itemize}
    
\item {\bf Safeguards}
    \item[] Question: Does the paper describe safeguards that have been put in place for responsible release of data or models that have a high risk for misuse (e.g., pre-trained language models, image generators, or scraped datasets)?
    \item[] Answer: \answerNA{} % Replace by \answerYes{}, \answerNo{}, or \answerNA{}.
    \item[] Justification: \answerNA{}
    \item[] Guidelines:
    \begin{itemize}
        \item The answer \answerNA{} means that the paper poses no such risks.
        \item Released models that have a high risk for misuse or dual-use should be released with necessary safeguards to allow for controlled use of the model, for example by requiring that users adhere to usage guidelines or restrictions to access the model or implementing safety filters. 
        \item Datasets that have been scraped from the Internet could pose safety risks. The authors should describe how they avoided releasing unsafe images.
        \item We recognize that providing effective safeguards is challenging, and many papers do not require this, but we encourage authors to take this into account and make a best faith effort.
    \end{itemize}

\item {\bf Licenses for existing assets}
    \item[] Question: Are the creators or original owners of assets (e.g., code, data, models), used in the paper, properly credited and are the license and terms of use explicitly mentioned and properly respected?
    \item[] Answer: \answerNo{}% Replace by \answerYes{}, \answerNo{}, or \answerNA{}.
    \item[] Justification: We built on prior codebase from baselines as well as commonly used benchmarks used for dynamical systems
    \item[] Guidelines:
    \begin{itemize}
        \item The answer \answerNA{} means that the paper does not use existing assets.
        \item The authors should cite the original paper that produced the code package or dataset.
        \item The authors should state which version of the asset is used and, if possible, include a URL.
        \item The name of the license (e.g., CC-BY 4.0) should be included for each asset.
        \item For scraped data from a particular source (e.g., website), the copyright and terms of service of that source should be provided.
        \item If assets are released, the license, copyright information, and terms of use in the package should be provided. For popular datasets, \url{paperswithcode.com/datasets} has curated licenses for some datasets. Their licensing guide can help determine the license of a dataset.
        \item For existing datasets that are re-packaged, both the original license and the license of the derived asset (if it has changed) should be provided.
        \item If this information is not available online, the authors are encouraged to reach out to the asset's creators.
    \end{itemize}

\item {\bf New assets}
    \item[] Question: Are new assets introduced in the paper well documented and is the documentation provided alongside the assets?
    \item[] Answer: \answerYes{} % Replace by \answerYes{}, \answerNo{}, or \answerNA{}.
    \item[] Justification: Yes, we have released the codebase for SplineFlow along with its core additions.
    \item[] Guidelines:
    \begin{itemize}
        \item The answer \answerNA{} means that the paper does not release new assets.
        \item Researchers should communicate the details of the dataset\slash code\slash model as part of their submissions via structured templates. This includes details about training, license, limitations, etc. 
        \item The paper should discuss whether and how consent was obtained from people whose asset is used.
        \item At submission time, remember to anonymize your assets (if applicable). You can either create an anonymized URL or include an anonymized zip file.
    \end{itemize}

\item {\bf Crowdsourcing and research with human subjects}
    \item[] Question: For crowdsourcing experiments and research with human subjects, does the paper include the full text of instructions given to participants and screenshots, if applicable, as well as details about compensation (if any)? 
    \item[] Answer: \answerNA{}% Replace by \answerYes{}, \answerNo{}, or \answerNA{}.
    \item[] Justification: \answerNA{}
    \item[] Guidelines:
    \begin{itemize}
        \item The answer \answerNA{} means that the paper does not involve crowdsourcing nor research with human subjects.
        \item Including this information in the supplemental material is fine, but if the main contribution of the paper involves human subjects, then as much detail as possible should be included in the main paper. 
        \item According to the NeurIPS Code of Ethics, workers involved in data collection, curation, or other labor should be paid at least the minimum wage in the country of the data collector. 
    \end{itemize}

\item {\bf Institutional review board (IRB) approvals or equivalent for research with human subjects}
    \item[] Question: Does the paper describe potential risks incurred by study participants, whether such risks were disclosed to the subjects, and whether Institutional Review Board (IRB) approvals (or an equivalent approval/review based on the requirements of your country or institution) were obtained?
    \item[] Answer: \answerNA{}% Replace by \answerYes{}, \answerNo{}, or \answerNA{}.
    \item[] Justification: \answerNA{}
    \item[] Guidelines:
    \begin{itemize}
        \item The answer \answerNA{} means that the paper does not involve crowdsourcing nor research with human subjects.
        \item Depending on the country in which research is conducted, IRB approval (or equivalent) may be required for any human subjects research. If you obtained IRB approval, you should clearly state this in the paper. 
        \item We recognize that the procedures for this may vary significantly between institutions and locations, and we expect authors to adhere to the NeurIPS Code of Ethics and the guidelines for their institution. 
        \item For initial submissions, do not include any information that would break anonymity (if applicable), such as the institution conducting the review.
    \end{itemize}

\item {\bf Declaration of LLM usage}
    \item[] Question: Does the paper describe the usage of LLMs if it is an important, original, or non-standard component of the core methods in this research? Note that if the LLM is used only for writing, editing, or formatting purposes and does \emph{not} impact the core methodology, scientific rigor, or originality of the research, declaration is not required.
    %this research? 
    \item[] Answer: \answerNA{} % Replace by \answerYes{}, \answerNo{}, or \answerNA{}.
    \item[] Justification: \answerNA{}
    \item[] Guidelines:
    \begin{itemize}
        \item The answer \answerNA{} means that the core method development in this research does not involve LLMs as any important, original, or non-standard components.
        \item Please refer to our LLM policy in the NeurIPS handbook for what should or should not be described.
    \end{itemize}

\end{enumerate}

\end{document}